\documentclass[11pt]{article}
\usepackage[round]{natbib}
 \usepackage{stmaryrd}
\usepackage{amssymb}
\usepackage{amsmath}
\usepackage{subfigure}
\usepackage{fullpage}
\usepackage{float}
\usepackage{amsthm}
\usepackage{graphicx}
\usepackage{multirow}
\usepackage{bbm,bm}
\usepackage{xcolor}
\usepackage{algorithm} 
\usepackage{bbm}
\usepackage{algpseudocode}
\algnewcommand\algorithmicinput{\textbf{INPUT:}}
\algnewcommand\INPUT{\item[\algorithmicinput]}
\algnewcommand\algorithmicoutput{\textbf{OUTPUT:}}
\algnewcommand\OUTPUT{\item[\algorithmicoutput]}

\usepackage{hyperref}[]
\hypersetup{
    colorlinks=true,
    linkcolor=blue,
    filecolor=magenta,      
    urlcolor=cyan,
      citecolor=blue,
    }

\usepackage{cleveref}
\newsavebox\newcaptionbox\newdimen\newcaptionboxwid

\long\def\@makecaption#1#2{
 \vskip 10pt
        \baselineskip 11pt
        \setbox\@tempboxa\hbox{#1. #2}
        \ifdim \wd\@tempboxa >\hsize
        \sbox{\newcaptionbox}{\small\sl #1.~}
        \newcaptionboxwid=\wd\newcaptionbox
        \usebox\newcaptionbox {\footnotesize #2}
        \else
          \centerline{{\small\sl #1.} {\small #2}}
        \fi}

\def\fnum@figure{Figure \thefigure}
\def\fnum@table{Table \thetable}

\newcommand{\lt}{ {\mathcal L_2}}

\newtheorem{lemma}{Lemma}

\newtheorem{corollary}{Corollary}
\newtheorem{definition}{Definition}
\newtheorem{remark}{Remark}
\newtheorem{theorem}{Theorem}

\usepackage{booktabs} 
\usepackage{multirow} 
\usepackage{adjustbox} 
\usepackage{array} 
\usepackage{arydshln} 

\title{\bf Confidence Interval Construction and Conditional Variance Estimation with Dense ReLU Networks}
    \author{
	Carlos Misael Madrid Padilla$^{1\,\,\,*}$, Oscar Hernan Madrid Padilla$^{2}$        \footnote{Equal contribution from the first two authors}\\
	Yik Lun Kei$^3$, Zhi Zhang$^2$, and Yanzhen Chen$^4$\\
	\\
	$^1$Department of Statistics and Data Science,\\ 
	Washington University in St. Louis\\
	$^2$Department of Statistics and Data Science,\\ 
	University of California, Los Angeles\\
	$^3$Department of Statistics, 
	University of California, Santa Cruz\\
	$^4$Department of ISOM,
	Hong Kong University of Science and Technology\\
}
\date{
    \today
}

\begin{document}

\maketitle

\begin{abstract}

This paper addresses the problems of conditional variance estimation and confidence interval construction in nonparametric regression using dense networks with the Rectified Linear Unit (ReLU) activation function . 


We present a residual-based framework for conditional  variance estimation, deriving non-asymptotic bounds for variance estimation under both heteroscedastic and homoscedastic settings. We relax the sub-Gaussian noise assumption, allowing the proposed bounds to accommodate sub-Exponential noise and beyond. Building on this, for a ReLU neural network estimator, we derive non-asymptotic bounds for both its conditional mean and variance estimation, representing the first result for variance estimation using ReLU networks.  


Furthermore, we develop a ReLU network based robust bootstrap procedure \citep{efron1992bootstrap}  for constructing confidence intervals for the true mean that comes with a theoretical guarantee on the coverage, providing a significant advancement in  uncertainty quantification and the construction of reliable confidence intervals   in  deep learning settings.


\textbf{Keywords:} Inference, bootstrap, deep learning.

\end{abstract}

\section{Introduction}

Uncertainty quantification is a cornerstone of statistical modeling and data analysis, providing essential insights into the reliability and variability of predictions. A key aspect of uncertainty quantification is variance estimation, which enables practitioners to assess the variability of responses conditional on the inputs. For example, in financial risk assessment, understanding the conditional variance is critical for evaluating uncertainties associated with specific covariates, offering valuable insights for decision-making \citep{harris2000robust, glasserman2000variance}. Similarly, in medical studies, variance estimation identifies regions of high variability in patient outcomes, facilitating targeted interventions and personalized treatment strategies \citep{barlow1994robust, veroniki2016methods}.


In this paper, we start by looking at the task of nonparametric variance estimation with ReLU networks, with the ultimate goal of constructing confidence intervals in nonparametric regression. Specifically, we
consider data  $\{(x_i,y_i)\}_{i=1}^n \subset [0,1]^d \times \mathbb{R}$ independent  copies of $(X,Y)$ generated as
\begin{equation}
    \label{eqn:model1}
    y_i \,=\, f^*(x_i)  \,+\, \epsilon_i, 
\end{equation}
where  $x_1,\ldots,x_n$ are i.i.d. from a distribution $\mathbb{P}_X$, and where  $\epsilon_1,\ldots, \epsilon_n$ are independent random variables satisfying $\mathbb{E}(\epsilon_i | x_i ) =0$ and $\mathrm{Var}(\epsilon_i | x_i ) =g^*(x_i)$, for functions 
$f^*  \,:\, [0,1]^d \rightarrow \mathbb{R}$, and $g^*  \,:\, [0,1]^d \rightarrow \mathbb{R}_{ +}$. Thus, 
\begin{equation}
    \label{eqn:model2}
       \mathbb{E}(y_i |x_i) \,=\, f^*(x_i),\,\,\,\,\,\,    \mathrm{Var}(y_i |x_i) \,=\, g^*(x_i).
\end{equation}
Our main goals in this paper are two-folds. First, to study  estimators for the function $g^*$. Second,  we would like construct random functions $\hat{l}(\cdot)$ and $\hat{u}(\cdot)$, depending on $\{(x_i,y_i)\}_{i=1}^n$, such that $\mathbb{P}(f^*(X) \in [\hat{l}(X),\hat{u}(X)] ) \geq 1- \alpha$ for an independent  $X$ from the same distribution of covariates, and 
for some prespecified $\alpha\in (0,1)$.


Conditional variance estimation has been widely explored in the literature. Many methods, including the one proposed in this paper, follow a three-step approach: first estimating the conditional mean, computing the residuals, and then estimating the conditional variance. Notable examples of this approach include \cite{hall1989variance,fan1998efficient}


Other methods for variance estimation bypass residual computation altogether. For example, \cite{wang2008effect} and \cite{cai2009variance} explored rates of convergence for univariate nonparametric regression with Lipschitz function classes, while \cite{cai2008adaptive} introduced a wavelet thresholding approach for univariate data. More recently, \cite{shen2020optimal} examined univariate Hölder function classes and certain homoscedastic multivariate settings, and \cite{padilla2024variance} proposed the $k$-NN fused lasso method.

Despite these advancements, most existing approaches are limited to low-dimensional settings. An exception is \cite{kolar2012variance}, which extends variance estimation to higher dimensions but relies on strong parametric assumptions, underscoring the need for more flexible methods applicable to general data settings.


Beyond variance estimation, the literature on confidence intervals for nonparametric regression remains relatively sparse. Early works, such as \cite{eubank1993confidence}, \cite{brown2011confidence}, \cite{cai2014adaptive}, \cite{xia1998bias}, and \cite{hardle1988bootstrapping}, primarily addressed one-dimensional settings with homoscedastic errors.

More recent contributions have expanded the scope of confidence interval construction. For instance, \cite{neumann1998simultaneous} developed confidence bands using a local polynomial estimator in conjunction with bootstrap methods. \cite{degras2011simultaneous} extended this approach to heteroscedastic error models using a kernel-based method, while \cite{delaigle2015confidence} applied bootstrap and kernel techniques to construct confidence bands in nonparametric errors-in-variables regression. In multivariate settings, \cite{proksch2016confidence} focused on local polynomial regression with homoscedastic errors. Additionally, \cite{hall1992bootstrap} and \cite{hall2013simple} proposed approaches based on the combination of bootstrap and kernel methods. 

Despite these developments, most existing methods remain constrained to either low dimensional or parametric  settings, and the challenge of constructing theoretically supported  confidence intervals in more complex, high-dimensional nonparametric regression contexts remains an open area of research.

Neural networks have established themselves as a fundamental tool in contemporary nonparametric statistical methods, excelling at capturing and analyzing complex data relationships. Their versatility has driven widespread adoption across diverse fields \citep{krizhevsky2017imagenet,taigman2014deepface,graves2012long,goodfellow2014generative,vaswani2017attention}.

Despite their success, uncertainty quantification for neural networks often lacks theoretical guarantees. Early works \citep{white1989some,franke2000bootstrapping} focused on network parameters rather than the regression function, while \cite{paass1992assessing} explored using the bootstrap to improve predictions. More recent efforts include uncertainty quantification with ensemble methods \cite{lakshminarayanan2017simple}, a bootstrap algorithm for response variables \citep{lee2020bootstrapping}, critiques of bootstrap methods \citep{nixon2020aren}, a general uncertainty quantification framework \citep{gal2016dropout}, and a fast bootstrap method \citep{shin2021neural}. 

The primary goal of this paper is to contribute to the literature on neural networks by studying procedures for conditional variance estimation with ReLU networks. These methods will then be used to construct confidence intervals for the conditional mean. Our work adds to the extensive body of neural network theory, which has primarily focused on mean estimation in various contexts \cite{mccaffrey1994convergence}, \cite{kohler2005adaptive}, \cite{hamers2006nonasymptotic}, \cite{schmidt2020nonparametric}, \cite{bauer2019deep}, \cite{padilla2022quantile}, \cite{ma2022theoretical}, and \cite{zhang2024dense}.

\subsection{Summary of results}

We now summarize the main contributions of the paper.

\begin{enumerate}
    \item \textbf{General Framework.} We establish general non-asymptotic bounds for conditional mean estimation that apply to a broad class of nonparametric estimators. Our results extend beyond the sub-Gaussian noise assumption, accommodating, for example, sub-Exponential noise and more general distributions. Building on this foundation, we address the broader problem of conditional variance estimation for the model specified in (\ref{eqn:model1}). Specifically, we focus on residual-based estimators for conditional variance. In this framework, Theorem~\ref{thm2_v2} provides non-asymptotic upper bounds for variance estimation in heteroscedastic settings, while Theorem~\ref{thm3_v2} establishes analogous bounds for the homoscedastic case.

    \item \textbf{Dense ReLU Networks:} For a residual-based ReLU variance estimator, Corollary \ref{thm_var_v2} provides a general upper bound under hierarchical smoothness assumptions on both the true conditional means and variances. This represents the first result of its kind for variance estimation using ReLU networks.

\item  \textbf{Confidence:}  Leveraging the proposed residual-based ReLU variance estimator, we develop a bootstrap procedure \citep{efron1992bootstrap,efron1994introduction} to construct confidence intervals for $f^*(X)$, the conditional mean, where $X$ is sampled from the covariate distribution. This approach provides a guaranteed coverage and marks a significant step forward in quantifying uncertainty and constructing reliable confidence intervals in complex, nonparametric settings.

\item  \textbf{Experiments:} Both simulations and real data experiments demonstrate that the studied ReLU-based estimator outperforms state-of-the-art methods for conditional variance estimation. Furthermore, our experiments reveal that the constructed confidence intervals consistently achieve the desired coverage levels while maintaining reasonable lengths, striking an effective balance between reliability and precision.


    
\end{enumerate}


\subsection{Notation}

Throughout this article, the following notation is used:

The set of positive integers is denoted by $\mathbf{Z}^{+}$, and the set of natural numbers is denoted by $\mathbf{N} = \mathbf{Z}^{+} \cup \{0\}$. For two positive sequences $\{a_n\}_{n\in \mathbf{Z}^{+} }$ and $\{b_n\}_{n\in \mathbf{Z}^{+} }$, we write $a_n = O(b_n)$ or $a_n\lesssim b_n$ if $a_n\le C b_n$ with some constant $C > 0$ that does not depend on $n$, and $a_n = \Theta(b_n)$ or $a_n\asymp b_n$ if $a_n = O(b_n)$ and $b_n = O(a_n)$.


For a $d$-dimensional vector $\mathbf{x} \in \mathbb{R}^d$, the Euclidean and the supremum norms of $\mathbf{x}$ are denoted by $\|\mathbf{x}\|$ and $\|\mathbf{x}\|_{\infty}$, respectively.  The $\ell_\infty$ norm of a function $f: \mathbb{R}^d \rightarrow \mathbb{R}$ is defined by
$ 
\|f\|_{\infty}=\sup _{\mathbf{x} \in \mathbb{R}^d}|f(\mathbf{x})|.
$ 
Given a function $f: \mathcal{X} \rightarrow \mathbb{R}$ and the  probability distribution $\mathbb{P}_X$ of $X$  over $\mathcal{X}$, the usual $\mathcal L_2(\mathbb{P}_X)$-norm is given by
$
\|f\|_{\mathcal L_2(\mathbb{P}_X)}:=\bigl( \int_\mathcal{X} f^2(x) \mathbb{P}_X(d x)\bigr)^{1/2}=\mathbb{E}\left[f^2(X)\right]^{1/2}.
$
 The space of functions such that $\|f\|_{\mathcal L_2(\mathbb{P}_X)} < \infty$ is denoted by $\mathcal L_2(\mathbb{P}_X)$. Similarly, the $\mathcal L_2(\mathbb{P}_X)$-inner product is given by
$
\langle f, g\rangle_{\mathcal L_2(\mathbb{P}_X)}:=\int_\mathcal{X} f(x)g(x) \mathbb{P}_X(d x)=\mathbb{E}\left[f(X)g(X)\right].
$ 
Given a collection of samples $\{x_i\}_{i=1}^n $ that are independently and identically distributed according to $\mathbb{P}_X$, we define the empirical probability measure as
$\mathbb{P}_n(x) := \frac{1}{n} \sum_{i=1}^n \delta_{x_i}(x).$ The empirical $\mathcal{L}_2$-norm is expressed as
$\|f\|_{\mathcal{L}_2\left(\mathbb{P}_n\right)} := \left(\frac{1}{n} \sum_{i=1}^n f^2(x_i)\right)^{1/2} = \left(\int_{\mathcal{X}} f^2(x) \mathbb{P}_n(dx)\right)^{1/2}.$
Similarly, the empirical $\mathcal{L}_2$-inner product is given by
$\langle f, g \rangle_{\mathcal{L}_2\left(\mathbb{P}_n\right)} := \frac{1}{n} \sum_{i=1}^n f(x_i) g(x_i) = \int_{\mathcal{X}} f(x) g(x) \mathbb{P}_n(dx).$
When there is no ambiguity, we use the shorthands $\|f\|_n$, $\|f\|_2$, $\langle f, g \rangle_n$, and $\langle f, g \rangle_2$ to refer to $\|f\|_{\mathcal{L}_2\left(\mathbb{P}_n\right)}$, $\|f\|_{\mathcal{L}_2\left(\mathbb{P}_X\right)}$, $\langle f, g \rangle_{\mathcal{L}_2\left(\mathbb{P}_n\right)}$, and $\langle f, g \rangle_{\mathcal{L}_2\left(\mathbb{P}_X\right)}$, respectively. 
Also, for a function $f \,:\, \mathbb{R}^d\rightarrow \mathbb{R}$, the clipped version of $f$ is the function $f_{A_n}$ is given as
\[
  f_{A_n}(x) \,=\, \begin{cases}
      f(x) & \text{if}\,\, \vert f(x)\vert \leq A,\\
      -A & \text{if}\,\,  f(x)  < -A,\\
            A & \text{if}\,\,  f(x)  > A.\\
  \end{cases}
\]
Finally, write $1_S$ for the indicator function of a set $S$. Thus, $1_S(x) = 1$ if $x \in S$ and $1_S(x) =0$ otherwise.

\subsection{Outline}

The paper is organized as follows: Section~\ref{sec_methodology} introduces the methodology, including a three-stage procedure for variance estimation with ReLU networks and a bootstrap-based approach for constructing  confidence intervals. Section~\ref{sec:mean} presents theoretical error bounds for mean estimation under general noise, with applications to ReLU networks. Section~\ref{sec:variance} provides results for variance estimation in homoscedastic and heteroscedastic settings, focusing on dense ReLU networks. Section~\ref{sec:confidence_int} supports the proposed confidence intervals theoretically. Section~\ref{sec_exp} reports simulation and real-data experiments, comparing the method with Random Forests and MARS.    Empirical evaluations of our proposed confidence intervals are given in Section \ref{sec:experiements_confidence_int}.  The conclusion in Section \ref{sec:conclusion} discusses potential extensions, and Section~\ref{sec_proofs} includes technical proofs.

\section{Methodology}
\label{sec_methodology}

\subsection{Conditional variance estimation}
\label{sec:var}

In this subsection, we introduce a conditional variance estimator that operates in three stages, similar in spirt to \cite{fan1998efficient}. First, it estimates the conditional mean using ReLU networks. Next, it computes the residuals based on the estimated conditional mean. Finally, it fits a ReLU neural network to model the variance of the residuals. The detailed procedure is provided below.

We start by providing the necessary notation before establishing the proposed variance estimators. The architecture of a neural network,  involves the number of layers $L \in \mathbb{N}$, and the width vector $w = (w_1,\ldots,w_L) \in \mathbb{N}^L$. The latter is used to indicate the number of neurons in each  hidden layer.  

We denote by   $\tau(x)=\max(0,x)$, $x\in \mathbb{R}$ the  Rectified Linear Unit (ReLU) activation function. This is then use to construct  multilayer feedforward neural network with architecture, which is a function  $f: \mathbb{R}^d \rightarrow \mathbb{R}$, given as 

\begin{align}
f(a)=\sum_{i=1}^{w_L} c_{1, i}^{(L)} f_i^{(L)}(a)+c_{1,0}^{(L)}  \label{eq:form of approximation function 0}
\end{align}
for some weights $c_{1,0}^{(L)}, \ldots, c_{1, w_L}^{(L)} \in \mathbb{R}$ and for $f_i^{(L)}$ 's recursively defined by
\begin{align}
f_i^{(s)}( a)=\tau\left(\sum_{j=1}^{w_{s-1}} c_{i, j}^{(s-1)} f_j^{(s-1)}(a )+c_{i, 0}^{(s-1)}\right) \label{eq:form of approximation function L}
\end{align}
for some  $c_{i,0}^{(s-1)}, \ldots, c_{i, w_{s-1} }^{(L)} \in \mathbb{R}$, $s \in \{2,\ldots,L\}$, and 
\begin{align}
f_i^{(1)}(a)=\tau\left(\sum_{j=1}^{d} c_{i, j}^{(0)} a_j +c_{i, 0}^{(0)}\right) \label{eq:form of approximation function L2}
\end{align}
for some  $c_{i,0}^{(0)}, \ldots, c_{i, d }^{(0)} \in \mathbb{R}$.

As in  \cite{kohler2019rate}, we focus our study  on  multilayer feedforward neural network where all hidden layers have the same number of neurons. Specifically, we define the class of dense  ReLU neural networks as
\begin{align}
\mathcal{F}(L, \nu )=\left\{f: f \text { is of the form \eqref{eq:form of approximation function 0} with } w_1=w_2=\ldots=w_L=\nu\right\} \label{eq:space of neural network}
\end{align}
Here, dense simply indicates that no sparsity constraints are enforced in the parameters of the neural networks. Thus, we focus on fully connected networks.

With this notation, we consider the model described in (\ref{eqn:model1}), and first construct the mean estimator 
\begin{equation}
    \label{eqn:f_hat}
    \hat{f}\,\in\,  \underset{ f\in \mathcal{F}(L,\nu ) } {\arg\,\min} \, \left\{ \frac{1}{n}  \sum_{i=1}^n  (y_i - f(x_i) )^2  \right\}
\end{equation}
for appropriate  choices $L$ and $\nu$, and study its clipped version $\hat{f}_{\mathcal{A}_n}$, for some $\mathcal{A}_n>0$, under general error terms $\epsilon_1,\ldots,\epsilon_n$. 

Now let $\{ (x_i^{\prime},y_i^{\prime})\}_{i=1}^n$ be an independent copy of $\{(x_i,y_i)\}_{i=1}^n$, and  construct the residuals $\hat{\epsilon}_i\,:=\,  y_i^{\prime} - \hat{f}_{\mathcal{A}_n}(x_i^{\prime}) $ for $i=1,\ldots,n$. Then since by definition 
\[
 g^*(x_i^{\prime} )\,=\,  \mathbb{E}\left( (y_i^{\prime} - f^*(x_i^{\prime}) )^2\right),
\]
we construct the estimator $\hat{g}$ of $g$ given as
\begin{equation}
    \label{eqn:g_hat}
\hat{g}\,\in\,  \underset{ g\in \mathcal{F}(L^{\prime},\nu^{\prime } ) } {\arg\,\min} \, \left\{ \frac{1}{n}  \sum_{i=1}^n  (\hat{\epsilon}_i^2 - g(x_i^{\prime} ) )^2  \right\}
\end{equation}
for tuning parameters $L^{\prime} >0$ and $\nu^{\prime} >0$.  We will study the properties of the conditional variance  estimator $\hat{g}_{\mathcal{B}_n }$ for an appropriate $\mathcal{B}_n>0$.


\subsection{Confidence intervals}
\label{sec:con_intervals}

In this subsection, we present our  algorithm for constructing a confidence interval for $f^*(X)$, where $X$ is independently sampled  from the covariate distribution. The justification for this construction will be provided in Theorem \ref{thm6}.

The core idea of the construction is as follows. First, we estimate the conditional mean and variance as described in Section \ref{sec:var}. Using these estimates, we approximate the empirical distribution of the normalized residuals $\epsilon_i/\sqrt{g(x_i)}$. Subsequently, the procedure generates bootstrapped samples for the mean estimates. However, instead of directly determining the boundaries of the confidence interval using the empirical quantiles of the bootstrapped samples, we propose an alternative approach. Specifically, we compute the quantiles of the prediction error of the bootstrapped samples on a held-out validation set. These quantiles are then adjusted to ensure finite-sample validity of the resulting confidence intervals.

To enhance accuracy, we propose two versions of this correction: one based on a plug-in estimate and the other derived from theoretical considerations. The detailed steps of the construction are presented in the following section, with the full theoretical justification provided in Section \ref{sec:confidence_int}.


For this subsection, we will assume that $A_n$ is large enough such that $A_n > \max\{\mathcal{A}_n,\mathcal{B}_n \}$ with the notation from Section \ref{sec:var}, and  we assume that the response $y_i$ satisfies $\vert y_i\vert \leq A_n$ with probability one. In practice, this can be ensured by rescaling the inputs as necessary, or can be considered an approximation to the true model with an unbounded $Y$ when $A_n$ is large enough. More discussion will be given in Sections \ref{sec:confidence_int} and \ref{sec:experiements_confidence_int}.

Suppose that $\{1,\ldots,n\} \,=\, I_1 \cup I_2 \cup I_3 \cup I_4$ such that $\vert I_k \vert \asymp n/4$ for $k= 1,2,3,4$ and $\alpha \in (0,1)$.   Let $B\in \mathbb{N}$ be the number of bootstrap samples. Let $\widetilde{B} < B$ be a positive integer. This will be used for constructing  the center of the confidence interval.

\textbf{Step 1.} Let 
\begin{equation}
    \label{eqn:f_hat0}
    \hat{f}\,\in\,  \underset{ f\in \mathcal{F}(L,\nu) } {\arg\,\min} \, \left\{ \frac{1}{ \vert I_1\vert  }  \sum_{i\in I_1 }  (y_i - f(x_i) )^2  \right\}
\end{equation}
for tuning parameters $L$ and $\nu$. \\

\textbf{Step 2.}
Next, construct the residuals $\hat{\epsilon}_i\,:=\,  y_i - \hat{f}_{A_n}(x_i) $ for $i \in I_2$. Then 
let $\hat{g}$ be an estimator of  $g^*$  given as
\begin{equation}
    \label{eqn:g_hat0}
\hat{g}\,\in\,  \underset{ g\in \mathcal{F}(L^{\prime},\nu^{\prime}) } {\arg\,\min} \, \left\{ \frac{1}{ \vert I_2\vert  }  \sum_{i\in I_2 }   (\hat{\epsilon}_i^2 - g(x_i ) )^2  \right\}
\end{equation}
for tuning parameters $L^{\prime}$ and $\nu^{\prime}$.  \\


\textbf{Step 3.}
Now, define $\epsilon_i^{\prime}\,=\,   (y_i -  \hat{f}_{A_n}(x_i) )/\sqrt{  \vert\hat{g}_{A_n}(x_i) \vert  }$ for $i \in I_2$ and let 
\begin{equation}
    \label{eqn:Fhat}
    \hat{F}(t ) \,=\, \sum_{i \in I_2 } 1_{ \{ \tilde{\epsilon}_i     \leq  t   \}  },
\end{equation}
for all $t \in \mathbb{R}$, where $\{\tilde{\epsilon}_i\}_{i\in I_2 }$ are the standardized version of $\{ \epsilon^{\prime}_i \}_{i\in I_2 }$. Then, for $j \in \{1,\ldots,B+\widetilde{B}\}$, construct $\tilde{y}^{(j)}_i  \,=\, \hat{f}_{A_n}(x_i) \,+\,    \sqrt{\vert \hat{g}_{A_n}(x_i)\vert }\tilde{\epsilon}_i^{(j)}$ for $i \in I_3$, where $\tilde{\epsilon}_i^{(j)} \overset{\text{ind}  }{\sim} \hat{F}$. Then set 
$ \hat{f}^{(j)}$ as 
\begin{equation}
    \label{eqn:f_hat2}
    \hat{f}^{(j)}\,\in\,  \underset{ f\in \mathcal{F}(L_j,\nu_j) } {\arg\,\min} \, \left\{ \frac{1}{ \vert I_3\vert  }  \sum_{i\in I_3 }  (\tilde{y}_i^{(j)} - f(x_i) )^2  \right\}
\end{equation}
and 
\begin{equation}
    \label{eqn:g_hat2}
\hat{g}^{(j)}\,\in\,  \underset{ g\in \mathcal{G}(L_j^{\prime},\nu_j^{\prime}) } {\arg\,\min} \, \left\{ \frac{1}{ \vert I_3\vert  }  \sum_{i\in I_3 }    ((\tilde{y}_i^{(j)} -   \hat{f}_{A_n}^{(j)}(x_i))^2 - g(x_i ) )^2  \right\},\\
\end{equation}
for parameters $\{ (L_j,\nu_j,L_j^{\prime}, \nu_j^{\prime})\}_{j=1}^{B+\widetilde{B}}$.

\textbf{Step 4.}
Let $a_1(\alpha )$ be the $(1-\alpha /(4\widetilde{B} ))$-quantile of the numbers 
\[
   \frac{1}{\vert I_4\vert} \sum_{i \in I_4 }( y_i    \,-\, \hat{f}^{(1)}_{A_n}(x_i)  )^2,\,\,  \frac{1}{\vert I_4\vert} \sum_{i \in I_4 }( y_i    \,-\, \hat{f}^{(2)}_{A_n}(x_i)  )^2,\,\ldots, \,\frac{1}{\vert I_4\vert} \sum_{i \in I_4 }( y_i    \,-\, \hat{f}^{(B)}_{A_n}(x_i)  )^2.\\
\]

\textbf{Step 5.}
For $a_0>0$, given below, set
\begin{equation}
    \label{eqn:a}
   \begin{array}{lll}
         a(\alpha) &= & \displaystyle\bigg\vert  a_1(\alpha)  \,+\, A_n^2 \sqrt{\frac{32[ \log(8/ \alpha)\,+\,\log \widetilde{B} ] }{n}} \,+\,   A_n \sqrt{\frac{8[\log(64/\alpha )+ \log \widetilde{B}] }{n}} - \\
         &&\displaystyle \,\,\,\frac{1}{\vert I_4\vert} \sum_{i \in I_4 }\hat{g}_{A_n}^{(B+1)}(x_i)  \,+\,  a_0\bigg\vert\\
   \end{array}
\end{equation}

\textbf{Step 6.}
Corollary \ref{cor1_v2} shows that with the choices  $a_0\,=\, \alpha /(100\log^2 n)$ and 
\[
b(\alpha )\,=\,  1/(100\log^2 n)
\]
it holds for  
\begin{equation}
    \label{eqn:delt}
    \Delta(\alpha) \,:=\,  2\sqrt{ a(\alpha)/(\alpha\,\widetilde{B})}  +b(\alpha)
\end{equation}
that 
\[
\mathbb{P}\bigg(   f^*(X) \in \bigg[ \frac{1}{\widetilde{B}} \sum_{j=B+1}^{B+\widetilde{B}} \hat{f}^{(j)}_{A_n}(X) \pm   \Delta(\alpha)\bigg ] \bigg)\,\geq\, 1-\alpha.
\]
Moreover,  the choices
\begin{equation}
    \label{eqn:a0}
    a_0\,=\,   \bigg\vert  \widehat{\text{Var}}(Y)  \,-\,  \frac{1}{\vert I_4 \vert}\sum_{i\in I_4}( \hat{f}_{A_n}^{B+1}(x_i)\,-\, \bar{y}   )^2 \,-\,  \frac{1}{\vert I_4 \vert} \sum_{i\in I_4 }  \hat{g}_{A_n}^{(B+1)}(x_i)    \bigg\vert,
\end{equation}
and 
\begin{equation}
    \label{eqn:balpha}
  b(\alpha)\,=\,  \frac{32}{5\alpha (1-0.58 \alpha)  }\left[ \bigg\vert \frac{1}{\vert I_4 \vert}\sum_{i\in I_4}( \hat{f}_{A_n}^{B+1}(x_i)\,-\,y_i   )^2    \,-\,\frac{1}{\vert I_4 \vert} \sum_{i\in I_4 }  \hat{g}_{A_n}^{(B+1)}(x_i)   \bigg\vert \right]^{1/2}
\end{equation}
results in an approximation to the interval, where 
\[
 \widehat{\text{Var}}(Y) \,=\, \frac{1}{n-1}\sum_{i=1}^n (y_i\,-\, \bar{y} )^2
\]
and  $\bar{y} \,=\,\sum_{i=1}^n y_i /n$. 


\section{Theory}
\label{sec_theory}
\subsection{Conditional mean}
\label{sec:mean}

Accurate estimation of $f^*(x) = \mathbb{E}(Y|X=x)$ can be  essential for reliable conditional variance estimation. In this section, we present a general result for estimating
$f^*$,  accommodating sub-Exponential noise and providing non-asymptotic upper bounds on the estimation error.
To estimate $f^*$, we consider estimators of the form (\ref{eqn:f_hat}), where $\mathcal{F}(L,\nu)$  is replaced with an arbitrary function class $\mathcal{F}$. The main result in this context is presented in the following theorem.


\begin{theorem}
    \label{thm1_v2}
\textbf{[General mean estimation].} Suppose that   $\bar{f} \in \mathcal{F}$ is such that 
 \[
  \| \bar{f} -f^*\|_{\infty} \,\leq\,\sqrt{\phi_n},
 \]
 so that $\phi_n$ is the approximating error. Suppose that $\mathcal{A}_n$ is chosen to satisfy
 \[
 \mathcal{A}_n \geq 8\max\{  \| f^* \|_{\infty} +  \mathcal{U}_n , 8\|f^*\|_{\infty} ,8\sqrt{\phi_n}\}.
 \]
 Moreover, let $\mathcal{F}_{\mathcal{A}_n}   \,:=\, \{ f_{\mathcal{A}_n}/\mathcal{A}_n \,:\,  f \in \mathcal{F}\}$ and  assume that
 \[
 \log \mathcal N \left( \delta,  \mathcal{F}_{\mathcal{A}_n},  \left\| \cdot\right\|_n \right)\,\leq\, \eta_n(\delta)
 \]
 for some decreasing function $\eta_n \,:\, (0,1) \rightarrow \mathbb{R}_{\geq 0}$.  If 
\begin{equation}
    \label{eqn:entropy_0_v2}
    \underset{n
 \rightarrow \infty}{\lim} \,\left[  \sum_{l=0}^{\infty} \sum_{l^{\prime}=1 }^{\infty} \exp\left( -  C_1  \eta_n( 2^{-l-l^{\prime} -1}  )\right) \,+\, \sum_{l=0}^{\infty}    \exp\left(- C_2 \eta_n(2^{-l-1}) \right) \,+\,\mathbb{P}(\|\epsilon\|_{\infty} > \mathcal{U}_n)    \right]\,=\,0,
\end{equation}
for some constants $C_1,C_2>0$ and
\begin{equation}
    \label{eqn:cond_sum_v2}
   \underset{l  \in \mathbb{N} }{\sup}  \sum_{l^{\prime}=1 }^{\infty} \frac{ \eta_n(2^{-l-l^{\prime}}   )  }{   2^{2 l^{\prime}   }  \eta_n( 2^{-l} )  } \,\leq\, 1,
\end{equation}
then
\begin{equation}
    \label{eqn:claim1_v2}
    \max\{\| f^*-  \hat{f}_{ \mathcal{A}_n } \|_n^2  ,   \| f^*-  \hat{f}_{ \mathcal{A}_n } \|_{\mathcal L_2}^2  \}  \,=\, o_{ \mathbb{P} }\left(   \phi_n\,+\, \frac{ \mathcal{U}_n^2 \eta_n( \delta_n )    }{n}  \,+\, \mathcal{A}_n \delta_n^2   \right), \\
\end{equation}
where $\delta_n$ is a critical radius of $\mathcal{F}_{\mathcal{A}_n}$, see Definition \ref{def1}.
\end{theorem}

Theorem \ref{thm1_v2} establishes a general upper bound for both the empirical norm and the $\mathcal{L}_2$ 
  norm when estimating the conditional mean using an arbitrary function class  $\mathcal{F}$.  
  
  The conditions in \eqref{eqn:entropy_0_v2} and \eqref{eqn:cond_sum_v2} ensure that the complexity of the function class is properly controlled. The conclusion in Theorem \ref{thm1_v2}  provides  a solid foundation for studying variance estimation by leveraging the residuals constructed from $\hat{f}_{\mathcal{A}_n}$.

Notice that besides the conditions $\mathbb{E}( \epsilon_i \,|\,x_i) \,=\,0$  and $\mathbb{E}(\epsilon_i^2\,|\,x_i) = g^*(x_i)$ implied by  (\ref{eqn:model1}) and (\ref{eqn:model2}), the only other condition on the errors is 
\begin{equation}
    \label{eqn:error_cond}
    \underset{n \rightarrow \infty }{\lim}\, \mathbb{P}(\|\epsilon\|_{\infty} > \mathcal{U}_n)   \,=\,0,
\end{equation}
which follows from  (\ref{eqn:entropy_0_v2}).



Before applying Theorem \ref{thm1_v2} to the ReLU neural network estimator defined in (\ref{eqn:f_hat}), we now borrow some of the notation from \cite{kohler2019rate} which allows to impose structural assumptions on $f^*$.



\begin{definition}[$(p, C)$-smoothness]\label{definition: p,c smoothness}
Let $p=q+s$ for some $q \in \mathbb{N}=\mathbb{Z}^{+}\cup\left\{0\right\}$ and $0<s \leq 1$. A function $g: \mathbb{R}^d \rightarrow \mathbb{R}$ is called $(p, C)$-smooth, if for every $\alpha=\left(\alpha_1, \ldots, \alpha_d\right) \in \mathbb{N}^d$, where $d \in \mathbb{Z}^{+}$ , with $\sum_{j=1}^d \alpha_j=q$ the partial derivative $\partial^q g /\left(\partial a_1^{\alpha_1} \ldots \partial a_d^{\alpha_d}\right)$ exists and satisfies
$$
\left|\frac{\partial^q g}{\partial a_1^{\alpha_1} \ldots \partial a_d^{\alpha_d}}\left(a\right)-\frac{\partial^q g}{\partial a_1^{\alpha_1} \ldots \partial a_d^{\alpha_d}}\left(b\right)\right| \leq C\| a- b\|^s
$$
for all $a, b \in \mathbb{R}^d$.
\end{definition}

With the definition of $(p, C)$-smoothness in hand, we are ready to define the class of the generalized hierarchical interaction models $\mathcal{H}(l, \mathcal{P})$.

\begin{definition} [Space of Hierarchical Composition Models, \cite{kohler2019rate}]\label{definition: hierarchical composition model}
For $l=1$ and  smoothness constraint $\mathcal{P} \subseteq(0, \infty) \times \mathbb{N}$ the  space of hierarchical composition models is defined as
$$
\begin{array}{lll}
         \mathcal{H}(1, \mathcal{P})  & := &\left\{h: \mathbb{R}^d \rightarrow \mathbb{R}: h(a )=m\left(a_{(\pi(1))}, \ldots, a_{(\pi(K))}\right),\right. \text { where }   m: \mathbb{R}^K \rightarrow \mathbb{R} \text { is }\\
    & & \,\,\,(p, C) \text {-smooth for some }(p, K) \in \mathcal{P} \text { and } \pi:\{1, \ldots, K\} \rightarrow\{1, \ldots, d\}\} .
    \end{array}
$$
For $l>1$, we recursively define
$$
\begin{array}{lll}
 \mathcal{H}(l, \mathcal{P})    &:=&  \left\{h: \mathbb{R}^d \rightarrow \mathbb{R}: h(\mathbf{x})=m\left(f_1(a), \ldots, f_K(a)\right),\right. \text { where }  m: \mathbb{R}^K \rightarrow \mathbb{R} \text { is }\\
     && \,\,\,(p, C) \text {-smooth for some }(p, K) \in \mathcal{P}\text { and } \left.f_i \in \mathcal{H}(l-1, \mathcal{P})\right\}.
\end{array}
$$	
\end{definition}


From Definition \ref{definition: hierarchical composition model}, we observe that  $\mathcal{H}(l, \mathcal{P})$ 
 comprises functions formed through the composition of smooth functions. This class serves as the focus of our next theorem, whose proof is derived from Theorem \ref{thm_var_v2}. This result represents a generalization of Theorem 1 from \cite{kohler2019rate}, relaxing the sub-Gaussian assumption on the errors to accommodate sub-Exponential errors or more general error distributions.

\begin{theorem}
    \label{thm5_v2}
  \textbf{[Dense ReLU estimation].}   Suppose that $f^* \in  \mathcal{H}(l_1, \mathcal{P}_1) $ for some $ l_1 \in \mathbb{N}$ and $\mathcal{P}_1\subset [1,\infty ) \times \mathbb{N}$. In addition, assume that each function $m$ in the definition of $f^*$ can have different smoothness $p_m =  q_m +s_m$, for $q_m \in \mathbb{N} $, $s_m \in (0,1]$,  and of potentially different input dimension $K_m $, so that $(p_m, K_m ) \in \mathcal{P}_1$. Let $K_{\max}$ be the largest input dimension and $p_{\max}$ the largest smoothness of any of the functions $m$. Suppose that all the partial derivatives of order less than or equal to $q_g$ are uniformly bounded by constant $c_2$, and each function $m$ is Lipschitz continuous
with Lipschitz constant $C_{\mathrm{Lip} } \geq1 $. Also, assume that $\max\{p_{\max},K_{\max} \} =O(1)$.      Let  
    \[
    \phi_{n} = \max_{(p, K) \in \mathcal P_1 } n^{\frac{-2p}{ (2p+K)}}. 
\]
Then there exists sufficiently large positive constants $c_3$ and $c_4$ such that if 
\begin{equation}
    \label{eqn:choice1}
       L  =\lceil c_3 \log n \rceil\,\,\,\,\,\,\text{and}\,\,\,\,\,\, \nu = \left\lceil c_4  \max_{(p, K) \in \mathcal P_1  } n^{\frac{K}{ 2(2p+K)}}\right\rceil
\end{equation}
or 
\begin{equation}
    \label{eqn:choice2}
    L  =\left\lceil c_3  \max_{(p, K) \in \mathcal P_1  } n^{\frac{K}{ 2(2p+K)}} \log n\right\rceil\,\,\,\,\,\,\text{and}\,\,\,\,\,\, \nu = \left\lceil c_4  \right\rceil,
\end{equation}
then $\hat{f}_{\mathcal{A}_n}$, with $\hat{f}$ as defined in (\ref{eqn:f_hat}), satisfies, 
\begin{equation}
   \label{eqn:den_nn1}
     \| f^*-\hat{f}_{\mathcal{A}_n}\|_{ \lt}^2 \,=\, o_{\mathbb{P}}\left(   r_n \right),
\end{equation}
where
\begin{equation}
    \label{eqn:r_n_nn}
    r_n\,:=  \,\frac{ \max\{\mathcal{A}_n,\mathcal{U}_n^2 \} \log n  }{n}    \,+\,   \phi_n \log^3(n) \log(\mathcal{A}_n) \max\{\mathcal{A}_n,\mathcal{U}_n^2 \}  , 
\end{equation}
provided that (\ref{eqn:error_cond}) holds.
\end{theorem}

To gain some intuition on (\ref{eqn:error_cond}), suppose that  the functions $f^*$ and  $g^*$ satisfy $\max\{\|f^*\|_{\infty},\|g^*\|_{\infty}\} = O(1)$  and the errors are sub-Exponential with fixed parameters. Then, $\mathcal{U}_n$  in  (\ref{eqn:error_cond}) scales as $\mathcal{U}_n = O(\log n)$. Hence, ignoring logarithmic factors, we obtain that 
\[
   \| f^*-\hat{f}_{\mathcal{A}_n}\|_{ \lt}^2 \,=\, o_{\mathbb{P}}\left(   \frac{1}{n} +  \phi_n  \right).
\]
Thus,  Theorem \ref{thm5_v2} extends previous results on mean estimation with dense ReLU networks, as presented in \cite{kohler2019rate}. The primary distinction lies in the relaxation of assumptions on the error distribution. Unlike prior work, which relies on sub-Gaussian error assumptions, Theorem \ref{thm5_v2} accommodates more general error distributions, including sub-Exponential noise. This added flexibility is formalized in (\ref{eqn:error_cond}) and will be crucial when studying conditional variance estimation.

    

\subsection{Conditional variance}
\label{sec:variance}

We now shift our focus to the problem of conditional variance estimation, a critical task for understanding and quantifying the variability of a response variable given covariates. Our first result provides a general framework for constructing a conditional variance estimator, $\hat{g}_{\mathcal{B}_n}$.  This is achieved by replacing  $\mathcal{F}(L^{\prime},\nu^{\prime } )$ with an arbitrary function class  $\mathcal{G}$ in (\ref{eqn:g_hat}), allowing for broad applicability across different modeling settings.

In this framework, we assume that the conditional mean estimator, $\hat{f}$ is obtained using an arbitrary function class $\mathcal{F}$ in (\ref{eqn:f_hat}) instead of $\mathcal{F}(L,\nu)$.
By separating the function classes used for mean and variance estimation, this approach enables the use of tailored methods for each component, providing greater flexibility. This generality makes the framework well-suited for diverse applications, including those with complex or high-dimensional data, where the interplay between the conditional mean and variance is particularly important.



 
\begin{theorem}
    \label{thm2_v2}
\textbf{[General variance estimation].}     Suppose that   $\bar{g} \in \mathcal{G}$ is such that 
 \[
  \| \bar{g} -g^*\|_{\infty} \,\leq\,\sqrt{\psi_n},
 \]
 so that $\phi_n$ is the approximating error.     Let $\hat{g}$ be the estimator defined in (\ref{eqn:g_hat}), replacing $\mathcal{F}(L^{\prime},\nu^{\prime})$ with a function class $\mathcal{G}$, and consider $\hat{g}_{ \mathcal{B}_n}$ with $ \mathcal{B}_n$ satisfying
 \begin{equation}
     \label{eqn:b_condition}
      \mathcal{B}_n \geq   \max\{  8\mathcal{A}_n \| g^* \|_{\infty}, 8\mathcal{A}_n \sqrt{\psi_n}, \mathcal{U}_n^2  +  \frac{9 \mathcal{U}_n \mathcal{A}_n }{4}  + \frac{65  \mathcal{A}_n^2   }{32}   \}  .
 \end{equation}
 Moreover, let  $\mathcal{G}_{ \mathcal{B}_n } \,:=\, \{g_{ \mathcal{B}_n }/\mathcal{B}_n  \,:\, g \in \mathcal{G} \}$ and assume that
$\log \mathcal N \left( \delta,  \mathcal{G}_{ \mathcal{B}_n},  \left\| \cdot\right\|_n \right)\,\leq\, \eta_n^{\prime}(\delta)$
 for some decreasing function $\eta_n^{\prime} \,:\, (0,1) \rightarrow \mathbb{R}_{\geq 0}$ that satisfies (\ref{eqn:entropy_0_v2}) and (\ref{eqn:cond_sum_v2})  with $\eta_n^{\prime}$ instead of $\eta_n$.
 
 Then if the assumptions of Theorem \ref{thm1_v2}  hold, for $\delta_n^{\prime}$ a critical radius of $ \mathcal{G}_{ \mathcal{B}_n}$,  it holds that 
\begin{equation}
    \label{eqn:claim4_v2}
   \begin{array}{lll}
\displaystyle      \| g^*-  \hat{g}_{ \mathcal{B}_n } \|_{ \lt}^2   &\,=\,& \displaystyle   o_{ \mathbb{P} }\bigg(  \mathcal{B}_n r_n  \,+\,\psi_n\,+\, \frac{ \left[\mathcal{U}_n^4\,+\, \|g^*\|_{\infty}^2 \,+\, \mathcal{A}_n^2\mathcal{U}_n^2\right]  \eta_n^{\prime}( 1/\delta_n^{\prime} ) \,+\, \mathcal{B}_n^2    }{n} \,+\, \mathcal{B}_n \cdot (\delta_n^{\prime})^2\bigg),
   \end{array}
\end{equation}
provided that $\hat{f}_{\mathcal{A}_n}$ satisfies 
\[
 \| f^*-\hat{f}_{\mathcal{A}_n}\|_{ \lt}^2 \,=\, o_{\mathbb{P}}(r_n) 
\]
for  $r_n$ the rate of convergence in Theorem \ref{thm1_v2}.
\end{theorem}

Theorem \ref{thm2_v2} establishes a general upper bound for conditional variance estimation applicable to arbitrary estimators. These estimators are constructed by first fitting a conditional mean model, computing the residuals, and then fitting another mean model to the squared residuals.

\begin{remark}
    \label{remark:1}
If the $\{\epsilon_i/\sqrt{g^*(x_i)}\}_{i=1}^n  $ are sub-Gaussian or sub-Exponential with fixed parameter and $\|g^*\|_{\infty} =O(1)$, then $\mathcal{U}_n$, $\mathcal{A}_n$ and $\mathcal{B}_n$ all behave like polynomial functions of $\log n$. Hence, ignoring all logarithmic  factors, Theorem \ref{thm2_v2} 
implies that 
\begin{equation}
    \label{eqn:claim4_v3}
  \| g^*-  \hat{g}_{ \mathcal{A}_n } \|_{ \lt}^2 \,=\, o_{\mathbb{P}}\left(   r_n\,+\,\psi_n\,+\, \frac{ \eta_n^{\prime}( 1/\delta_n^{\prime} )}{n}  \,+\,\delta_n^{\prime 2}  \right)  
\end{equation}
as long as $ \| f^*-\hat{f}_{\mathcal{A}_n}\|_{ \lt}^2 \,=\, o_{\mathbb{P}}(r_n) $.  Here, the factor  $r_n$ reflects the importance of estimating $f^*$ well to ensure better estimation of the conditional variance. Furthermore, the expression in (\ref{eqn:claim4_v3}) gives a plug-in formula for the convergence rate that depends on the critical radius for the function class $\mathcal{G}$, and for the approximation error $\psi_n$.

\end{remark}

We are now ready to state our main result concerning conditional variation estimation with ReLU networks. This a direct consequence of Theorems \ref{thm5_v2} and \ref{thm2_v2}.

\begin{corollary}
      \label{thm_var_v2}
\textbf{[Dense ReLU estimation].}    Consider the notation and conditions of Theorem \ref{thm5_v2}. Suppose that 
$g^* \in  \mathcal{H}(l_2, \mathcal{P}_2) $ for some $ l_2 \in \mathbb{N}$ and $\mathcal{P}_2\subset [1,\infty ) \times \mathbb{N}$. In addition, assume that each function $m$ in the definition of $g^*$ can have different smoothness $p_m =  q_m +s_m$, for $q_m \in \mathbb{N} $, $s_m \in (0,1]$,  and of input dimension $K_m $, so that $(p_m, K_m ) \in \mathcal{P}_2$. Denote by $K_{\max}$  the largest input dimension and $p_{\max}$ the largest smoothness of any of the functions $m$. Suppose that all the partial derivatives of order less than or equal to $q_m$ are uniformly bounded by constant $c_2$, and each function $m$ is Lipschitz continuous
with Lipschitz constant $C_{\mathrm{Lip} } \geq1 $. Also, assume that $\max\{p_{\max},K_{\max} \} =O(1)$.      Let  
    \[
    \psi_{n} = \max_{(p, K) \in \mathcal P_2} n^{\frac{-2p}{ (2p+K)}}. 
\]
If $L^{\prime}$ and $\nu^{\prime} $ in  (\ref{eqn:g_hat}) are chosen  as in (\ref{eqn:choice1}) and (\ref{eqn:choice2}), respectively, with   $\mathcal{P}_1$ replaced by $\mathcal{P}_2$, then the estimator $\hat{g}_{\mathcal{B}_n}$ based on (\ref{eqn:g_hat}) satisfies 
\begin{equation}
    \label{eqn:var_rate}
     \| g^*-  \hat{g}_{ \mathcal{B}_n } \|_{ \lt}^2 \,=\, o_{\mathbb{P} }\left( \mathrm{poly}_1(\log n)\bigg[ \phi_n\,+\,  \psi_n \,+\, \frac{1}{n}  \bigg]    \right), 
\end{equation}
provided that $ \max\{\|f^*\|_{\infty},\|g^*\|_{\infty},\mathcal{U}_n \}  \,=\,  O( \mathrm{poly}_2(\log n) ) $. Here, $\mathrm{poly}_1(\cdot)$  and $\mathrm{poly}_2(\cdot)$  are polynomial functions.
\end{corollary}
  

Notice that as in Theorem \ref{thm5_v2}, our only requirement on the errors is  (\ref{eqn:error_cond}). While,
Corollary  \ref{thm_var_v2}  requires that $\mathcal{U}_n =  O( \mathrm{poly}_2(\log n) )$, this is done to conveniently write the rate in (\ref{eqn:var_rate}), but we could also write the upper bound for general $\mathcal{U}_n$, inflating the upper bound in (\ref{eqn:var_rate}) with a  polynomial function of  $\mathcal{U}_n$.

Finally, we highlight that Corollary \ref{thm_var_v2} is the first result of its kind for conditional variance estimation with  ReLU networks.  When  $\phi_n \lesssim \psi_n$ the rate only depends on the function $g^*$. In contrast, if $\psi_n << \phi_n$ then rate becomes $\phi_n$.



\subsection{Homoscedastic case}

Thus far, we have focused on the fully nonparametric case, allowing both the conditional mean and variance to be arbitrary functions. In this subsection, we take a step back to consider the homoscedastic case, where $g^*(x) =  \sigma^{2}$ for all $x \in \mathbb{R}^d$ with $\sigma^{2} = O(1)$. Within this simplified framework, we establish a general upper bound for estimating $\sigma^2$. This is achieved by estimating the conditional mean using (\ref{eqn:f_hat}), where  $\mathcal{F}(L,\nu)$ is replaced with an arbitrary function class, resulting in $\hat{f}_{\mathcal{A}_n}$, an estimator for the conditional mean.


As an estimator for $\sigma^2$ consider 
\begin{equation}
    \label{eqn:hom_estimator}
      \hat{\sigma}^2\,:=\,   \underset{ v\in [0,\mathcal{B}_n] }{\arg \min} \sum_{i=1}^n (\hat{\epsilon}_i^2 - v  )^2,
\end{equation}
$\hat{\epsilon}_i\,:=\,  y_i^{\prime} - \hat{f}_{\mathcal{A}_n}(x_i^{\prime}) $ for $i=1,\ldots,n$, for $\mathcal{B}_n>0$. 

We are now ready to state our main result for the homoscedastic setting.

\begin{theorem}
    \label{thm3_v2}
    Suppose that $\mathcal{B}_n $ satisfies
    \[
     \mathcal{B}_n \geq   \max\{  8\mathcal{A}_n \sigma^2,  \mathcal{U}_n^2  +  \frac{9 \mathcal{U}_n \mathcal{A}_n }{4}  + \frac{65  \mathcal{A}_n^2   }{32}   \}  .
    \]
    and $\mathcal{B}_n \,=\, o(n^b)$ for some $b>0$. Moreover, assume that the conditions of Theorem \ref{thm1_v2} hold. Then 
    \begin{equation}
    \label{eqn:claim3_v2}
    (\sigma^2 - \hat{\sigma}^2)^2    \,=\, o_{ \mathbb{P} }\left(  \mathcal{B}_n r_n \,+\,\frac{ \left[\mathcal{U}_n^4\,+\, \mathcal{A}_n^2\mathcal{U}_n^2\right]  \log(n )   \,+\,\mathcal{B}_n^2 }{n}  \,+\, \frac{\mathcal{B}_n \log(n)}{n}\right) \\
\end{equation}
provided that $\hat{f}_{\mathcal{A}_n}$ satisfies  $\| f^*-\hat{f}_{\mathcal{A}_n}\|_{ \lt}^2 \,=\, o_{\mathbb{P}}(r_n)$.
\end{theorem}

The upper bound in Theorem \ref{thm3_v2} is established for an arbitrary function class $\mathcal{F}$ employed in estimating the conditional mean. We next focus on applying this result specifically to ReLU neural networks.

\begin{corollary}
      \label{cor2_v2}
\textbf{[Dense ReLU estimation].}   Suppose that $\hat{f}$ is given as in (\ref{eqn:f_hat}) and  consider the notation and conditions of Theorem \ref{thm5_v2}. Suppose that $g^*(x) =\sigma^2 = O(1)$ for all $x$ and  let $\hat{\sigma}^2$ be the estimator defined (\ref{eqn:hom_estimator}) with $\mathcal{B}_n$ satisfying the conditions of Theorem \ref{thm3_v2}. Then 
\begin{equation}
    \label{eqn:var_rate_v2}
        (\sigma^2 - \hat{\sigma}^2)^2    \,=\, o_{\mathbb{P} }\left( \mathrm{poly}_1(\log n)\bigg[ \phi_n\,+\, \frac{1}{n}  \bigg]    \right) 
\end{equation}
provided that $  \max\{\|f^*\|_{\infty},\sigma^{2},\mathcal{U}_n  \}\,=\,  O( \mathrm{poly}_2(\log n) ) $, where $\mathrm{poly}_1(\cdot)$  and $\mathrm{poly}_2(\cdot)$ are polynomial functions.
\end{corollary}




\subsection{Confidence intervals}
\label{sec:confidence_int}

In this subsection, we demonstrate the practical utility of our theoretical results by applying them to the construction of confidence intervals. Specifically, our objective is to formalize the methodology introduced in Section \ref{sec:con_intervals}, providing a rigorous framework for deriving confidence intervals based on the principles established earlier in the paper.

A key component of this approach is the use of the bootstrap \cite{efron1994introduction,efron1992bootstrap}, which serves as a powerful tool for approximating the distribution of the prediction error in out-of-sample scenarios. By leveraging the bootstrap, we estimate the variability of predictions and account for uncertainties in the model. Unlike traditional methods that might rely on restrictive parametric assumptions, the bootstrap offers a flexible, data-driven mechanism for assessing the sampling distribution of the prediction error.

This subsection details how our theoretical bounds for conditional mean and variance estimators support the bootstrap-based construction of confidence intervals. By aligning the theoretical results with practical implementation, we aim to provide a robust method for inference that is particularly suited to nonparametric and high-dimensional settings. 

We now present the first result from this subsection.

\begin{theorem}
\label{thm6}
Suppose that $A_n$ satisfies  $A_n> \max\{\mathcal{A}_n, \mathcal{B}_n  \}  $,  $\mathbb{P}(\vert Y\vert \leq A_n)= 1$ and let
\begin{equation}
    \label{eqn:a02}
    a_0 \,:=\, \vert \mathbb{E}( g^*(X) )  \,-\, \mathbb{E}( \hat{g}_A^{(B+1)}(X) \,|\,  \hat{g}_A^{(B+1)} ) \vert
\end{equation}
and  for $\alpha\in (0,1)$ let 
\begin{equation}
    \label{eqn:balpha0}
    b(\alpha)  \,:=\,  \frac{32 [ \mathbb{E}((\hat{f}_{A_n}^{(B+1)}(X) - Y )^2  \,-\, \mathbb{E}(g^*(X))  ]^{1/2} }{ 5\alpha(1-0.58\alpha)    }
\end{equation}
and suppose that $B\alpha^2 /\widetilde{B} \rightarrow \infty$ fast enough.  Let  $a(\alpha) $ given as (\ref{eqn:a}),
and  $\Delta(\alpha)$ as in (\ref{eqn:delt}). Then  
\begin{equation}
    \label{eqn:coverage2}
    \mathbb{P}\bigg(   f^*(X) \in \bigg[ \frac{1}{\widetilde{B}} \sum_{j=B+1}^{B+\widetilde{B}} \hat{f}^{(j)}_{A_n}(X)  \,\pm\,  \Delta(\alpha)  \bigg] \bigg)\,\geq\, 1-\alpha.
\end{equation}
In addition,  if $\mathrm{support}(\hat{F}) \subset [-A_n^{\prime},A_n^{\prime }]$  with $\max\{A_n,A_n^{\prime}\} = O( \mathrm{poly}_1(\log n) )$, $f^*$ and $g^*$ satisfy the conditions in Theorem \ref{thm5_v2} and Corollary \ref{thm_var_v2}, then there exists $\widetilde{A}_n$ with $\widetilde{A}_n = O( \mathrm{poly}_2(\log n) )$ such that using $\widetilde{A}_n$ instead of $A_n$ in Steps 4-5 of our algorithm leads to 
\begin{equation}
    \label{eqn:lenght}
     \Delta(\alpha)\,:=\, o_{\mathbb{P}}\left( \mathrm{poly}_3(\log n) \bigg[ \frac{\phi_n^{1/4}+ \psi_n^{1/4} }{ \widetilde{B}^{1/2} \alpha^{1/2} }  \,+\, \frac{ \log^{1/4}(1/\alpha) }{  \widetilde{B}^{1/2} \alpha^{1/2} n^{1/4} }  \,+\,  \frac{\phi_n^{1/2} }{\alpha} \bigg] \right),
\end{equation}
provided that in Section \ref{sec:con_intervals} we set $L_j =L $, $\nu_j =\nu$, $L_j^{\prime} = L^{\prime}$, and $\nu^{\prime}_j = \nu^{\prime}$, with $(L,\nu,L^{\prime},\nu^{\prime})$ as in Theorem \ref{thm5_v2} and Corollary \ref{thm_var_v2}, and $B = O(n^q)$ for an appropriate $q>0$ such that 
$$  \underset{n \rightarrow \infty}{\lim } \,\, n^{q} \,\cdot\,\mathbb{P}(\| \epsilon\|_{\infty}  > \mathcal{U}_n )  = 0 .$$
Here, $\mathrm{poly}_j(\cdot)$ for $j=1,2,3$ are polynomial functions. 
\end{theorem}


Theorem \ref{thm6}  can be use to construct a confidence interval for $f^*(X)$, with vanishing length (\ref{eqn:lenght}), provided that we can accurately estimate the quantity $a_0$.
With that end mind,   the law of total variance,
\[
\mathbb{E}(g^*(X))\,=\, \text{Var}(Y) \,-\, \mathbb{E}( (f^*(X)  - \mathbb{E}(Y)    )^2 ),
\]
motivates the plugin approximation for $a_0$ given in the right hand side of (\ref{eqn:a0}). Moreover, (\ref{eqn:balpha0}) motivates the choice (\ref{eqn:balpha}).

\begin{remark}
    \label{remark2}
  In settings with  homoscedastic  errors, $g^*(x) = \sigma^2$ for all $x$, we can consider a plugin estimator of $a_0$ in (\ref{eqn:a02}) given by
   \begin{equation}
       \label{eqn:a03}
         \bigg\vert  \hat{\sigma}^2 \,-\,  \frac{1}{\vert I_4 \vert} \sum_{i\in I_4 }  \hat{g}_{A_n}^{(B+1)}(x_i)    \bigg\vert
   \end{equation}
    where $\hat{\sigma}^2$ is defined  in (\ref{eqn:hom_estimator}). Corollary \ref{cor2_v2} implies that, ignoring logarithmic factors,  (\ref{eqn:a03}) is within $O(\phi_n^{1/2} )$ of $a_0$ in (\ref{eqn:a02}) with high probability. 
\end{remark}

Next we present choices for  $a_0$ and $b(\alpha)$ that lead to valid confidence intervals.

\begin{corollary}
    \label{cor1_v2}
    With the conditions and notations of Theorem \ref{thm6}, suppose that 
    \begin{equation}
        \label{eqn:cond10}
             n^{q+1}\cdot\mathbb{P}(  \max\{\| g^*  \,-\,\hat{g}_{A_n}^{(B+1)}\|_{ \mathcal{L}_2}^2 ,  \| f^*  \,-\,\hat{f}_{A_n}^{(B+1)}\|_{ \mathcal{L}_2}^2 \}>  \varepsilon_n ) \,\rightarrow  \, 0.
    \end{equation}
    If for $s>0$ it holds that $\varepsilon_n \lesssim \alpha^2/\log^{2s} n$, then, for large enough $n$ (\ref{eqn:coverage2}) holds replacing $a_0$ with $\alpha/(100\log^s n )$, $b(\alpha)$ with $1/(100\log^s n )$, and the right hand side of (\ref{eqn:lenght}) would include the additional term $\sqrt{1/(\widetilde{B}\log^{s} n}) + 1/\log^s n$.   
\end{corollary}

\begin{remark}
    \label{rem4}
  The value $100$ and the constant $s>0$ are  arbitrary as the rate of convergence of $\hat{g}_{A_n}^{(B+1)}$ will generally have polynomial decay. However, larger $s$ can mean that a larger $n$ is needed for the conclusion of Corollary \ref{cor1_v2} to hold. In practice, we  find that $s=2$ leads to reasonable performance.
    \end{remark}

\begin{remark}
    \label{rem3}
  Finally,   in Theorem \ref{thm6}, the condition $\mathbb{P}(\vert Y\vert \leq A_n) =1 $ is similar to Assumption 1 in \cite{farrell2021deep}    and the condition  in Theorem 7.1  from   \cite{gyorfi2006distribution}. This can be relaxed to allow for sub-Exponential and sub-Gaussian errors, but it would require knowledge of an upper bound for $\max\{ (\mathbb{E}(Y^2))^{1/2},(\mathbb{E}(Y^4))^{1/4} \}$ and the resulting theory would only be applicable to very large $n$. In contrast, the confidence intervals proposed in Section \ref{sec:con_intervals} perform well even in cases where the output $Y$ is unbounded, see Section \ref{sec:experiements_confidence_int}.
    \end{remark}

\section{Experiments}
\label{sec_exp}

\subsection{Conditional variance experiments}
\label{sec:sim}

For our experiments we will consider two classes of estimators. The first one is based on residuals, by replacing $\mathcal{F}(L,\nu)$ and $\mathcal{F}(L^{\prime},\nu^{\prime})$ with other function classes or estimation procedures. The second class of estimators are called direct estimators. These are based on the 
the formula
\[
     g^*(x_i)\,=\, \mathrm{Var}(y_i | x_i)\,=\,\mathbb{E}(y_i^2| x_i) \,-\, (\mathbb{E}(y_i| x_i) )^2 \,=\, h(x_i) \,-\, ( f^*(x_i) )^2
\]
with $h(x_i)\,=\,\mathbb{E}(y_i^2| x_i)$. This suggests  to consider the estimator
\begin{equation}
    \label{eqn:dir_est}
     \hat{g}_{\mathrm{dir}}(x_i)\,=\,  \hat{h}(x_i) \,-\, ( \hat{f}_{\mathcal{A}_n}(x_i) )^2,
\end{equation}
where $\hat{f}$ is an estimator for $\mathbb{E}(Y|X=x)$ and $\hat{h}$ an estimator for $\mathbb{E}(Y^2|X=x)$. In our experiments, we will use different mean estimators to induce conditional variance estimators by considering the corresponding residual-based and direct estimators. 

We evaluate the performance of variance estimation with ReLU networks on five different scenarios. Each scenario involves a set of covariates and a univariate response target. The variance estimation involves two stages: estimating the function $f^*(\cdot)$ in the first stage, followed by estimating the function $g^*(\cdot)$ in the second stage. With this setting, we also include the following two data splitting strategies for comparison: (1) we estimate both functions $f^*(\cdot)$ and $g^*(\cdot)$ using the full dataset, and (2) we estimate the function $f^*(\cdot)$ using first half of the dataset and then we estimate the function $g^*(\cdot)$ using the remaining half of the dataset. In this simulation study, the performance of variance estimation is evaluated by the mean squared error (MSE) between the estimated $\hat{g}(\cdot)$ from the second stage and the true $g^*(\cdot)$ values.

We compare the ReLU Networks with two other competing models: Random Forests \citep{breiman2001random, geurts2006extremely} and Multivariate Adaptive Regression Splines \citep{friedman1991multivariate, friedman1993fastmars}. For each model, we consider two different estimators: (1) the estimator based on residuals as described before, and (2) the direct estimator as in (\ref{eqn:dir_est}). Both estimators rely on the estimated function $\hat{f}(\cdot)$ from the first stage to estimate the conditional variance $\hat{g}(\cdot)$ in the second stage. For the ReLU Networks (NN), we employ a fully connected feedforward architecture comprising an input layer, two hidden layers, and an output layer. Each hidden layer consists of $64$ units and uses the Rectified Linear Unit (ReLU) activation function. The network parameters are updated using the Adam optimizer with a learning rate of $0.001$. For the Random Forests (RF) from the \texttt{sklearn.ensemble} package \citep{scikit-learn}, we use the default $100$ tree estimators and the maximum depth of tree is $10$. For the Multivariate Adaptive Regression Splines (MARS) from the \texttt{earth} package \citep{earthRpackage}, we use the default maximum number of model terms calculated by the package.

In each scenario, the simulated data is generated at different sample sizes $n$ for $n \in \{2000, 4000, 6000\}$. We generate $100$ datasets independently and we report the mean of $\frac{1}{n}\sum_{i=1}^n (\hat{g}(x_i) - g^*(x_i))^2$ over $100$ trials. Specifically, we generate data in the form:
\begin{align*}
y_i & = f^*(x_i) + \sqrt{g^*(x_i)} \cdot \epsilon_i,\ i = 1, \dots, n,\\
x_i & \overset{\text{ind}}{\sim} [0,1]^d
\end{align*}
where $\epsilon_i \sim G_i$ for a distribution $G_i$ in $\mathbb{R}$, with $f^*: [0,1]^d \rightarrow \mathbb{R}$ and $g^*: [0,1]^d \rightarrow \mathbb{R}_+$. The choice of 
$d$ depends on the specific scenario. For  $G_i$, we focus on sub-Gaussian errors, which ensure that the errors in variance estimation are sub-Exponential, enabling us to validate our results beyond the sub-Gaussian setting. We examine five distinct scenarios based on this framework:


\subsection*{Scenario 1}

In this scenario, we specify
\begin{align*}
f^*(q) & = h_2 \circ h_1(q), \forall q \in \mathbb{R}^2,\\
g^*(q) & = \|q - (1/2,1/2)\|_2, \forall q \in \mathbb{R}^2,
\end{align*}
with 
\begin{align*}
h_1(q) & = \big(\sqrt{q_1} + q_1 \cdot q_2,\ \cos(2 \pi \cdot q_2)\big),\\
h_2(q) & = \sqrt{q_1 + q_2^2} + q_1^2 \cdot q_2,
\end{align*}
and we generate $\epsilon_i \sim N(0,1)$ for $i = 1, \dots, n$.

\begin{figure}[t!]
\begin{center}
\includegraphics[width=\columnwidth]{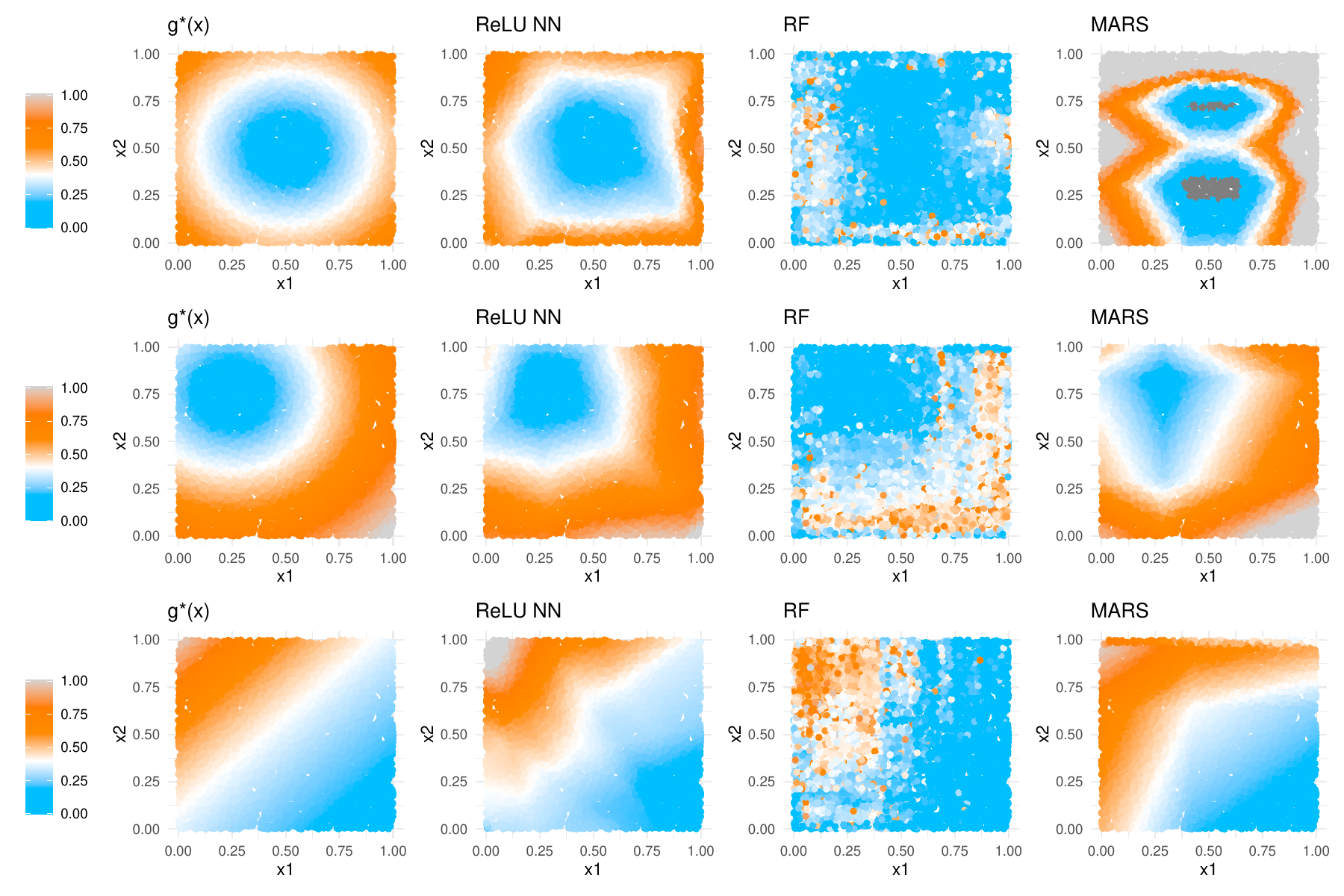}
\end{center}
\caption{The true variance  function $g^*(\cdot)$ and the predicted $\hat{g}(\cdot)$ from the estimator based on the residuals. The rows from top to bottom refer to Scenario 1,2, and 3 with $x_i \in \mathbb{R}^2$. For visualization purposes, regions with variance greater than $1$ are colored in grey.}
\label{figure_y_val}
\end{figure}



\begin{table}[!ht]
\centering
\caption{Mean (std) of $\frac{1}{n}\sum_{i=1}^n (\hat{g}(x_i) - g^*(x_i))^2$ for $n=2000$ over $100$ independent trials, using the full data estimation strategy. The best performance in terms of lower MSE is bolded.}
\begin{tabular}{|c|c|c|c|c|c|}
\hline
\textbf{Method}     & \textbf{Scenario 1} & \textbf{Scenario 2} & \textbf{Scenario 3} & \textbf{Scenario 4} & \textbf{Scenario 5}\\ 
\hline
$\text{NN}_{\text{res}}$ & $\bm{0.005}$ $(0.002)$ & $\bm{0.007}$ $(0.003)$ & $\bm{0.005}$ $(0.002)$ & $\bm{0.487}$ $(0.681)$ & $\bm{0.209}$ $(0.092)$\\ 
\hline
$\text{NN}_{\text{dir}}$   & $0.160$ $(0.057)$ & $0.020$ $(0.008)$ & $0.028$ $(0.019)$ & $783.006$ $(175.001)$ & $539.275$ $(183.769)$\\ 
\hline
$\text{RF}_{\text{res}}$   & $0.076$ $(0.003)$ & $0.111$ $(0.006)$ & $0.056$ $(0.004)$ & $2.254$ $(0.077)$ & $1.522$ $(0.038)$\\ 
\hline
$\text{RF}_{\text{dir}}$   & $0.083$ $(0.008)$ & $0.116$ $(0.011)$ & $0.064$ $(0.010)$ & $14.707$ $(1.688)$ & $20.641$ $(2.435)$\\ 
\hline
$\text{MARS}_{\text{res}}$ & $0.265$ $(0.043)$ & $0.012$ $(0.007)$ & $0.009$ $(0.007)$ & $1.424$ $(0.318)$ & $0.284$ $(0.097)$ \\ 
\hline
$\text{MARS}_{\text{dir}}$ & $1.438$ $(0.221)$ & $0.142$ $(0.052)$ & $0.039$ $(0.019)$ & $141.156$ $(21.381)$ & $487.455$ $(45.431)$ \\ 
\hline
\end{tabular}
\label{tbl_1_n2000}
\end{table}

\begin{table}[!ht]
\centering
\caption{Mean (std) of $\frac{1}{n}\sum_{i=1}^n (\hat{g}(x_i) - g^*(x_i))^2$ for $n=4000$ over $100$ independent trials, using the full data estimation strategy. The best performance in terms of lower MSE is bolded.}
\begin{tabular}{|c|c|c|c|c|c|}
\hline
\textbf{Method} & \textbf{Scenario 1} & \textbf{Scenario 2} & \textbf{Scenario 3} & \textbf{Scenario 4} & \textbf{Scenario 5}\\ 
\hline
$\text{NN}_{\text{res}}$   & $\bm{0.002}$ $(0.001)$ & $\bm{0.004}$ $(0.002)$ & $\bm{0.002}$ $(0.001)$ & $\bm{0.297}$ $(0.457)$ & $\bm{0.136}$ $(0.095)$\\ 
\hline
$\text{NN}_{\text{dir}}$   & $0.151$ $(0.055)$ & $0.014$ $(0.006)$ & $0.026$ $(0.019)$ & $801.229$ $(186.489)$ & $529.488$ $(180.445)$\\ 
\hline
$\text{RF}_{\text{res}}$   & $0.057$ $(0.003)$ & $0.077$ $(0.004)$ & $0.039$ $(0.003)$ & $1.646$ $(0.061)$ & $1.214$ $(0.030)$\\ 
\hline
$\text{RF}_{\text{dir}}$   & $0.082$ $(0.006)$ & $0.099$ $(0.009)$ & $0.054$ $(0.008)$ & $14.973$ $(1.546)$ & $19.511$ $(1.850)$\\ 
\hline
$\text{MARS}_{\text{res}}$ & $0.265$ $(0.038)$ & $0.007$ $(0.004)$ & $0.004$ $(0.003)$ & $1.349$ $(0.251)$ & $0.215$ $(0.059)$\\ 
\hline
$\text{MARS}_{\text{dir}}$ & $1.337$ $(0.186)$ & $0.097$ $(0.026)$ & $0.024$ $(0.011)$ & $138.628$ $(17.655)$ & $487.085$ $(32.186)$\\ 
\hline
\end{tabular}
\label{tbl_2_n4000}
\end{table}

\begin{table}[!ht]
\centering
\caption{Mean (std) of $\frac{1}{n}\sum_{i=1}^n (\hat{g}(x_i) - g^*(x_i))^2$ for $n=6000$ over $100$ independent trials, using the full data estimation strategy. The best performance in terms of lower MSE is bolded.}
\begin{tabular}{|c|c|c|c|c|c|}
\hline
\textbf{Method} & \textbf{Scenario 1} & \textbf{Scenario 2} & \textbf{Scenario 3} & \textbf{Scenario 4} & \textbf{Scenario 5}\\ 
\hline
$\text{NN}_{\text{res}}$   & $\bm{0.002}$ $(0.001)$ & $\bm{0.003}$ $(0.001)$ & $\bm{0.001}$ $(0.001)$ & $\bm{0.232}$ $(0.367)$ & $\bm{0.093}$ $(0.062)$\\ 
\hline
$\text{NN}_{\text{dir}}$   & $0.139$ $(0.053)$ & $0.012$ $(0.006)$ & $0.023$ $(0.015)$ & $827.884$ $(176.237)$ & $534.048$ $(160.156)$\\ 
\hline
$\text{RF}_{\text{res}}$   & $0.046$ $(0.002)$ & $0.061$ $(0.003)$ & $0.030$ $(0.002)$ & $1.307$ $(0.042)$ & $0.989$ $(0.023)$\\ 
\hline
$\text{RF}_{\text{dir}}$   & $0.082$ $(0.006)$ & $0.089$ $(0.007)$ & $0.048$ $(0.007)$ & $15.577$ $(1.454)$ & $20.331$ $(1.869)$\\ 
\hline
$\text{MARS}_{\text{res}}$ & $0.269$ $(0.028)$ & $0.006$ $(0.002)$ & $0.003$ $(0.002)$ & $1.353$ $(0.200)$ & $0.189$ $(0.046)$ \\ 
\hline
$\text{MARS}_{\text{dir}}$ & $1.299$ $(0.121)$ & $0.090$ $(0.015)$ & $0.022$ $(0.007)$ & $136.297$ $(16.438)$ & $485.846$ $(23.258)$ \\ 
\hline
\end{tabular}
\label{tbl_3_n6000}
\end{table}

\subsection*{Scenario 2}

In this scenario, we specify
\begin{align*}
f^*(q) & = h_2 \circ h_1(q), \forall q \in \mathbb{R}^2,\\
g^*(q) & = \|q - (1/4,3/4)\|_2, \forall q \in \mathbb{R}^2,
\end{align*}
with 
\begin{align*}
h_1(q) & = \big(\log(1 + q_1^2) + q_1 \cdot q_2,\ \sin(3 \pi \cdot q_2) \cdot \exp(-q_1) \big),\\
h_2(q) & = (q_1^2 + q_2^2)^{1/3} + q_1^2 / (1 + |q_2|),
\end{align*}
and we generate $\epsilon_i \sim N(0,1)$ for $i = 1, \dots, n$.

\subsection*{Scenario 3}

We set 
\begin{align*}
f^*(q) & = h_2 \circ h_1(q), \forall q \in \mathbb{R}^2,\\
g^*(q) & = \exp(-|q_1|-|q_2-1|), \forall q \in \mathbb{R}^2,
\end{align*}
where
\begin{align*}
h_1(q) & = (\tan(q_1) + q_1^2 \cdot q_2^2,\ q_2^3 - 2 \cdot q_1),\\
h_2(q) & = q_1 \cdot \sin(q_2) + \sqrt{|q_1 - q_2|} + (1+q_1^2)^{-1},
\end{align*}
with $\epsilon_i \sim N(0,1)$ for $i = 1, \dots, n$.

\subsection*{Scenario 4}

In this scenario, we specify
\begin{align*}
f^*(q) & = h_2 \circ h_1(q), \forall q \in \mathbb{R}^5,\\
g^*(q) & = h_4 \circ h_3(q), \forall q \in \mathbb{R}^5,
\end{align*}
with
\begin{align*}
h_1(q) & = \big(\sin(q_1) \cdot q_2^2 + \exp(q_3) - q_4 \cdot q_5,\ \cos(q_2) + q_3 \cdot \tanh(q_4) + q_5^3 \big),\\
h_2(q) & = \left[(q_1 + q_2)^2 + \sqrt{|q_1 \cdot q_2|}\right],\\
h_3(q) & = q - (-1/2,\ 1/2,\ -1/2,\ 1/2,\ -1/2),\\
h_4(q) & = \left[\exp(-||q||_2) + \sqrt{||q||_1 + 1}\right],
\end{align*}
and we generate $\epsilon_i \sim \text{Unif}(-\sqrt{3},\sqrt{3})$ for $i = 1, \dots, n$.



\subsection*{Scenario 5}

We define the functions \(f^*(q)\) and \(g^*(q)\) as follows:
\begin{align*}
f^*(q) & = h_2 \circ h_1(q), \quad \forall q \in \mathbb{R}^{10}, \\
g^*(q) & = \|q - 0.5 \cdot \bm{1}\|_2 + \sqrt{\|q - 0.5 \cdot \bm{1}\|_1}, \quad \forall q \in \mathbb{R}^{10},
\end{align*}
with
\begin{align*}
h_1(q) & = \left(\sum_{i=1}^4 q_i^2 + \sum_{j=5}^{10} q_j,\ \sum_{i=1}^4 q_i + \sum_{j=5}^{10} q_j^2\right), \\
h_2(q) & = q_1 + q_2 + q_1 \cdot q_2.
\end{align*}
and we let $\epsilon_i \sim \text{Unif}(-\sqrt{3},\sqrt{3})$ for $i = 1, \dots, n$

Table \ref{tbl_1_n2000} summarizes the performance of variance estimation, using the full data estimation strategy for $n=2000$. The corresponding results for $n=4000$ and $n=6000$ are summarized in Tables \ref{tbl_2_n4000} and \ref{tbl_3_n6000}, respectively. The ReLU Networks model with the estimator based on the residuals (res) achieves the lowest MSE, across different scenarios and different sample sizes. For each model, the estimator based on the residual (res) achieves lower MSE comparing to the direct estimator (dir). Furthermore, for Scenarios $4$ and $5$ with $x_i \in \mathbb{R}^5$, the performance of direct estimator (dir) is significantly worse than the estimator based on the residuals (res).

Figure \ref{figure_y_val} exhibits the true variance $g^*(x_i)$ and the predicted variance $\hat{g}(x_i)$ with the estimator based on the residuals, with $x_i \in \mathbb{R}^2$ in Scenarios $1$, $2$, and $3$. Based on the visualization, Random Forest tends to underestimate the variances in the region where the true variances are high, while the MARS tends to overestimate the variances in the region where the true variances are high. Among the three methods  with the estimator based on the residuals, the ReLU Networks provide reasonable estimates.

Table \ref{tbl_4_n2000_split} summarizes the performance of variance estimation, using the data splitting estimation strategy for $n=2000$. The corresponding results for $n=4000$ and $n=6000$ are summarized in Tables \ref{tbl_5_n4000_split} and \ref{tbl_6_n6000_split}, respectively. While the MSEs with the data splitting estimation strategy are slightly greater that the corresponding MSEs with the full data estimation strategy, the result that the ReLU Networks outperform the other two competitors remain the same. Similarly, the estimator based on the residuals outperforms the direct estimator for all three models under the data splitting estimation strategy.

\begin{table}[!ht]
\centering
\caption{Mean (std) of $\frac{1}{n}\sum_{i=1}^n (\hat{g}(x_i) - g^*(x_i))^2$ for $n=2000$ over $100$ independent trials, using the data splitting estimation strategy. The best performance in terms of lower MSE is bolded.}
\begin{tabular}{|c|c|c|c|c|c|}
\hline
\textbf{Method} & \textbf{Scenario 1} & \textbf{Scenario 2} & \textbf{Scenario 3} & \textbf{Scenario 4} & \textbf{Scenario 5}\\ 
\hline
$\text{NN}_{\text{res}}$   & $\bm{0.012}$ $(0.006)$ & $\bm{0.018}$ $(0.010)$ & $\bm{0.013}$ $(0.008)$ & $\bm{1.071}$ $(0.887)$ & $\bm{0.532}$ $(0.272)$\\ 
\hline
$\text{NN}_{\text{dir}}$   & $0.440$ $(0.199)$ & $0.127$ $(0.061)$ & $0.124$ $(0.055)$ & $795.194$ $(212.549)$ & $546.446$ $(203.746)$\\ 
\hline
$\text{RF}_{\text{res}}$   & $0.156$ $(0.036)$ & $0.250$ $(0.060)$ & $0.138$ $(0.037)$ & $3.326$ $(0.765)$ & $7.103$ $(1.984)$\\ 
\hline
$\text{RF}_{\text{dir}}$   & $3.033$ $(0.556)$ & $1.966$ $(0.248)$ & $0.975$ $(0.156)$ & $757.401$ $(66.586)$ & $760.954$ $(108.248)$\\ 
\hline
$\text{MARS}_{\text{res}}$ & $0.403$ $(0.140)$ & $0.030$ $(0.017)$ & $0.020$ $(0.015)$ & $2.024$ $(0.676)$ & $0.539$ $(0.260)$\\ 
\hline
$\text{MARS}_{\text{dir}}$ & $2.394$ $(0.950)$ & $0.350$ $(0.140)$ & $0.148$ $(0.077)$ & $234.738$ $(48.027)$ & $533.072$ $(61.728)$\\ 
\hline
\end{tabular}
\label{tbl_4_n2000_split}
\end{table}

\begin{table}[!ht]
\centering
\caption{Mean (std) of $\frac{1}{n}\sum_{i=1}^n (\hat{g}(x_i) - g^*(x_i))^2$ for $n=4000$ over $100$ independent trials, using the data splitting estimation strategy. The best performance in terms of lower MSE is bolded.}
\begin{tabular}{|c|c|c|c|c|c|}
\hline
\textbf{Method} & \textbf{Scenario 1} & \textbf{Scenario 2} & \textbf{Scenario 3} & \textbf{Scenario 4} & \textbf{Scenario 5}\\ 
\hline
$\text{NN}_{\text{res}}$   & $\bm{0.005}$ $(0.002)$ & $\bm{0.008}$ $(0.004)$ & $\bm{0.006}$ $(0.003)$ & $\bm{0.587}$ $(0.483)$ & $\bm{0.284}$ $(0.171)$\\ 
\hline
$\text{NN}_{\text{dir}}$   & $0.296$ $(0.121)$ & $0.065$ $(0.028)$ & $0.073$ $(0.032)$ & $790.833$ $(209.319)$ & $539.577$ $(184.690)$\\ 
\hline
$\text{RF}_{\text{res}}$   & $0.106$ $(0.017)$ & $0.160$ $(0.034)$ & $0.091$ $(0.024)$ & $1.867$ $(0.400)$ & $4.015$ $(1.263)$\\ 
\hline
$\text{RF}_{\text{dir}}$   & $2.615$ $(0.340)$ & $1.459$ $(0.174)$ & $0.640$ $(0.087)$ & $595.626$ $(42.258)$ & $619.797$ $(77.590)$\\ 
\hline
$\text{MARS}_{\text{res}}$ & $0.321$ $(0.080)$ & $0.014$ $(0.008)$ & $0.010$ $(0.007)$ & $1.629$ $(0.373)$ & $0.342$ $(0.130)$\\ 
\hline
$\text{MARS}_{\text{dir}}$ & $1.925$ $(0.559)$ & $0.198$ $(0.069)$ & $0.075$ $(0.037)$ & $184.599$ $(29.540)$ & $514.706$ $(40.821)$\\ 
\hline
\end{tabular}
\label{tbl_5_n4000_split}
\end{table}

\begin{table}[!ht]
\centering
\caption{Mean (std) of $\frac{1}{n}\sum_{i=1}^n (\hat{g}(x_i) - g^*(x_i))^2$ for $n=6000$ over $100$ independent trials, using the data splitting estimation strategy. The best performance in terms of lower MSE is bolded.}
\begin{tabular}{|c|c|c|c|c|c|}
\hline
\textbf{Method} & \textbf{Scenario 1} & \textbf{Scenario 2} & \textbf{Scenario 3} & \textbf{Scenario 4} & \textbf{Scenario 5}\\ 
\hline
$\text{NN}_{\text{res}}$   & $\bm{0.004}$ $(0.002)$ & $\bm{0.006}$ $(0.003)$ & $\bm{0.003}$ $(0.002)$ & $\bm{0.417}$ $(0.454)$ & $\bm{0.192}$ $(0.108)$\\ 
\hline
$\text{NN}_{\text{dir}}$   & $0.252$ $(0.099)$ & $0.054$ $(0.033)$ & $0.054$ $(0.027)$ & $810.553$ $(181.230)$ & $529.496$ $(158.183)$\\ 
\hline
$\text{RF}_{\text{res}}$   & $0.083$ $(0.013)$ & $0.123$ $(0.024)$ & $0.072$ $(0.014)$ & $1.287$ $(0.251)$ & $2.738$ $(0.628)$\\ 
\hline
$\text{RF}_{\text{dir}}$   & $2.344$ $(0.264)$ & $1.204$ $(0.135)$ & $0.503$ $(0.059)$ & $515.916$ $(34.059)$ & $538.583$ $(45.651)$\\ 
\hline
$\text{MARS}_{\text{res}}$ & $0.304$ $(0.059)$ & $0.011$ $(0.009)$ & $0.007$ $(0.005)$ & $1.538$ $(0.288)$ & $0.278$ $(0.086)$\\ 
\hline
$\text{MARS}_{\text{dir}}$ & $1.598$ $(0.389)$ & $0.158$ $(0.054)$ & $0.049$ $(0.022)$ & $165.279$ $(26.868)$ & $503.702$ $(31.397)$\\ 
\hline
\end{tabular}
\label{tbl_6_n6000_split}
\end{table}

\subsection{Confidence intervals comparisons}
\label{sec:experiements_confidence_int}

In this subsection, we evaluate the performance of the confidence interval construction proposed in Section \ref{sec:con_intervals}. To this end, we generate data following Scenarios 1-5 from Section \ref{sec:sim} and consider values of $\alpha\in \{0.01,0.05,0.1\}$ and $n \in \{20000,50000,100000\}$. For each scenario, we generate 100 datasets, and for each dataset, we sample 100 new covariates. We then calculate the proportion of times the confidence intervals contain $f^*$ evaluated at the new covariates. The average of these proportions across the 100 datasets is reported as the coverage (Cov). Additionally, we compute and report the average of the length of the constructed intervals, divided by the range of the true mean function, across the sampled covariates and datasets (PRange).

\begin{table}[t!]
	\centering
	\renewcommand{\arraystretch}{1.3} 
	\setlength{\tabcolsep}{12pt} 
	\caption{ 		Average Coverage (Cov) and Length of the confidence interval divided by the length of the range of $f^*$ (PRange) results for different competitors, scenarios, and values of $n$, with  $\alpha=0.1.$}
	\begin{adjustbox}{width=\textwidth}
		\begin{tabular}{|cc|cc|cc|cc|cc|cc|}
			\hline 
			\multirow{2}{*}{\textbf{Scenario}} & \multirow{2}{*}{$\textbf{n}$} & \multicolumn{2}{c|}{NN} & \multicolumn{2}{c|}{$\text{NN}_{\text{Emp}}$ } & \multicolumn{2}{c|}{Naive Bootstrap} & \multicolumn{2}{c|}{Standard Bootstrap} & \multicolumn{2}{c|}{NeuBoots} \\
			\cline{3-12} 
			& & \textbf{Cov} & \textbf{PRange} & \textbf{Cov} & \textbf{PRange} & \textbf{Cov} & \textbf{PRange} & \textbf{Cov} & \textbf{PRange} & \textbf{Cov} & \textbf{PRange} \\
			\hline
			\multirow{3}{*}{Scenario 1} & $20000$ & $\bm{91}\%$ & $\bm{0.1824}$ & $90\%$ & $0.2621$ & $33\%$& $0.0563$ &$82\%$& $0.2265$&$83\%$ & $0.2155$\\
			& $50000$ &$\bm{91}\%$ & $\bm{0.1372}$ & $90\%$ & $0.1881$ & $36\%$ & $0.0573$ & $84\%$ & $0.1575$ & $84\%$ & $0.1576$ \\
			& $100000$ & $\bm{90}\%$& $\bm{0.1007}$& $90\%$ & $0.1645$& $34\%$ & $0.0496$& $85\%$ &$0.1707$ & $83\%$& $0.1518$\\
			\hdashline
			\multirow{3}{*}{Scenario 2} & $20000$ & $\bm{91}\%$ & $\bm{0.2714}$& $90\%$ & $0.3123$ & $42\%$ &$0.1269$ & $85\%$ & $0.3424$&$86\%$ & $0.3403$\\
			& $50000$ & $\bm{90\%}$ & $\bm{0.2172}$& $90\%$ & $0.2498$ & $43\%$ &$0.0958$ & $84\%$ & $0.2251$& $85\%$ & $0.2271$\\
			& $100000$ & $\bm{90\%}$ & $\bm{0.1617}$ & $90\%$ & $0.2087$ & $41\%$ & $0.0850$& $84\%$& $0.1855$ & $82\%$& $0.1783$\\
			\hdashline
			\multirow{3}{*}{Scenario 3} & $20000$ & $\bm{90\%}$ &$\bm{0.3148}$  & $91\%$ &$0.3568$ & $53\%$ & $0.1240$ & $89\%$&$0.2894$ & $89\%$& $0.2860$\\
			& $50000$ & $\bm{90\%}$ &$\bm{0.2519}$  & $89\%$ &$0.3072$ & $50\%$ & $0.0792$ & $88\%$&$0.2755$ & $87\%$& $0.2799$\\
			& $100000$ & $\bm{92\%}$ & $\bm{0.2162}$& $91\%$& $0.2658$ & $51\%$ & $0.0659$ & $89\%$ & $0.2207$& $89\%$ & $0.2283$\\
			\hdashline
			\multirow{3}{*}{Scenario 4} & $20000$ & $\bm{90\%}$ & $\bm{0.1385}$ & $90\%$ & $0.1873$ & $44\%$ & $0.0289$& $78\%$ & $0.1892$ & $80\%$ & $0.1758$\\
			& $50000$ & $\bm{91\%}$ & $\bm{0.1214}$  & $91\%$ & $0.1791$ & $40\%$ & $0.0269$ & $77\%$ & $0.1746$& $78\%$&$0.1847$ \\
			& $100000$ & $\bm{91\%}$ &$\bm{0.0984}$ & $90\%$ & $0.1227$ & $42\%$ & $0.0327$ & $85\%$ & $0.1606$ & $83\%$ & $0.1693$\\
			\hdashline
			\multirow{3}{*}{Scenario 5} & $20000$ &$\bm{91\%}$ & $\bm{0.1437}$  & $90\%$ & $0.1778$& $35\%$ &$0.0112$ & $80\%$& $0.1162$ &$82\%$ & $0.1243$\\
			& $50000$ &$\bm{90\%}$ & $\bm{0.1157}$  & $91\%$ & $0.1484$ & $35\%$& $0.0102$& $83\%$& $0.1540$& $82\%$ & $0.1099$ \\
			& $100000$ & $\bm{90\%}$ & $\bm{0.0814}$ & $91\%$ & $0.1123$ & $33\%$ & $0.0096$ & $97\%$ & $0.1484$ & $89\%$ & $0.1002$ \\
			\hline 
		\end{tabular}
	\end{adjustbox}
	\label{tab:compact_results}
\end{table}

\begin{table}[t!]
	\centering
	\renewcommand{\arraystretch}{1.3} 
	\setlength{\tabcolsep}{12pt} 
	\caption{Average Coverage (Cov) and Length of the confidence interval divided by the length of the range of $f^*$ (PRange) results for different competitors, scenarios, and values of $n$, with $\alpha=0.05.$}
	\begin{adjustbox}{width=\textwidth}
		\begin{tabular}{|cc|cc|cc|cc|cc|cc|}
			\hline 
			\multirow{2}{*}{\textbf{Scenario}} & \multirow{2}{*}{$\textbf{n}$} & \multicolumn{2}{c|}{NN} & \multicolumn{2}{c|}{$\text{NN}_{\text{Emp}}$ } & \multicolumn{2}{c|}{Naive Bootstrap} & \multicolumn{2}{c|}{Standard Bootstrap} & \multicolumn{2}{c|}{NeuBoots} \\
			\cline{3-12} 
			& & \textbf{Cov} & \textbf{PRange} & \textbf{Cov} & \textbf{PRange} & \textbf{Cov} & \textbf{PRange} & \textbf{Cov} & \textbf{PRange} & \textbf{Cov} & \textbf{PRange} \\
			\hline
			\multirow{3}{*}{Scenario 1} & $20000$ & $\bm{95\%}$ & $\bm{0.2651}$ & $95\%$ & $0.3593$ & $16\%$ & $0.0017$ & $69\%$ & $0.3226$ & $69\%$ & $0.3124$\\
			& $50000$ & $\bm{96\%}$ & $\bm{0.1941}$ & $96\%$ & $0.2573$ & $14\%$ & $0.0015$ & $73\%$ & $0.2366$ & $75\%$ & $0.2289$ \\
			& $100000$ & $\bm{96\%}$ & $\bm{0.1374}$ & $95\%$ & $0.2008$ & $25\%$ & $0.0014$ & $94\%$ & $0.2049$ & $95\%$ & $0.1883$\\
			\hdashline
			\multirow{3}{*}{Scenario 2} & $20000$ & $\bm{95\%}$ & $\bm{0.2948}$ & $94\%$ & $0.3791$ & $24\%$ & $0.0059$ & $85\%$ & $0.4905$ & $87\%$ & $0.4718$\\
			& $50000$ & $\bm{96\%}$ & $\bm{0.2323}$ & $96\%$ & $0.3365$ & $30\%$ & $0.0089$ & $89\%$ & $0.3765$ & $89\%$ & $0.3731$\\
			& $100000$ & $\bm{96\%}$ & $\bm{0.1812}$ & $96\%$ & $0.2927$ & $34\%$ & $0.0058$ & $92\%$ & $0.2646$ & $90\%$ & $0.2619$\\
			\hdashline
			\multirow{3}{*}{Scenario 3} & $20000$ & $\bm{95\%}$ & $\bm{0.3342}$ & $96\%$ & $0.4232$ & $13\%$ & $0.1019$ & $54\%$ & $0.2777$ & $50\%$ & $0.3327$\\
			& $50000$ & $\bm{96\%}$ & $\bm{0.2873}$ & $97\%$ & $0.3827$ & $21\%$ & $0.0167$ & $61\%$ & $0.3133$ & $65\%$ & $0.3219$\\
			& $100000$ & $\bm{96\%}$ & $\bm{0.2523}$ & $95\%$ & $0.3346$ & $39\%$ & $0.0114$ & $91\%$ & $0.2947$ & $95\%$ & $0.3019$\\
			\hdashline
			\multirow{3}{*}{Scenario 4} & $20000$ & $\bm{95\%}$ & $\bm{0.2813}$ & $96\%$ & $0.3342$ & $20\%$ & $0.0836$ & $67\%$ & $0.1662$ & $85\%$ & $0.2862$\\
			& $50000$ & $\bm{95\%}$ & $\bm{0.2232}$ & $95\%$ & $0.2472$ & $26\%$ & $0.0658$ & $81\%$ & $0.1816$ & $91\%$ & $0.2036$\\
			& $100000$ & $\bm{96\%}$ & $\bm{0.1576}$ & $97\%$ & $0.2147$ & $28\%$ & $0.0255$ & $96\%$ & $0.2219$ & $95\%$ & $0.2077$\\
			\hdashline
			\multirow{3}{*}{Scenario 5} & $20000$ & $\bm{95\%}$ & $\bm{0.2614}$ & $95\%$ & $0.3442$ & $28\%$ & $0.0330$ & $63\%$ & $0.2282$ & $68\%$ & $0.2573$\\
			& $50000$ & $\bm{96\%}$ & $\bm{0.2227}$ & $95\%$ & $0.3142$ & $30\%$ & $0.0156$ & $53\%$ & $0.2101$ & $80\%$ & $0.2322$\\
			& $100000$ & $\bm{95\%}$ & $\bm{0.1651}$ & $96\%$ & $0.2453$ & $33\%$ & $0.0099$ & $91\%$ & $0.1902$ & $95\%$ & $0.2079$\\
			\hline 
		\end{tabular}
	\end{adjustbox}
	\label{tab:compact_results2}
\end{table}

\begin{table}[t!]
	\centering
	\renewcommand{\arraystretch}{1.3} 
	\setlength{\tabcolsep}{12pt} 
	\caption{Average Coverage (Cov) and Length of the confidence interval divided by the length of the range of $f^*$ (PRange) results for different competitors, scenarios, and values of $n$, with $\alpha=0.01.$}
	\begin{adjustbox}{width=\textwidth}
		\begin{tabular}{|cc|cc|cc|cc|cc|cc|}
			\hline 
			\multirow{2}{*}{\textbf{Scenario}} & \multirow{2}{*}{$\textbf{n}$} & \multicolumn{2}{c|}{NN} & \multicolumn{2}{c|}{$\text{NN}_{\text{Emp}}$ } & \multicolumn{2}{c|}{Naive Bootstrap} & \multicolumn{2}{c|}{Standard Bootstrap} & \multicolumn{2}{c|}{NeuBoots} \\
			\cline{3-12} 
			& & \textbf{Cov} & \textbf{PRange} & \textbf{Cov} & \textbf{PRange} & \textbf{Cov} & \textbf{PRange} & \textbf{Cov} & \textbf{PRange} & \textbf{Cov} & \textbf{PRange} \\
			\hline
			\multirow{3}{*}{Scenario 1} & $20000$ & $\bm{99\%}$ & $\bm{0.3521}$ & $99\%$ & $0.4516$ & $25\%$ & $0.0029$ & $62\%$ & $0.3729$ & $60\%$ & $0.3885$\\
			& $50000$ & $\bm{100\%}$ & $\bm{0.3143}$ & $99\%$ & $0.4027$ & $36\%$ & $0.0212$ & $67\%$ & $0.3744$ & $67\%$ & $0.3856$ \\
			& $100000$ & $\bm{99\%}$ & $\bm{0.2752}$ & $99\%$ & $0.3539$ & $45\%$ & $0.0879$ & $64\%$ & $0.3350$ & $65\%$ & $0.3108$\\
			\hdashline
			\multirow{3}{*}{Scenario 2} & $20000$ & $\bm{100\%}$ & $\bm{0.3958}$ & $99\%$ & $0.4789$ & $27\%$ & $0.0102$ & $74\%$ & $0.4654$ & $74\%$ & $0.4792$\\
			& $50000$ & $\bm{100\%}$ & $\bm{0.3459}$ & $100\%$ & $0.4223$ & $32\%$ & $0.0100$ & $75\%$ & $0.6133$ & $75\%$ & $0.6346$\\
			& $100000$ & $\bm{99\%}$ & $\bm{0.3042}$ & $99\%$ & $0.3868$ & $38\%$ & $0.0092$ & $75\%$ & $0.4621$ & $79\%$ & $0.4880$\\
			\hdashline
			\multirow{3}{*}{Scenario 3} & $20000$ & $\bm{100\%}$ & $\bm{0.4372}$ & $99\%$ & $0.5068$ & $21\%$ & $0.3074$ & $45\%$ & $0.2397$ & $48\%$ & $0.2211$\\
			& $50000$ & $\bm{99\%}$ & $\bm{0.3924}$ & $99\%$ & $0.4658$ & $30\%$ & $0.1294$ & $47\%$ & $0.1824$ & $46\%$ & $0.1750$\\
			& $100000$ & $\bm{100\%}$ & $\bm{0.3405}$ & $99\%$ & $0.4371$ & $33\%$ & $0.1132$ & $95\%$ & $0.3427$ & $95\%$ & $0.3559$\\
			\hdashline
			\multirow{3}{*}{Scenario 4} & $20000$ & $\bm{99\%}$ & $\bm{0.4003}$ & $99\%$ & $0.4832$ & $15\%$ & $0.0755$ & $70\%$ & $0.4698$ & $72\%$ & $0.4608$\\
			& $50000$ & $\bm{99\%}$ & $\bm{0.3647}$ & $99\%$ & $0.4692$ & $35\%$ & $0.1846$ & $71\%$ & $0.4396$ & $71\%$ & $0.4065$ \\
			& $100000$ & $\bm{99\%}$ & $\bm{0.3412}$ & $99\%$ & $0.4487$ & $30\%$ & $0.2418$ & $92\%$ & $0.3911$ & $93\%$ & $0.3994$\\
			\hdashline
			\multirow{3}{*}{Scenario 5} & $20000$ & $\bm{99\%}$ & $\bm{0.3085}$ & $99\%$ & $0.4142$ & $33\%$ & $0.1406$ & $75\%$ & $0.3950$ & $70\%$ & $0.3710$\\
			& $50000$ & $\bm{100\%}$ & $\bm{0.2714}$ & $99\%$ & $0.3728$ & $28\%$ & $0.1228$ & $73\%$ & $0.2969$ & $74\%$ & $0.3356$ \\
			& $100000$ & $\bm{100\%}$ & $\bm{0.2112}$ & $99\%$ & $0.2813$ & $27\%$ & $0.0995$ & $97\%$ & $0.2477$ & $98\%$ & $0.2349$ \\
			\hline 
		\end{tabular}
	\end{adjustbox}
	\label{tab:compact_results3}
\end{table}

For competitors, we evaluate several methods alongside our proposed approach from Section \ref{sec:con_intervals}. These include: our method with $a_0\,=\, \alpha /(100\log^2 n)$ and $b(\alpha) = 1/(100 \log^2 n)$, referred to as NN; 
our approach using  $a_0 $ and $b(\alpha)$ as defined in (\ref{eqn:a0}) and (\ref{eqn:balpha}), respectively, denoted as  $\mathrm{NN}_{\text{Emp}}$; a naive bootstrap method that follows Steps 1-3 from Section \ref{sec:con_intervals} and constructs confidence intervals using the empirical quantiles of $\{\hat{f}_{A_n}^{(j)} \}$;  a standard bootstrap approach that resamples with replacement from $\{(x_i,y_i\}_{i=1}^n$, refits a ReLU network to estimate the mean, and uses the quantiles of the mean estimates for confidence interval construction; and finally, the Neu Boots method proposed by \cite{shin2021neural}. This set of methods provides a comprehensive comparison for evaluating confidence interval performance. 

As for the tuning parameters, $A_n$ is chosen as the maximum value of the responses among all observations for the methods NN, $\mathrm{NN}_{\text{Emp}}$, and the naive bootstrap, which depend on this parameter. The neural network structure used across all methods follows the specifications outlined in Section \ref{sec:sim}, ensuring consistency in the number of layers, neurons, learning rates, and other architectural details. For the bootstrap parameters, $B$ and $\widetilde{B}$ are set to $1500$ and $1000$, respectively, for NN, $\mathrm{NN}_{\text{Emp}}$, and the naive bootstrap method. For the standard bootstrap and NeuBoots methods, we use $2500$ bootstrap samples. The NeuBoots method incorporates several additional tuning parameters: the number of blocks is set to $100$, and the Dirichlet distribution parameters for sampling bootstrap weights $\alpha$ is fixed to $\text{Dirichlet}(1, \ldots, 1)$. The training process for NeuBoots employs stochastic gradient descent (SGD) with a momentum of $0.9$, an initial learning rate of $0.001$, weight decay of $0.0005$, and a mini-batch size of $128$.

The results in Tables \ref{tab:compact_results}--\ref{tab:compact_results3} demonstrate that our proposed approach outperforms the competitors in terms of both coverage and the length of the constructed confidence intervals. This outcome is unsurprising, as our method explicitly accounts for the bias inherent in neural network estimates, whereas the competing methods rely on heuristics lacking theoretical justification.

\subsection{Real Data Experiments}

In this section, we conduct an experiment on predicting house values in California using median income, average occupancy, and population, following a similar configuration to \cite{petersen2016convex}. The dataset, which includes $20,640$ observations, was originally introduced by \cite{pace1997sparse} and is publicly accessible through the Carnegie Mellon StatLib repository (\texttt{lib.stat.cmu.edu}). The features are normalized to be in  $[0,1]$ range using min-max scaling, while the response variable $y$, representing the house values, is log-transformed as $\log (y)$ to reduce skewness and stabilize variance.

We perform $100$ random train-test splits with training sizes of $n =5000$ and $n=10000$ observations. For each split, the remaining data serves as the test set. We estimate the functions $f^*(\cdot)$ and $g^*(\cdot)$ with the training data. As we do not have the true variance for the individual house value, to evaluate the performance, we construct the $95\%$ and $90\%$ prediction intervals and calculate the proportion of the individual house value in the test set that falls within the intervals. Specifically, the prediction  interval for the testing data is constructed in the form: $[\hat{f}(x_i^{\text{test}}) - \hat{q}_{\alpha/2} \times \sqrt{\hat{g}(x_i^{\text{test}})},\hat{f}(x_i^{\text{test}}) + \hat{q}_{1-\alpha/2} \times \sqrt{\hat{g}(x_i^{\text{test}})}]$.  The quantile value $\hat{q}_{a}$ is calculated as the $a$-quantile from the data: $(y_i^{\text{train}} - \hat{f}(x_i^{\text{train}})) / \sqrt{ \hat{g}(x_i^{\text{train}}) }$. We report the mean proportions that are averaged over $100$ repetitions for each model. The results with $n=5000$ and $n=10000$ are reported in Tables \ref{tbl_real_n5000} and \ref{tbl_real_n10000}, respectively.

The ReLU Networks with the estimator based on the residuals (res) have better performance among the three models, as the proportion of the testing data that fall within the prediction intervals is closer to the pre-specified coverage. For each of the three models, the estimator based on the residuals (res) also outperforms the direct estimator (dir), aligned to the evaluation from the simulation study. When the training sample size increases, the performance of different models with the estimator based on the residuals has slightly improved, with better coverage of the testing data within the prediction intervals.

To further explore the confidence interval estimation, described in Section \ref{sec:con_intervals}, we also consider constructing confidence intervals for the same dataset (house values in California) and summarize the results in Table \ref{Realdata-confidence-interval}. We perform $100$ data splits of the dataset observations partitioning into $19,640$ used for the training set, and the remaining $1000$ values for testing  to compute the length of the confidence intervals. This process was conducted for different values of $\alpha$: $0.01$, $0.05$, and $0.1$, and the average length of the intervals were reported.

From  Table \ref{Realdata-confidence-interval}, it is clear  that the Neural Network (NN) methods achieve the smallest average interval lengths for all confidence levels, with the   NeuBoots method closely following. These results highlight the potential of NN-based methods for constructing compact and reliable confidence intervals in this context.

\begin{table}[t!]
\centering
\renewcommand{\arraystretch}{1} 
\setlength{\tabcolsep}{5.5pt} 
\caption{Average lengths of the confidence intervals for the log-transformed house values in the California dataset, calculated over $1000$ test points, for various confidence levels ($\alpha$).}
\begin{scriptsize}
\begin{tabular}{|c|c|c|c|c|c|}
\hline 
\textbf{{\textbf{$\alpha$}}-value} & \textbf{NN} & \textbf{$\text{NN}_{\text{Emp}}$ } & \textbf{Naive B.} & \textbf{Standard B.} & \textbf{NeuBoots} \\
\hline
$0.01$ & $\bm{0.8530}$ & $1.2192$ & $2.9395$ & $1.2759$ & $1.0391$ \\
\hline
$0.05$ & $\bm{0.5273}$ & $0.7963$ & $2.1069$ & $0.9641$ & $0.8906$ \\
\hline
$0.1$ & $\bm{0.3894}$ & $0.5103$ & $1.4196$ & $0.5344$ & $0.4894$ \\
\hline
\end{tabular}
\end{scriptsize}
\label{Realdata-confidence-interval}
\end{table}

\begin{table}[!ht]
\centering
\caption{Averaged proportion of the test data located in the prediction interval for $n_\text{train} = 5000$.}
\begin{tabular}{|c|c|c|c|c|c|}
\hline
$n_{\text{train}}$ & \textbf{Method} &  \textbf{$95\%$ Interval} & \textbf{$90\%$ Interval} \\ 
\hline
\multirow{6}{2.5em}{$5000$} & 
$\text{NN}_{\text{res}}$ & $\bm{94.94}\%$ $(0.004)$ & $\bm{89.93}\%$ $(0.005)$ \\ 
& $\text{NN}_{\text{dir}}$   & $65.11\%$ $(0.037)$ & $61.70\%$ $(0.036)$\\ 
& $\text{RF}_{\text{res}}$   & $93.21\%$ $(0.004)$ & $87.92\%$ $(0.006)$\\ 
& $\text{RF}_{\text{dir}}$   & $72.65\%$ $(0.034)$ & $68.61\%$ $(0.033)$\\ 
& $\text{MARS}_{\text{res}}$ & $94.33\%$ $(0.006)$ & $89.37\%$ $(0.006)$\\
& $\text{MARS}_{\text{dir}}$ & $79.01\%$ $(0.132)$ & $74.86\%$ $(0.125)$\\ 
\hline
\end{tabular}
\label{tbl_real_n5000}
\end{table}

\begin{table}[!ht]
\centering
\caption{Averaged proportion of the test data located in the prediction interval for $n_\text{train} = 10000$.}
\begin{tabular}{|c|c|c|c|c|c|}
\hline
$n_{\text{train}}$ & \textbf{Method} &  \textbf{$95\%$ Interval} & \textbf{$90\%$ Interval} \\ 
\hline
\multirow{6}{2.5em}{$10000$} & 
$\text{NN}_{\text{res}}$   & $\bm{94.97}\%$ $(0.003)$ & $\bm{89.93}\%$ $(0.004)$\\ 
& $\text{NN}_{\text{dir}}$   & $65.06\%$ $(0.037)$ & $61.62\%$ $(0.035)$\\ 
& $\text{RF}_{\text{res}}$   & $94.07\%$ $(0.004)$ & $88.85\%$ $(0.005)$ \\ 
& $\text{RF}_{\text{dir}}$   & $71.40\%$ $(0.040)$ & $67.55\%$ $(0.038)$\\ 
& $\text{MARS}_{\text{res}}$ & $94.40\%$ $(0.004)$ & $89.40\%$ $(0.005)$\\ 
& $\text{MARS}_{\text{dir}}$ & $83.71\%$ $(0.103)$ & $79.33\%$ $(0.098)$\\ 
\hline
\end{tabular}
\label{tbl_real_n10000}
\end{table}


\section{Conclusion}
\label{sec:conclusion}

In this paper, we have developed theoretically driven confidence intervals for $f^*(X)$  for newly sampled covariate vectors. Our approach leverages a bootstrap method based on estimates of both the conditional mean and variance of the response given the covariates, accommodating heteroscedastic errors.

A potential future research direction involves constructing confidence intervals that hold uniformly for $f^*(x)$ across all 
$x$ in the domain of the covariates. This would provide more robust inference applicable to the entire feature space.

Another avenue for exploration is applying the ideas developed in this paper to construct confidence intervals in fixed-design nonparametric regression. An example of this is trend filtering, as studied in \cite{rudin1992nonlinear,mammen1997locally,tibshirani2014adaptive,padilla2020adaptive}. However, this would require the development of alternative proof techniques to address the intrinsic challenges associated with functions of bounded variation, which may include discontinuities.

\bibliography{references}

\appendix



\section{Proofs}
\label{sec_proofs}
\subsection{Auxiliary lemmas}

We will study the theoretical properties of the estimator $\hat{g}_{\mathcal{A}_n}$ based on $\hat{g}$ defined in (\ref{eqn:g_hat}).

Let  $\bar{f} \in \mathcal{F}$ 
 and  $\bar{g} \in \mathcal{G}$  such that
 \[
  \| \bar{f} -f^*\|_{\infty} \,\leq\,\sqrt{\phi_n},\,\,\,\,\,\,\,  \| \bar{g} -g^*\|_{\infty} \,\leq\,\sqrt{\psi_n}
 \]
 so that $\sqrt{\phi_n}$.

\begin{lemma}
    \label{lem1}

    Suppose that 
    \begin{equation}
        \label{eqn:upper}
        \mathcal{A}_n \geq 8\max\{  \| f^* \|_{\infty}, \sqrt{\phi_n}\}, 
    \end{equation}
    
    then $\bar{f}_{ \mathcal{A}_n/4 } = \bar{f}$.

\end{lemma}

\begin{proof}
    Notice that
    \[
    \| \bar{f} \|_{\infty} \,\leq\,   \| f^* \|_{\infty} \,+\,  \| f^* - \bar{f} \|_{\infty} \,\leq \,  \frac{\mathcal{A}_n}{ 8}  \,+\, \sqrt{\phi_n } \,\leq\, \frac{\mathcal{A}_n}{4} 
    \]
    and the claim follows. 
\end{proof}

\begin{lemma}
   \label{lem2}
   Suppose that (\ref{eqn:upper}) and the event
   \[
     \Omega_1 \,:=\, \{   \| \epsilon \|_{\infty} \,\leq\,  \mathcal{U}_n   \}
   \]
   both hold. Then the event $\Omega_2$ given as 
   \[
    \Omega_2\,:=\,    \left\{  \| \hat{\epsilon}^2\|_{\infty} \,\leq \,   \mathcal{U}_n^2  +  \frac{9 \mathcal{U}_n \mathcal{A}_n }{4}  + \frac{65  \mathcal{A}_n^2   }{32}   \right\}
   \]
      holds.
\end{lemma}

\begin{proof}
    Notice that
    \[
    \arraycolsep=1.4pt\def\arraystretch{1.6}
    \begin{array}{lll}
         \vert \hat{\epsilon}_{j} \vert^2  & =& \vert (y_j - f^*(x_j) ) +  (f^*(x_j) -\hat{f}_{\mathcal{A}_n}(x_j)) \vert^2 \\
          & = &  \epsilon_j^2 \,+\,  2\epsilon_j  (f^*(x_j) -\hat{f}_{\mathcal{A}_n}(x_j))\,+\, (f^*(x_j) -\hat{f}_{\mathcal{A}_n}(x_j))^2\\
           & \leq & \underset{i =1,\ldots,n}{\max}\,   \epsilon_i^2 \,+\,  2  \left[ \underset{i =1,\ldots,n}{\max} \vert \epsilon_i\vert \right]\cdot \left[  \underset{i =1,\ldots,n}{\max} \vert f^*(x_i) -\hat{f}_{\mathcal{A}_n}(x_i) \vert \right] \,+\,  \\
            &&  \underset{i =1,\ldots,n}{\max} \vert f^*(x_i) -\hat{f}_{\mathcal{A}_n}(x_i) \vert^2\\
             & \leq&    \|\epsilon\|_{\infty}^2 \,+\,      2  \|\epsilon \|_{\infty} ( \|f^*\|_{\infty} +  \| \hat{f}_{ \mathcal{A}_n } \|_{\infty} ) \,+\,      2   ( \|f^*\|_{\infty}^2 +  \| \hat{f}_{ \mathcal{A}_n } \|_{\infty}^2 ) \\
              & \leq&    \|\epsilon\|_{\infty}^2 \,+\,      2  \|\epsilon \|_{\infty} ( \frac{ \mathcal{A}_n}{8} +  \mathcal{A}_n) \,+\,      2   ( \frac{ \mathcal{A}_n^2}{64} +  \mathcal{A}_n^2).
    \end{array}
    \]
\end{proof}

\subsection{Deviation bounds}


Let 
$$ 
\left\|g\right\|_n^2 =\frac{1}{n}\sum_{i=1}^n g^2 \left(x_i\right),
$$

and define 
$$ 
D_n\left(g,\epsilon\right) =  \frac{1}{n}\sum_{i=1}^n g \left(x_i\right) \epsilon_i. 
$$ 
Furthermore, for $\xi_1,\ldots,\xi_n$ independent Radamacher variables independent of $\{\epsilon_i\}_{i=1}^n$ we write
$$ 
D_n(g,\epsilon,\xi ) =  \frac{1}{n}\sum_{i=1}^n g \left(x_i\right) \epsilon_i  \xi_i. 
$$ 
In this section, the analysis is made conditioning on  $ \left\{x_i \right\}_{i=1}^n$. 

For a function $g \,:\,[0,1]^d \,\rightarrow \mathbb{R} $, We denote $\mathcal{B}_n(g,\lambda)  \,=\, \{ f  \,:\,[0,1]^d \,\rightarrow \mathbb{R}\,\,:\,\, \|f-g\|_n \leq \lambda  \}$. And we write $\mathcal{B}_n(\lambda) = \mathcal{B}_n(0,\lambda)$. We also set $B_{\infty}(\lambda) \,:=\, \{ f  \,:\,[0,1]^d \,\rightarrow \mathbb{R}\,\,:\,\, \|f\|_{\infty} \leq \lambda  \}$.

\begin{lemma} 
\label{lem6} 
Let $\mathcal F$ be  a class of functions such that
\begin{align}\label{eq:metric entropy assumption in epsilon 1}
    \log \mathcal N \left(\delta,  \mathcal F,  \left\| \cdot\right\|_n \right) \le \eta_n(\delta), \,\,\,\forall \delta \in (0,1),
\end{align}
for a decreasing function  $\eta_n \,:\, (0,1)  \,\rightarrow \mathbb{R}_{+} $.

For  $\lambda \in (0,1)$ and 
$$
t\ge \max\left\{C \mathcal{U}_n\lambda \sqrt {  \frac{\eta_n(\lambda)}{n} },  \sup_{ g \in \mathcal{F} }   \frac{ 2\sqrt{8} \|g\|_n }{\sqrt{n}}\right\},
$$
it holds that 
$$ 
\mathbb{P}\left(\sup_{g \in \mathcal F  \cap \mathcal{B}_n \left(\lambda\right) } D(g,\epsilon) \ge  t  \,\bigg|\, \{x_i\}_{i=1}^n  \right)\le  5  \sum_{l=1}^{\infty}   \exp\left(-\frac{ C_1^2  t^2 n\,  \eta_n( 2^{-l}\lambda )   }{ 2^{11} \mathcal{U}_n^2 \lambda^2 \eta_n(\lambda )  }\right)  + 8 \mathbb{P}( \|\epsilon\|_{\infty}>\mathcal{U}_n \,|\, \left\{x_i \right\}_{i=1}^n    ) 
$$
for a positive constant  $c>0$. 
\end{lemma}

\begin{proof}
For $ l\in \mathbb{Z}^+$, let $\mathcal{C}_l $ be the $2^{-l} \lambda  $ covering set of $ \mathcal{F} \cap \mathcal{B}_n \left(\lambda\right)$. Then, letting $\delta=2^{-l}\lambda$ in \eqref{eq:metric entropy assumption in epsilon 1},  we obtain that  
$$ \left|\mathcal{C}_l \right| \le \exp\left( \eta_n(2^{-l} \lambda ) \right).$$
Next, denote $ \psi_l (g) \in \mathcal C_l$  such that 
$$\left\| g-\psi_l(g) \right\|_n \leq 2^{-l} \lambda . $$
\
\\
Now, we observe that

\begin{equation}
\label{eqn:e1}
\begin{array}{l}
      \mathbb{P}\left( \sup_{ g \in\mathcal F \cap \mathcal{B}_n \left(\lambda\right) }  D\left(g,\epsilon\right) \ge  t \,\bigg|\, \{x_i\}_{i=1}^n \right)\\
    \le    \mathbb{P}\left( \sup_{ g\in \mathcal F \cap \mathcal{B}_n\left(\lambda\right) }  D \left( \psi_0 \left(g\right),\epsilon\right) \ge  t/2  \,\bigg|\, \{x_i\}_{i=1}^n\right)\\
    +    \mathbb{P}\left(  \sup_{g\in\mathcal F \cap \mathcal{B}_n(\lambda) }  \{ D(g -\psi_0(g),\epsilon)\} \ge  t/2 \,\bigg|\, \{x_i\}_{i=1}^n \right)\\
    \leq    4  \mathbb{P}\left( \sup_{ g\in \mathcal F \cap \mathcal{B}_n\left(\lambda\right) }  D \left( \psi_0 \left(g\right),  \frac{\xi\circ \epsilon}{\mathcal{U}_n} \right) \ge  \frac{t}{8 \mathcal{U}_n}  \,\bigg|\, \{x_i\}_{i=1}^n  ,\, \|\epsilon\|_{\infty} \leq  \mathcal{U}_n \right) \\
    +    4  \mathbb{P}\left( \sup_{g\in\mathcal F \cap \mathcal{B}_n(\lambda) }  \{ D\left(g -\psi_0(g),\frac{\xi\circ \epsilon}{\mathcal{U}_n} \right)\} \ge  \frac{t}{8 \mathcal{U}_n}  \,\bigg|\, \{x_i\}_{i=1}^n,\, 
    \|\epsilon\|_{\infty} \leq  \mathcal{U}_n \right)\\
     + 8 \mathbb{P}( \|\epsilon\|_{\infty}>\mathcal{U}_n   \,|\, \left\{x_i \right\}_{i=1}^n  ) \\
\end{array}
\end{equation}

where the second inequality follows from Lemma \ref{lem5}.

Next, we observe that for some constant $c>0$, by the fact that $D(\psi_0(g),\xi\circ \epsilon/\mathcal{U}_n)$ is sub-Gaussian with parameter $\left\|\psi_0(g)\right\|_n/\sqrt{n} \le \lambda/\sqrt{n} $ given $\{x_i\}_{i=1}^n$ and $\|\epsilon\|_{\infty} \leq \mathcal{U}_n$, 
\begin{equation}
    \label{eqn:e2}
      \begin{array}{l}
  \displaystyle     \mathbb{P}\left( \sup_{ g\in \mathcal F \cap \mathcal{B}_n\left(\lambda\right) }  D \left( \psi_0 \left(g\right), \frac{\xi \circ \epsilon}{ \mathcal{U}_n} \right) \ge  \frac{t}{8 \mathcal{U}_n}  \,\bigg|\, \{x_i\}_{i=1}^n ,\|\epsilon\|_{\infty}\leq \mathcal{U}_n \right)  \\
  =     \displaystyle   \mathbb{P}\left( \sup_{g\in \mathcal{C}_0 }  D \left( g,\frac{\xi \circ \epsilon}{ \mathcal{U}_n} \right) \ge  \frac{t}{8 \mathcal{U}_n}  \,\bigg|\, \{x_i\}_{i=1}^n,\|\epsilon\|_{\infty}\leq \mathcal{U}_n \right)\\
    \leq  \displaystyle \sum_{g\in \mathcal{C}_0 }  \mathbb{P}\left(  D \left( g,\frac{\xi \circ \epsilon}{ \mathcal{U}_n} \right) \ge  \frac{t}{8 \mathcal{U}_n}  \,\bigg|\, \{x_i\}_{i=1}^n,\|\epsilon\|_{\infty}\leq \mathcal{U}_n \right)\\
     \leq  \displaystyle \sum_{g\in \mathcal{C}_0 }   \exp\left(- \frac{ct^2 n}{ \mathcal{U}_n^2  \lambda^2} \right) \\
        \leq \displaystyle \vert \mathcal{C}_0  \vert \cdot \exp\left(- \frac{ct^2 n}{ \mathcal{U}_n^2 \lambda^2} \right) \\
    \leq  \displaystyle \exp \left(  \eta_n(\lambda) \right) \exp\left(- \frac{ct^2n}{\mathcal{U}_n^2\lambda^2} \right)\\
               \leq    \displaystyle\exp\left( - \frac{ct^2 n}{2 \mathcal{U}_n^2 \lambda^2} \right),\\
  \end{array}
\end{equation}
where the last inequality follows if $t\ge C \lambda \mathcal{U}_n \sqrt{ \eta_n(\lambda)/n }$ for an appropriate constant $C>0$, and where $c>0$.

Note that 
\begin{align*}
   D(\psi_0(g) -g, \xi)  =   D(\psi_0(g),\xi )-D(g,\xi) = \sum_{l=1}^\infty \{D( \psi  _{l-1}(g),\xi ) - D(\psi  _l (g,\xi) )\}.
\end{align*}

Then, for $a_l\ge 0 $ and $\sum_{l=1}^\infty a_l\le 1 $, 
\begin{equation}
    \label{eqn:e4}
       \begin{array}{l}
       \displaystyle   \mathbb{P}\left( \underset{g\in\mathcal F \cap \mathcal{B}_n(\lambda)}{\sup}  D\left(g -\psi_0(g),\frac{\xi\circ\epsilon}{\mathcal{U}_n} \right) \ge  \frac{t}{8 \mathcal{U}_n}  \,\bigg|\, \{x_i\}_{i=1}^n, 
 \|\epsilon\|_{\infty} \leq \mathcal{U}_n \right)  \\
       \leq               \displaystyle   \mathbb{P}\left( \bigcup_{l=1}^{\infty}\left\{\sup_{g\in\mathcal F \cap \mathcal{B}_n(\lambda) }   \left\{D\left(\psi_{l} (g),\frac{\xi\circ\epsilon}{\mathcal{U}_n}\right)  -D \left( \psi_{l-1}(g),\frac{\xi\circ\epsilon}{\mathcal{U}_n}\right)\right\}  \ge \frac{a_l  \cdot  t}{8 \mathcal{U}_n} \right\} \,\bigg|\, \{x_i\}_{i=1}^n, \|\epsilon\|_{\infty} \leq \mathcal{U}_n \right)  \\
       =              \displaystyle   \mathbb{P}\left( \bigcup_{l=1}^{\infty}\left\{\sup_{g\in\mathcal F \cap \mathcal{B}_n(\lambda) }   D\left(\psi_{l} (g) - \psi_{l-1}(g), \frac{\xi \circ \epsilon }{\mathcal{U}_n}\right)  \ge \frac{a_l  \cdot  t}{8 \mathcal{U}_n} \right\} \,\bigg|\, \{x_i\}_{i=1}^n,\|\epsilon\|_{\infty} \leq \mathcal{U}_n \right)  \\
       \leq  \displaystyle \sum_{l=1}^{\infty}\mathbb{P}\left( \sup_{g\in\mathcal F \cap \mathcal{B}_n(\lambda) }   D\left(\psi_{l} (g) - \psi_{l-1}(g),\frac{\xi\circ \epsilon }{\mathcal{U}_n }   \right)  \ge \frac{a_l  \cdot  t}{8 \mathcal{U}_n} \,\bigg|\, \{x_i\}_{i=1}^n,\|\epsilon\|_{\infty} \leq \mathcal{U}_n\right) \\
       =  \displaystyle \sum_{l=1}^{\infty}\mathbb{P}\left( \sup_{g\in \mathcal{C}_l, g^{\prime} \in \mathcal{C}_{l-1}  }   D\left(\psi_{l} (g) - \psi_{l-1}(g) ,\frac{\xi \circ \epsilon }{\mathcal{U}_n} \right)  \ge \frac{a_l  \cdot  t}{8 \mathcal{U}_n} \,\bigg|\, \{x_i\}_{i=1}^n,\|\epsilon\|_{\infty} \leq \mathcal{U}_n\right) \\
        \leq   \displaystyle \sum_{l=1}^{\infty}  \sum_{  \psi_{l} (g) \in \mathcal{C}_l, \psi_{l-1}(g) \in \mathcal{C}_{l-1} } \mathbb{P}\left(   D\left(\psi_{l} (g) - \psi_{l-1}(g),\frac{\xi \circ \epsilon }{\mathcal{U}_n} \right)  \ge \frac{a_l  \cdot  t}{8 \mathcal{U}_n} \,\bigg|\, \{x_i\}_{i=1}^n,\|\epsilon\|_{\infty} \leq \mathcal{U}_n\right). 
   \end{array}
\end{equation}

However,
$$ 
\|\psi_{l-1} (g) -\psi_l (g) \|_n \le \| \psi_{l-1} (g)- g \|_n + \| \psi_l (g) -g \|_n \le 2^{-l+1} \lambda,  
$$
and
\[
\left|\mathcal{T}_{l-1}\right| \le \left|\mathcal{T}_{l}\right| \le \exp(\eta_n(2^{-l} \lambda )). 
\]
Hence, by the sub-Gaussian tail inequality, 
\begin{equation}
    \label{eqn:e5}
    \arraycolsep=1.4pt\def\arraystretch{1.6}
       \begin{array}{l}
       \displaystyle   \mathbb{P}\left( \underset{g\in\mathcal F \cap \mathcal{B}_n(\lambda)}{\sup}  D\left(g -\psi_0(g),\frac{\xi \circ \epsilon}{\mathcal{U}_n} \right) \ge  \frac{t}{8 \mathcal{U}_n}  \,\bigg|\, \{x_i\}_{i=1}^n\right)  \\
       \leq     \displaystyle \sum_{l=1}^{\infty}  \sum_{  \psi_{l} (g) \in \mathcal{C}_l, \psi_{l-1}(g) \in \mathcal{C}_{l-1} }  \exp\left(-\frac{  a_l^2 t^2 n }{ 128 \mathcal{U}_n^2 (2^{-l+1} \lambda )^2  }\right)  \\    \leq  \displaystyle \sum_{l=1}^{\infty}  \exp \left(2 \eta_n(2^{-l} \lambda )\right) \exp\left(-\frac{ a_l^2 t^2 2^{2l} n }{ 2^{10} \mathcal{U}_n^2 \lambda^2  }\right)   \\
        \leq  \displaystyle \sum_{l=1}^{\infty}  \exp\left(-\frac{ a_l^2 t^2 2^{2l} n }{ 2^{11} \mathcal{U}_n^2 \lambda^2  }\right),   \\
   \end{array}
\end{equation}
provide that 
\[
2   \eta_n(2^{-l}\lambda) \,\leq \, \frac{ a_l^2 t^2 2^{2l} n }{ 2^{11} \mathcal{U}_n^2 \lambda^2 }
\]
which holds if 
\[
a_l^2 \,\geq\,  \frac{2^{12}  \mathcal{U}_n^2 \lambda^2 \eta_n(2^{-l}\lambda) }{  t^2 2^{2l}n }.
\]
However, $t  \geq  C \mathcal{U}_n\lambda \sqrt{\eta_n(\lambda)/n} $. Hence,  we can take 
$a_l^2 = C_1  2^{-2l} \eta_n(2^{-l}\lambda)/ \eta_n(\lambda)$ for some constant $C_1>0$,  and thus
\begin{equation}
    \label{eqn:e3}
\begin{array}{lll}
     \mathbb{P}\left( \underset{g\in\mathcal F \cap \mathcal{B}_n(\lambda)}{\sup} \{ D(g -\psi_0(g),\frac{\xi \circ \epsilon}{  \mathcal{U}_n })\} \ge \frac{t}{8 \mathcal{U}_n}  \,\bigg|\, \{x_i\}_{i=1}^n, \|\epsilon\|_{\infty} \leq \mathcal{U}_n \right) 
     &\,\leq\,&\displaystyle  \sum_{l=1}^{\infty}   \exp\left(-\frac{ C_1^2  t^2 n\, \eta_n( 2^{-l}\lambda )   }{ 2^{11} \mathcal{U}_n^2 \lambda^2 \eta_n(\lambda )  }\right)  \\
\end{array}
\end{equation}
for a positive constant $c^{\prime}>0$.

\end{proof}

\begin{lemma}
\label{lem3}
Let $\mathcal F$ be  a class of functions such that
\begin{align}\label{eq:metric entropy assumption in epsilon 1}
    \log \mathcal N \left(\delta,  \mathcal F,  \left\| \cdot\right\|_n \right) \le \eta_n(\delta), \,\,\,\forall \delta \in (0,1),
\end{align}
for a decreasing function  $\eta_n \,:\, (0,1)  \,\rightarrow \mathbb{R}_{+} $.

Let $a \in (0,\tau]$  and $s>0$  be such that 
\[
C^2 s^2 \eta_n( a )  \,\geq \, 8.
\]
 Suppose that $\mathcal{F}  \subset B_{\infty}(\tau) $  for $\tau \in (0,1)$. Then
\[
\begin{array}{lll}
\displaystyle    \mathbb{P}\left( \underset{g \in \mathcal{F} }{\sup} \,\frac{D(g,\epsilon) }{ \|g\|_n \sqrt{ \eta_n( \tau^{-1}\|g\|_n a ) /n   }}   \geq Cs \,\bigg|\,\{x_i\}_{i=1}^n \right)    
    & \leq & \displaystyle4\sum_{l=0}^{\infty} \sum_{l^{\prime}=1 }^{\infty} \exp\left( -  C_1  \eta_n( 2^{-l-l^{\prime}  } \tau  )\right) \,+\,\\
     & &\displaystyle 4\sum_{l=0}^{\infty}    \exp\left(- C_2 \eta_n(2^{-l}a) \right) \,+\,\mathbb{P}(\|\epsilon\|_{\infty} > s  \,|\, \left\{x_i \right\}_{i=1}^n  ) .
\end{array}
\]
for positive constants $C_1$ and $C_2$,
provided that 
\[
 \sum_{l^{\prime}=1 }^{\infty} \frac{ \eta_n(2^{-l-l^{\prime}}\tau   )  }{   2^{2 l^{\prime}   }  \eta_n( 2^{-l}a )  } \,\leq\, 1
\]
for all $l \in \mathbb{N}$.
\end{lemma}

\begin{proof}

Let  $g \in \mathcal{F}  \subset B_{\infty}(\tau) \subset \mathcal{B}_n(\tau)$. Then 
\[
\begin{array}{lll}
   \displaystyle\mathbb{P}\left(  \frac{D(g,\epsilon) }{ \|g\|_n \sqrt{ \eta_n( \tau^{-1}\|g\|_na ) /n   }}   \leq \frac{ Cs}{2} \,\bigg|\,\{x_i\}_{i=1}^n \right)    &  = & \displaystyle  1- \mathbb{P}\left(  \frac{D(g,\epsilon) }{ \|g\|_n \sqrt{ \eta_n(\tau^{-1}\|g\|_na  )/n    }} >\frac{ Cs}{2} \,\bigg|\,\{x_i\}_{i=1}^n \right)  \\
     &\geq &\displaystyle 1 - \frac{ 4  \text{var}\left(     \frac{D(g,\epsilon) }{ \|g\|_n \sqrt{ \eta_n( \tau^{-1}\|g\|_na  )/n    }}  \,\bigg|\,  \{ x_i\}_{i=1}^n\right) }{ C^2 s^2}  \\
      & = & \displaystyle 1 - \frac{ 4  } {  C^2 s^2 \eta_n( \tau^{-1}\|g\|_na  ) }\\
          & \geq & \displaystyle 1 - \frac{ 4  } {  C^2 s^2 \eta_n( 1 \cdot a )}\\
       & \geq& \displaystyle \frac{1}{2}
\end{array}
\]
provided that 
\[
 C^2 s^2 \eta_n( a )  \,\geq \, 8.
\]
Therefore,
\begin{equation}
    \label{eqn:e10}
    \begin{array}{l}
   \displaystyle\mathbb{P}\left( \underset{g \in \mathcal{F} }{\sup} \,\frac{D(g,\epsilon) }{ \|g\|_n \sqrt{ \eta_n( \tau^{-1}\|g\|_n a ) /n   }}   \geq Cs \,\bigg|\,\{x_i\}_{i=1}^n \right)    \\
     \leq   \displaystyle
   4 \mathbb{P}\left( \underset{g \in \mathcal{F} }{\sup} \,\frac{D(g,\epsilon,\xi) }{ \|g\|_n \sqrt{ \eta_n( \tau^{-1}\|g\|_na ) /n   }}   \geq \frac{C s}{4}\,\bigg|\,\{x_i\}_{i=1}^n \right)\\
   \leq \displaystyle  4 \mathbb{P}\left( \underset{g \in \mathcal{F} }{\sup} \,\frac{D(g,\epsilon,\xi) }{ \|g\|_n \sqrt{ \eta_n( \tau^{-1}\|g\|_na ) /n   }}   \geq \frac{C s}{4}\,\bigg|\,\{x_i\}_{i=1}^n , \|\epsilon\|_{\infty}\leq s\right) \,+\, \mathbb{P}(\|\epsilon\|_{\infty} > s \,|\, \left\{x_i \right\}_{i=1}^n ) \\
      \leq \displaystyle  4 \mathbb{P}\left( \underset{g \in \mathcal{F} }{\sup} \,\frac{D(g,\frac{\xi\circ \epsilon }{s}  ) }{ \|g\|_n \sqrt{ \eta_n( \tau^{-1}\|g\|_na ) /n   }}   \geq \frac{C }{4}\,\bigg|\,\{x_i\}_{i=1}^n , \|\epsilon\|_{\infty}\leq s\right) \,+\, \mathbb{P}(\|\epsilon\|_{\infty} > s\,|\, \left\{x_i \right\}_{i=1}^n ) \\
        \leq \displaystyle  4 \mathbb{P}\left( \exists l \in \mathbb{N} \,\,\,\text{s.t}\,\,\, \underset{ g \in \mathcal{F}\, \,:\, 2^{-l-1}\tau \leq \|g\|_n \leq 2^{-l} \tau } {\sup} \,\frac{D(g,\frac{\xi\circ \epsilon }{s}  ) }{ \|g\|_n \sqrt{ \eta_n( \tau^{-1}\|g\|_n a ) /n   }}   \geq \frac{C }{4 }    \,\bigg|\,\{x_i\}_{i=1}^n , \|\epsilon\|_{\infty}\leq s\right) \\
        \,+\, \mathbb{P}(\|\epsilon\|_{\infty} > s\,|\, \left\{x_i \right\}_{i=1}^n ) \\
        \leq \displaystyle\sum_{l=0}^{\infty}  4 \mathbb{P}\left( \underset{ g \in \mathcal{F}\, \,:\, 2^{-l-1}\tau \leq \|g\|_n \leq 2^{-l} \tau } {\sup} \,\frac{D(g,\frac{\xi\circ \epsilon }{s}  ) }{ \|g\|_n \sqrt{ \eta_n( \tau^{-1}\|g\|_na ) /n   }}   \geq \frac{C }{4 }    \,\bigg|\,\{x_i\}_{i=1}^n , \|\epsilon\|_{\infty}\leq s\right)  \,+\, \\
         \,\,\,\,\,\,\,\mathbb{P}(\|\epsilon\|_{\infty} > s\,|\, \left\{x_i \right\}_{i=1}^n ) \\
          \leq \displaystyle\sum_{l=0}^{\infty}  4 \mathbb{P}\left( \underset{ g \in \mathcal{F}\, \,:\, 2^{-l-1}\tau \leq \|g\|_n \leq 2^{-l} \tau } {\sup} \,\frac{D(g,\frac{\xi\circ \epsilon }{s}  ) }{ \tau 2^{-l-1} \sqrt{ \eta_n( 2^{-l}a ) /n   }}   \geq \frac{C }{4 }    \,\bigg|\,\{x_i\}_{i=1}^n , \|\epsilon\|_{\infty}\leq s\right)  \,+\,\\
           \,\,\,\,\,\,\,\mathbb{P}(\|\epsilon\|_{\infty} > s\,|\, \left\{x_i \right\}_{i=1}^n ) \\
         \leq \displaystyle\sum_{l=0}^{\infty}  4 \mathbb{P}\left( \underset{ g \in \mathcal{F}\, \,:\, \|g\|_n \leq 2^{-l} \tau } {\sup} \,D(g,\frac{\xi\circ \epsilon }{s}  )  \geq C \tau 2^{-l-3} \sqrt{ \eta_n( 2^{-l}a ) /n   } \,\bigg|\,\{x_i\}_{i=1}^n , \|\epsilon\|_{\infty}\leq s\right)  \,+\, \\
           \,\,\,\,\,\,\,\mathbb{P}(\|\epsilon\|_{\infty} > s\,|\, \left\{x_i \right\}_{i=1}^n ) \\  
\end{array}
\end{equation}
where the first inequality follows by   Lemma 2.3.7 in \cite{van1996weak}. 

Then, take   $t  =  C_2  \cdot (2^{-l}\tau) \cdot\sqrt{ \eta_n(2^{-l}a ) /n} $ for $C_2 = C/8$. Hence, proceedding as in Equation (\ref{eqn:e2}), (\ref{eqn:e5}) and (\ref{eqn:e3}), 
\begin{equation}
    \label{eqn:e6}
    \begin{array}{l}
 \displaystyle   \mathbb{P}\left( \underset{g\in\mathcal F \cap \mathcal{B}_n(2^{-l }\tau )}{\sup}  D(g ,\frac{\xi \circ \epsilon}{ s }) \ge C   \tau 2^{-l-3} \sqrt{ \eta_n( 2^{-l}a ) /n   }  \,\bigg|\, \{x_i\}_{i=1}^n, \|\epsilon\|_{\infty} \leq s \right) \\
 \displaystyle    \,\leq\,\exp \left(  \eta_n(2^{-l}\tau) \right) \exp\left(- \frac{ct^2n}{ (2^{-l} \tau )^2} \right) \,+\, \sum_{l^{\prime}=1 }^{\infty}  \exp \left(2 \eta_n(2^{-l^{\prime} } 2^{-l}\tau )\right)  \exp\left(-\frac{  C_3 b_{l^{\prime} }^2 t^2 2^{2l^{\prime} } n }{  (2^{-l}\tau )^2  }\right)    \\
\end{array}
\end{equation}
for a sequence   $b_{l^{\prime}} \ge 0 $ and $\sum_{l^{\prime}=1}^\infty b_{l^{\prime} }\le 1 $,  and for a constant $C_3>0$. Hence, for $C_2 >0$ is large enough then
\begin{equation}
    \label{eqn:e7}
    \begin{array}{lll}
 \displaystyle    \exp \left(  \eta_n(2^{-l}\tau) \right) \exp\left(- \frac{ct^2n}{(2^{-l} \tau )^2} \right) &\leq&\displaystyle   \exp \left(  \eta_n(2^{-l}a) \right) \exp\left(- \frac{ct^2n}{ (2^{-l} a )^2} \right)\\
   & \leq&\displaystyle  \exp\left(- C_4 \eta_n(2^{-l}a) \right)\\
\end{array}
\end{equation}
for $C_4>0$. Furthermore,
\[
\begin{array}{lll}
  \displaystyle   \exp \left(2 \eta_n(2^{-l^{\prime} } 2^{-l}\tau)\right)  \exp\left(-\frac{  C_3 b_{l^{\prime} }^2 t^2 2^{2l^{\prime} } n }{  (2^{-l}\tau )^2  }\right)  & \leq&
  \displaystyle \exp\left(-\frac{  C_3 b_{l^{\prime} }^2 t^2 2^{2l^{\prime} } n }{ 2 (2^{-l}\tau )^2  }\right) \\ 
\end{array}
\]
which holds provided that 
\begin{equation}
    \label{eqn:e8}
    2 \eta_n(2^{-l^{\prime} } 2^{-l}\tau)  \,\leq \, \frac{  C_3 b_{l^{\prime} }^2 t^2 2^{2l^{\prime} } n }{ 2 (2^{-l}\tau )^2  }
\end{equation}
Hence, we can choose
\[
b_{ l^{\prime} }\,=\, \sqrt{ C_4 \frac{ \eta_n(2^{-l-l^{\prime}}\tau   )  }{   2^{2 l^{\prime}   }  \eta_n( 2^{-l}a )  }  }  
\]
so that (\ref{eqn:e8}) holds and 
\[
 \begin{array}{lll}
  \displaystyle    \frac{  C_3 b_{l^{\prime} }^2 t^2 2^{2l^{\prime} } n }{ 2 (2^{-l}\tau )^2  }  &=        &  C_5  \eta_n( 2^{-l-l^{\prime} } \tau  )
 \end{array}
\]
for some constant $C_5>0$. Therefore,
\begin{equation}
    \label{eqn:e9}
     \begin{array}{l}
     \displaystyle   \mathbb{P}\left( \underset{g\in\mathcal F \cap \mathcal{B}_n(2^{-l }\tau )}{\sup}  D(g ,\frac{\xi \circ \epsilon}{ s }) \ge C  \tau 2^{-l-3} \sqrt{ \eta_n( 2^{-l}a ) /n   }  \,\bigg|\, \{x_i\}_{i=1}^n, \|\epsilon\|_{\infty} \leq s\right) \\
 \displaystyle     \leq \,   \exp\left(- C_4 \tau^2 \eta_n(2^{-l}a) \right) \,+\,\sum_{l^{\prime}=1 }^{\infty} \exp\left( -  C_5  \eta_n( 2^{-l-l^{\prime}  } \tau  )\right).
 \end{array}
\end{equation}

The claim follows combining (\ref{eqn:e10}) with (\ref{eqn:e9}).

\end{proof}

\begin{lemma}
\label{lem5}
Suppose that $\mathcal{F}  \subset B_{\infty}(\tau) $  for $\tau \in (0,1)$. Let $a\in (0,\tau^2]$ and $s\geq 1$. Then letting $\xi$ be a vector whose entries are independent Radamacher variables, we have that
\[
\begin{array}{lll}
\displaystyle    \mathbb{P}\left( \underset{g \in \mathcal{F} }{\sup} \,\frac{D(g^2,\xi) }{ \|g^2\|_n \sqrt{ \eta_n( \tau^{-2}\|g^2\|_n a ) /n   }}   \geq Cs \,\bigg|\,\{x_i\}_{i=1}^n \right)    
    & \leq & \displaystyle4\sum_{l=0}^{\infty} \sum_{l^{\prime}=1 }^{\infty} \exp\left( -  C_1  s^2 \eta_n( 2^{-l-l^{\prime}  } \tau^2  )\right) \,+\,\\
     & &\displaystyle 4\sum_{l=0}^{\infty}    \exp\left(- C_2  s^2 \eta_n(2^{-l} a) \right) 
\end{array}
\]
for positive constants $C$, $C_1$ and $C_2$, 
provided that
\[
 \sum_{l^{\prime}=1 }^{\infty} \frac{ \eta_n(2^{-l-l^{\prime}}\tau   )  }{   2^{2 l^{\prime}   }  \eta_n( 2^{-l}a )  } \,\leq\, 1
\]
for all $l \in \mathbb{N}$.
\end{lemma}

\begin{proof}
    First, notice that $g \in \mathcal{F}$ implies that $\|g^2\|_{\infty}\leq  \tau^2 $. Hence, $\mathcal{G} \,:=\, \{ g^2 \,:\, g\in \mathcal{F}\} \subset \mathcal{F} $ satisfies that $\mathcal{G}\subset B_{\infty}(\tau^2)$. Moreover, if $g,f \in \mathcal{G}$, then
    \[
   \begin{array}{lll}
          \| f-g \|_n &\,=\,& \displaystyle \sqrt{\frac{1}{n} \sum_{i=1}^n  (f(x_i) - g(x_i))^2 } \\
          &\,=\,&  \displaystyle  \sqrt{ \frac{1}{n} \sum_{i=1}^n  (\sqrt{f}(x_i) - \sqrt{g}(x_i))^2(\sqrt{f(x_i)}+ \sqrt{g(x_i)}  )^2 } \\
          &\,\leq \,& \displaystyle  \sqrt{ \frac{1}{n} \sum_{i=1}^n  (\sqrt{f}(x_i) - \sqrt{g}(x_i))^2( \tau + \tau  )^2 } \\
           & =& 2\tau  \| \sqrt{f} - \sqrt{g}\|_n.
    \end{array}
    \]
    Therefore, 
    \[
 \log \mathcal N \left(\delta,  \mathcal G,  \left\| \cdot\right\|_n \right) \,\leq \, \log \mathcal N \left(\delta/\tau,  \mathcal F,  \left\| \cdot\right\|_n \right). 
    \]
    Next, notice that proceeding as in Equation (\ref{eqn:e10}), with a peeling argument, we obtain that 

\begin{equation}
    \label{eqn:e14}
    \begin{array}{l}
   \displaystyle\mathbb{P}\left( \underset{g \in \mathcal{F} }{\sup} \,\frac{D(g^2,\xi) }{  \|g^2\|_n \sqrt{ \eta_n( \tau^{-2}\|g^2\|_na ) /n   }}   \geq s \,\bigg|\,\{x_i\}_{i=1}^n \right)    \\
           \leq \displaystyle\sum_{l=0}^{\infty}  4 \mathbb{P}\left( \underset{ g \in \mathcal{F}\, \,:\, \|g^2\|_n \leq 2^{-l} \tau^2 } {\sup} \,D(g,\xi )  \geq C s\tau^2 2^{-l-3} \sqrt{ \eta_n( 2^{-l}a ) /n   } \,\bigg|\,\{x_i\}_{i=1}^n \right). \\  
\end{array}
\end{equation}

Then, taking   $t  =  C_2  \cdot s (2^{-l}\tau^2) \cdot\sqrt{ \eta_n(2^{-l}a ) /n} $ for $C_2= C/8$, as in Equation (\ref{eqn:e6}), we obtain
\[
\begin{array}{l}
\displaystyle    \mathbb{P}\left( \underset{ g \in \mathcal{F}\, \,:\, \|g^2\|_n \leq 2^{-l} \tau^2 } {\sup} \,D(g,\xi )  \geq C s\tau 2^{-l-3} \sqrt{ \eta_n( 2^{-l}a ) /n   } \,\bigg|\,\{x_i\}_{i=1}^n \right) \\
 \displaystyle    \leq \exp \left(  \eta_n(2^{-l}\tau^2) \right) \exp\left(- \frac{ct^2n}{ (2^{-l} \tau^2 )^2} \right) \,+\, \sum_{l^{\prime}=1 }^{\infty}  \exp \left(2 \eta_n(2^{-l^{\prime} } 2^{-l}\tau^2 )\right)  \exp\left(-\frac{  C_3 b_{l^{\prime} }^2 t^2 2^{2l^{\prime} } n }{  (2^{-l}\tau^2 )^2  }\right)    \\
\end{array}
\]
for a constant $C_3>0$, and for arbitrary  $b_{l^{\prime}} \ge 0 $ and $\sum_{l^{\prime}=1}^\infty b_{l^{\prime} }\le 1 $. Then if $C_2>0$ is large enough, we obtain that 
    \[
  \exp \left(  \eta_n(2^{-l}\tau) \right) \exp\left(- \frac{ct^2n}{ (2^{-l} \tau^2 )^2} \right) \,\leq\, \exp\left( -C_4 s^2 \eta_n(2^{-l}a  ) \right)
    \]
    for a constant $C_4>0$.   
    
    Moreover,
    \[
     \exp \left(2 \eta_n(2^{-l^{\prime} } 2^{-l}\tau^2)\right)  \exp\left(-\frac{  C_3 b_{l^{\prime} }^2 t^2 2^{2l^{\prime} } n }{  (2^{-l}\tau^2 )^2  }\right)  \leq
  \displaystyle \exp\left(-\frac{  C_3 b_{l^{\prime} }^2 t^2 2^{2l^{\prime} } n }{ 2 (2^{-l}\tau^2 )^2  }\right) 
    \]
provided that 
    \[
b_{ l^{\prime} }\,=\, \sqrt{ C_4 \frac{ \eta_n(2^{-l-l^{\prime}}\tau^2   )  }{   2^{2 l^{\prime}   }  \eta_n( 2^{-l}a )  }  }  
\]
so that
\[
 \begin{array}{lll}
  \displaystyle    \frac{  C_3 b_{l^{\prime} }^2 t^2 2^{2l^{\prime} } n }{ 2 (2^{-l}\tau^2 )^2  }  &=        &  C_5 s^2  \eta_n( 2^{-l-l^{\prime} } \tau^2  )
 \end{array}
\]
for a constant $C_5>0$.  The claim then follows combining all the pieces. 

\end{proof}

\begin{definition}
    \label{def1}
    Given a function class $\mathcal{F}$ with $\mathcal{F} \subset   B_{\infty}(1)$, we call $\delta_n >0$ a critical radius for $\mathcal{F}$ if 
    \[
      \mathbb{E}\left( \underset{f \in \mathrm{star}( \mathcal{F})\,:\, \|f\|_n \leq \delta_n   } {\sup } \,  \frac{1}{n}\sum_{i=1}^n \xi_i f(x_i) \,\bigg|\, \{ x_i\}_{i=1}^n  \right) \,\leq\, \delta_n^2,
    \]
    where $\xi_1,\ldots,\xi_n$ are independent Radamacher variables independent of $\{x_i\}_{i=1}^n$ and where 
    $\mathrm{star}( \mathcal{F})$  is  defined as 
    \[
      \mathrm{star}( \mathcal{F}) \,:=\, \{  \lambda f \,:\, \lambda \in [0,1], f\in \mathcal{F}\}.
    \]
\end{definition}

\begin{lemma}
    \label{lem9}
    Let $\mathcal{F}$ be class of functions such that $f\in \mathcal{F} \subset B_{\infty}(1)$. Then, if $\delta_n>0$ is a critical radius for $\mathcal{F}$, for all $t \geq \delta_n$ it holds that 
    \[
     \mathbb{P}\left(   \underset{ f \in \mathcal{F}  }{\sup} \,\frac{ \vert   \|f\|_n^2 - \|f\|_{ \lt }^2  \vert }{  \frac{1}{2} \|f\|_{ \lt } + \frac{t^2}{2}  }  \geq 1\right) \,\leq\, c_2 \exp\left( - c_3  t^2 n  \right)
    \]
    for positive constants $c_2$ and $c_3$. Moreover, $\delta_n$ is a critical radius if 
       \begin{equation}
    \label{eqn:crit0}
  \frac{12}{\sqrt{n}}\int_{ \delta_n^2/48 }^{\delta_n}    \sqrt{ \log \mathcal N \left(t/2,  \mathcal{F} ,\left\| \cdot\right\|_n \right)   }    dt
 \,+\,   \frac{12 \delta_n }{\sqrt{n}}\sqrt{ c_1 \log( 48/\delta_n^2 )}    \,\leq\, \frac{3\delta_n^2}{4},
    \end{equation}
    or if 
    \begin{equation}
    \label{eqn:crit}
    \frac{12 \delta_n }{\sqrt{n}} \sqrt{    \log \mathcal N \left(\delta_n^2/96,   \mathcal{F} ,\left\| \cdot\right\|_n \right) } \,+\, \frac{12 \delta_n}{\sqrt{n} } \sqrt{ c_1 \log( 48/\delta_n^2 )  } \,\leq\, \frac{3 \delta_n^2}{4}.
    \end{equation}
        In particular, there exists $c_1 >0$ constant such that
    \[
    \delta\,\geq\, c_1 \left[ \sqrt{ \frac{ \log n}{n} }  \,+\,  \sqrt{   \frac{ \log \mathcal N \left(1/(24n),   \mathcal{F} ,\left\| \cdot\right\|_n \right) }{n}} \right] 
    \]
   is a critical radius for $\mathcal{F}$.
    
\end{lemma}
\begin{proof}
    Let  $\delta >0$ and let 
    $f_1,\ldots, f_N $ be a $\delta/2$-net of $\mathcal{F}$ and $a_1,\ldots, a_M$ be a $\delta/(2\tau)$-net of $[0,1]$. Let $g \in  \mathrm{star}( \mathcal{F}) $. Then there exists $f \in \mathcal{F}$ and $\lambda \in [0,1]$ such that  $g = \lambda f$. Let $i\in \{1,\ldots, N\}$ and $j \in \{1,\ldots,M\}$ be such that  
    \[
       \| f - f_i\|_n \,\leq \, \frac{\delta}{2},\,\,\,\,\text{and}\,\,\,\,\vert \lambda - \lambda_j \vert  \,\leq\, \frac{\delta}{2 \tau}.
    \]
    Then,
    \[
     \| g - \lambda_j f_i\|_n \, =\, \| \lambda f  - \lambda_j f_i\|_n \leq  \| \lambda f  - \lambda f_i\|_n  \,+\,  \| \lambda f_i  - \lambda_j f_i\|_n \,\leq\,\| f  - f_i\|_n \,+\,   \vert \lambda -\lambda_j\vert \cdot\|   f_i\|_n.
    \]
    Hence,
    \[
        \| g - \lambda_j f_i\|_n \,\leq\, \frac{\delta}{2} \,+\, \vert \lambda -\lambda_j\vert \cdot\|   f_i\|_{\infty} \,\leq\, \frac{\delta}{2} \,+\, \vert \lambda -\lambda_j\vert \cdot\tau \,\leq\, \delta.
    \]
    Therefore, 
    \[
    \mathcal N \left(\delta, \mathrm{star}( \mathcal{F}), \left\| \cdot\right\|_n \right) \,\leq\,    \mathcal N \left(\delta/2,  \mathcal{F}, \left\| \cdot\right\|_n \right) \cdot    \mathcal N \left(\delta,  [0,1], \vert  \cdot\vert \right)
    \]
    and so 
    \begin{equation}
        \label{eqn:entropy}
        \log  \mathcal N \left(\delta,  \mathrm{star}( \mathcal{F}) ,\left\| \cdot\right\|_n \right) \,\leq\,  \log  \mathcal N \left(\delta/2,\mathcal{F}, \left\| \cdot\right\|_n \right)\,+\, c_1\log(\tau /\delta )
    \end{equation}
    for a positive constant $c_1$.

Next notice that by Dudley's inequality, for $\xi_1,\ldots,\xi_n$ independent Rademacher random variables, we have that 
\[
\arraycolsep=1.4pt\def\arraystretch{1.6}
     \begin{array}{lll}
       \displaystyle    \mathbb{E}\left( \underset{f \in \mathrm{star}( \mathcal{F})\,:\, \|f\|_n \leq \delta   } {\sup } \,  \frac{1}{n}\sum_{i=1}^n \xi_i f(x_i) \,|\, \{ x_i\}_{i=1}^n  \right) & \leq&  \displaystyle  12\underset{ 0<\alpha< \delta }{\inf}\left[  \alpha +  \frac{1}{\sqrt{n} }  \int_{\alpha}^{ \delta } \sqrt{ \log \mathcal N \left(t,  \mathrm{star}( \mathcal{F}) ,\left\| \cdot\right\|_n \right)   }    dt \right]\\
        & \leq&  \displaystyle  12\underset{ 0<\alpha< \delta }{\inf}\bigg[  \alpha +  \frac{1}{\sqrt{n} }  \int_{\alpha}^{ \delta } \sqrt{ \log \mathcal N \left(t/2,   \mathcal{F} ,\left\| \cdot\right\|_n \right)   }    dt \,+\,\\
         & &\displaystyle \,\,\,\,\,\,\,\,\,\,\,\,\,\,\,\,\,\,\,\,\,\,\,\,\,\,\,\frac{1}{\sqrt{n}}\int_{\alpha}^{ \delta}  \sqrt{ c_1 \log(1/t) }  dt \bigg]\\
     \end{array}
\]
where the second inequality follows from (\ref{eqn:entropy}). Hence, by taking $\alpha = \delta^2/48$, we obtain that 
\begin{equation}
    \label{eqn:e30}
    \arraycolsep=1.4pt\def\arraystretch{1.6}
      \begin{array}{lll}
       \displaystyle    \mathbb{E}\left( \underset{f \in \mathrm{star}( \mathcal{F})\,:\, \|f\|_n \leq \delta   } {\sup } \,  \frac{1}{n}\sum_{i=1}^n \xi_i f(x_i) \,\bigg|\, \{ x_i\}_{i=1}^n  \right) & \leq&   \displaystyle\frac{\delta^2}{4} \,+\, \frac{12}{\sqrt{n}}\int_{ \delta^2/48 }^{\delta}    \sqrt{ \log \mathcal N \left(t/2,  \mathcal{F} ,\left\| \cdot\right\|_n \right)   }    dt
 \,+\,  \\
  & &\displaystyle\frac{12 \delta }{\sqrt{n}}\sqrt{ c_1 \log( 48/\delta^2 )}\\
        &\leq  &\displaystyle  \frac{\delta^2}{4} \,+\,  \frac{12 \delta }{\sqrt{n}} \sqrt{    \log \mathcal N \left(\delta^2/96,   \mathcal{F} ,\left\| \cdot\right\|_n \right) } \,+\, \frac{12 \delta}{\sqrt{n} } \sqrt{ c_1 \log( 48/\delta^2 )  }. 
     \end{array}
\end{equation}
Therefore,  if $\delta_n$ satisfies (\ref{eqn:crit0}) or (\ref{eqn:crit}), we obtain that 
\[
 \mathbb{E}\left( \underset{f \in \mathrm{star}( \mathcal{F})\,:\, \|f\|_n \leq \delta_n   } {\sup } \,  \frac{1}{n}\sum_{i=1}^n \xi_i f(x_i) \,\bigg|\, \{ x_i\}_{i=1}^n  \right) \,\leq\, \delta_n^2,
\]
and so $\delta_n$ is a critical radius for $\mathcal{F}$.
Moreover, if 
\begin{equation}
    \label{eqn:cond}
    \delta/2 \geq \sqrt{ \frac{C \log n}{n} }  \,+\,  \frac{24 }{\sqrt{n}} \sqrt{    \log \mathcal N \left(1/(24n),   \mathcal{F} ,\left\| \cdot\right\|_n \right) } 
\end{equation}
for a large enough constant $C>0$, then $\delta^2/4 \geq \frac{1}{n}$ and so 
 \begin{equation}
     \label{eqn:e31}
      \frac{12 \delta }{\sqrt{n}} \sqrt{    \log \mathcal N \left(\delta^2/96,   \mathcal{F} ,\left\| \cdot\right\|_n \right) } \,\leq\,      \frac{12 \delta }{\sqrt{n}} \sqrt{    \log \mathcal N \left(1/(24n),   \mathcal{F} ,\left\| \cdot\right\|_n \right) }\,\leq \,\frac{\delta^2}{4}.
 \end{equation}
     Furthermore,  if (\ref{eqn:cond}) holds then 
     \begin{equation}
         \label{eqn:e32}
\arraycolsep=1.4pt\def\arraystretch{1.6}
          \begin{array}{lll}
    \displaystyle      \frac{12 \delta}{\sqrt{n} } \sqrt{ c_1 \log( 48/\delta^2 )  } &\leq& \displaystyle \frac{12 \delta}{\sqrt{n} } \sqrt{ c_1 \log( 48n/(4C \log n)))  } \\
         &\leq&\displaystyle    \frac{12 \delta}{\sqrt{n} } \sqrt{ c_1 \log(12 n )  } \\
         &\leq&  \displaystyle  \frac{\delta^2}{4}\\      
    \end{array}
     \end{equation}
     where the last inequality holds if $C$  is large enough.  Therefore, combining (\ref{eqn:e30}), (\ref{eqn:e31}) and (\ref{eqn:e32}), we obtain that 
     \[
     \mathbb{E}\left( \underset{f \in \mathrm{star}( \mathcal{F})\,:\, \|f\|_n \leq \delta   } {\sup } \,  \frac{1}{n}\sum_{i=1}^n \xi_i f(x_i) \,\bigg|\, \{ x_i\}_{i=1}^n  \right)  \,\leq\, \delta^2.
     \]
     Finally, for any ciritical radius $\delta>0$, from Theorem 14.1 and Proposition 14.25 in \cite{wainwright2019high},
     \[
     \arraycolsep=1.4pt\def\arraystretch{1.6}
      \begin{array}{lll}
        \mathbb{P}\left(  \underset{f \in  \mathcal{F}  } {\sup } \,   \frac{ \vert \|f\|_{n}^2 - \|f\|_{\lt}^2 \vert }{ \frac{1}{2}\|f\|_{\lt}^2   + \frac{t^2}{2}  } \,\geq \, 1    \right )    & \leq  &\mathbb{P}\left(  \underset{f \in \mathrm{star}( \mathcal{F})   } {\sup } \,   \frac{ \vert \|f\|_{n}^2 - \|f\|_{\lt}^2 \vert }{ \frac{1}{2}\|f\|_{\lt}^2   + \frac{t^2}{2}  } \,\geq \, 1    \right )\\
         & \leq& c_2 \exp\left( - c_3  t^2 n  \right)
       \end{array}
     \]
     for positive constants $c_2$ and $c_3$ and any $t$ satisfying $t\geq \delta$.
\end{proof}

\section{Proof of Theorem \ref{thm1_v2} }
\label{proof_mean_upper}




\begin{theorem}
    \label{thm1}
Suppose that   $\bar{f} \in \mathcal{F}$ is such that 
 \[
  \| \bar{f} -f^*\|_{\infty} \,\leq\,\sqrt{\phi_n},
 \]
 so that $\phi_n$ is the approximating error. Suppose that $\mathcal{A}_n$ is chosen to satisfy
 \[
 \mathcal{A}_n \geq 8\max\{  \| f^* \|_{\infty} +  \mathcal{U}_n , 8\|f^*\|_{\infty} ,8\sqrt{\phi_n}\}.
 \]
 Moreover, let $\mathcal{F}_{\mathcal{A}_n}   \,:=\, \{ f_{\mathcal{A}_n}/\mathcal{A}_n \,:\,  f \in \mathcal{F}\}$ and  assume that
 \[
 \log \mathcal N \left( \delta,  \mathcal{F}_{\mathcal{A}_n},  \left\| \cdot\right\|_n \right)\,\leq\, \eta_n(\delta)
 \]
 for some decreasing function $\eta_n \,:\, (0,1) \rightarrow \mathbb{R}_{\geq 0}$.  If 
\begin{equation}
    \label{eqn:entropy_0}
    \underset{n
 \rightarrow \infty}{\lim} \,\left[  \sum_{l=0}^{\infty} \sum_{l^{\prime}=1 }^{\infty} \exp\left( -  C_1  \eta_n( 2^{-l-l^{\prime} -1}  )\right) \,+\, \sum_{l=0}^{\infty}    \exp\left(- C_2 \eta_n(2^{-l-1}) \right) \,+\,\mathbb{P}(\|\epsilon\|_{\infty} > \mathcal{U}_n)    \right]\,=\,0,
\end{equation}
for some constants $C_1,C_2>0$ and
\begin{equation}
    \label{eqn:cond_sum}
   \underset{l  \in \mathbb{N} }{\sup}  \sum_{l^{\prime}=1 }^{\infty} \frac{ \eta_n(2^{-l-l^{\prime}}   )  }{   2^{2 l^{\prime}   }  \eta_n( 2^{-l} )  } \,\leq\, 1,
\end{equation}
then
\begin{equation}
    \label{eqn:claim1}
    \| f^*-  \hat{f}_{ \mathcal{A}_n } \|_n^2    \,=\, o_{ \mathbb{P} }\left(   \phi_n\,+\, \frac{ \mathcal{U}_n^2 \eta_n( \delta_n )    }{n}  \,+\, \mathcal{A}_n \delta_n^2   \right), \\
\end{equation}
where $\delta_n$ is a critical radius of $\mathcal{F}_{\mathcal{A}_n}$.
Moreover, 
\begin{equation}
    \label{eqn:claim2}
    \| f^*-  \hat{f}_{ \mathcal{A}_n } \|_{\lt}^2    \,=\, o_{ \mathbb{P} }\left( \phi_n\,+\, \frac{ \mathcal{U}_n^2 \eta_n( \delta_n )  }{n}   \,+\, \mathcal{A}_n \delta_n^2  \right). \\
\end{equation}
\end{theorem}

\begin{proof}
    First, notice that by Lemma \ref{lem1}, we have that $\bar{f} = \bar{f}_{ \mathcal{A}_n}$. Then,
    \[
   \| \hat{f}_{ \mathcal{A}_n  }  - f^* \|_{ \lt }^2\,\leq \, 2\| \hat{f}_{ \mathcal{A}_n  }  - \bar{f}_{ \mathcal{A}n } \|_{ \lt }^2\,+\,2\|  \bar{f}_{ \mathcal{A}n }-f^* \|_{ \lt }^2\,\leq \, 2\| \hat{f}_{ \mathcal{A}_n  }  - \bar{f}_{ \mathcal{A}n } \|_{ \lt }^2\,+\, \phi_n
    \]
    and similarly,
      \[
   \| \hat{f}_{ \mathcal{A}_n  }  - f^* \|_{ n}^2\\,\leq \, 2\| \hat{f}_{ \mathcal{A}_n  }  - \bar{f}_{ \mathcal{A}n } \|_{ n }^2\,+\, \phi_n.
    \]
    Next, suppose that the event
    \[
    \Omega_1\,:=\, \{ \|\epsilon\|_{\infty} \leq \mathcal{U}_n \},
    \]
    holds. Then
    \[
    \|y\|_{\infty} \,\leq\, \|f^*\|_{\infty} +  \|\epsilon\|_{\infty} \,\leq\,  \|f^*\|_{\infty} +  \mathcal{U}_n\,\leq\,  \mathcal{A}_n.
    \]
    Hence,
    \[
      \sum_{i=1}^n  (y_i - \hat{f}_{ \mathcal{A}_n  }(x_i))^2\,\leq\,  \sum_{i=1}^n  (y_i - \hat{f}(x_i))^2
     \]
     and by the basic inequality, 
     \[
   \sum_{i=1}^n  (y_i - \hat{f}(x_i))^2   \,\leq\,  \sum_{i=1}^n  (y_i - \bar{f}(x_i))^2 \,=\,  \sum_{i=1}^n  (y_i - \bar{f}_{ \mathcal{A}_n  } (x_i))^2
     \]
     Therefore,
     \[
      \sum_{i=1}^n  (y_i - \hat{f}_{ \mathcal{A}_n  }(x_i))^2\,\leq\, \sum_{i=1}^n  (y_i - \bar{f}_{ \mathcal{A}_n  } (x_i))^2
     \]
     which implies
      \begin{equation}
          \label{eqn:e25}
             \begin{array}{lll}
   \displaystyle      \frac{1}{2} \| \bar{f}_{ \mathcal{A}_n }-  \hat{f}_{ \mathcal{A}_n } \|_n^2 &       \leq     & \displaystyle \frac{1}{n} \sum_{i=1}^n ( \hat{f}_{  \mathcal{A}_n }(x_i)  -  \bar{f}_{  \mathcal{A}_n }(x_i)   )(f^*(x_i) + \epsilon_i   -  \bar{f}_{ \mathcal{A}_n }(x_i) )  \\
    & \leq& \displaystyle \| \hat{f}_{\mathcal{A}_n } -\bar{f}_{ \mathcal{A}_n } \|_n \cdot \sqrt{\phi_n}\,+\,\frac{1}{n} \sum_{i=1}^n ( \hat{f}_{  \mathcal{A}_n }(x_i)  -  \bar{f}_{  \mathcal{A}_n }(x_i)   )\epsilon_i \\
     & \leq& \displaystyle  \frac{1}{32}\| \hat{f}_{\mathcal{A}_n } -\bar{f}_{ \mathcal{A}_n } \|_n^2   \,+\, 8\phi_n\,+\, \frac{ 2\mathcal{A}_n   }{n} \sum_{i=1}^n   \frac{( \hat{f}_{  \mathcal{A}_n }(x_i)  -  \bar{f}_{  \mathcal{A}_n }(x_i)   )}{2\mathcal{A}_n }\epsilon_i, 
     \end{array}
      \end{equation}
     where the second inequality follows from Cauchy–Schwarz inequality. Next, let
     \[
        \mathcal{H}\,:=\,  \{  ( f_{ \mathcal{A}_n }  - g_{\mathcal{A}_n})/(2\mathcal{A}_n ) \,:\, f,g \in \mathcal{F}  \}. 
     \]
     Then $f\in \mathcal{H}$ implies $\|f\|_{\infty} \leq 1$ and 
     \[
        \log \mathcal N \left(\delta,  \mathcal H,  \left\| \cdot\right\|_n \right)\,\leq\,   2\log \mathcal N \left( \delta/2,  \mathcal{F}_{A_n},  \left\| \cdot\right\|_n \right) \,\leq\,  2\eta_n( \delta/2 ).
     \]
Hence, by Lemma \ref{lem3}, 
\begin{equation}
    \label{eqn:e24}
    \begin{array}{lll}
\displaystyle   \frac{ 2\mathcal{A}_n   }{n} \sum_{i=1}^n   \frac{( \hat{f}_{  \mathcal{A}_n }(x_i)  -  \bar{f}_{  \mathcal{A}_n }(x_i)   )}{2\mathcal{A}_n }\epsilon_i &\leq&  2 C \mathcal{A}_n \mathcal{U}_n  \| (\hat{f}_{\mathcal{A}_n} -  \bar{f}_{  \mathcal{A}_n }   )/(2\mathcal{A}_n) \|_n \cdot\\
&&\sqrt{ \eta_n( \| (\hat{f}_{\mathcal{A}_n}  -  \bar{f}_{  \mathcal{A}_n }   )/(4\mathcal{A}_n)\|_n  ) /n   }\\
 & \leq&  C \mathcal{U}_n\| \hat{f}_{\mathcal{A}_n} -  \bar{f}_{  \mathcal{A}_n }  \|_n \cdot\sqrt{ \eta_n( \| (\hat{f}_{\mathcal{A}_n}  -  \bar{f}_{  \mathcal{A}_n }   )/(4\mathcal{A}_n)\|_n  ) /n   }\\
\end{array}
\end{equation}
with probability at least
\[
 1\,-\, 4\sum_{l=0}^{\infty} \sum_{l^{\prime}=1 }^{\infty} \exp\left( -  C_1  \eta_n( 2^{-l-l^{\prime}-1  }  )\right) \,-\, 4\sum_{l=0}^{\infty}    \exp\left(- C_2 \eta_n(2^{-l-1}) \right) \,-\,\mathbb{P}(\|\epsilon\|_{\infty} > \mathcal{U}_n ).
\]
Let $\Omega_2$ be the event such that (\ref{eqn:e24}) holds. From now on suppose that $\Omega_2$ holds. Then, if 
\[
\| (\hat{f}_{\mathcal{A}_n}  -  \bar{f}_{  \mathcal{A}_n }   )/(4\mathcal{A}_n)\|_n \,\leq \, \delta_n
\]
we obtain 
\[
\| (\hat{f}_{\mathcal{A}_n} -  \bar{f}_{  \mathcal{A}_n }   )\|_n^2 \,\leq \, 16\mathcal{A}_n^2 \delta_n^2
\]
Suppose now that 
\[
\| (\hat{f}_{\mathcal{A}_n}  -  \bar{f}_{  \mathcal{A}_n }   )/(4\mathcal{A}_n)\|_n \,>\, \delta_n.
\]
Then (\ref{eqn:e25}) and (\ref{eqn:e24}) imply
\begin{equation}
    \label{eqn:e26}
            \begin{array}{lll}
   \displaystyle      \frac{15}{32} \| \bar{f}_{ \mathcal{A}_n }-  \hat{f}_{ \mathcal{A}_n } \|_n^2    & \leq& \displaystyle  8\phi_n\,+\,  C \mathcal{U}_n\| \hat{f}_{\mathcal{A}_n}  -  \bar{f}_{  \mathcal{A}_n }   \|_n \cdot\sqrt{ \eta_n( \delta_n ) /n   }\\
   \displaystyle &\le& \displaystyle  8\phi_n\,+\,  \frac{1}{32}\| \hat{f}_{\mathcal{A}_n}  -  \bar{f}_{  \mathcal{A}_n }   \|_n^2\,+\, \frac{ 8 C^2 \mathcal{U}_n^2 \eta_n( \delta_n) }{n}\\
     \end{array}
\end{equation}
and so 
\begin{equation}
    \label{eqn:e27}
            \begin{array}{lll}
   \displaystyle      \frac{14}{32} \| \bar{f}_{ \mathcal{A}_n }-  \hat{f}_{ \mathcal{A}_n } \|_n^2   
   \displaystyle &\le& \displaystyle  8\phi_n\,+\, \frac{ 8 C^2 \mathcal{U}_n^2 \eta_n( \delta_n) }{n}.\\
     \end{array}
\end{equation}
The claim in (\ref{eqn:claim1}) follows.

Next,  let 
\[
    \mathcal{H} \,:=\,   \{ f-g \,:\,  f,g \in \mathcal{F}_{\mathcal{A}_n} \}.
\]
Then for any $\delta>0$, 
\[
\log \mathcal N \left(\delta,   \mathcal{H} ,\left\| \cdot\right\|_n \right)  \,\leq\,  2\log \mathcal N \left(\delta/2,   \mathcal{F}_{\mathcal{A}_n} ,\left\| \cdot\right\|_n \right). 
\]
Hence,  Lemma \ref{lem9} implies that the event 
\[
\Omega_3\,:=\, \left\{    \sup_{ g \in \mathcal \mathcal{H}  }\left|  \frac{ \vert \|g_{ \mathcal{A}_n } /\mathcal{A}_n \|_n^2  - \|g_{ \mathcal{A}_n }/\mathcal{A}_n \|_\lt ^2 \vert  }{ \frac{1}{2} \|g_{ \mathcal{A}_n }/\mathcal{A}_n\|_\lt ^2 + \frac{\delta_n^2}{2}   } \right|   \leq 1    \right\}
\]
for constants $c_1,c_2,c_3>0$ holds   with probability at least
\[
1  \,-\,   c_2 \exp( -c_3c_1^2\log n  ).
\]
Hence, if $\Omega_3$ holds, then 
\begin{equation}
    \label{eqn:e28}
    \begin{array}{lll}
\displaystyle     \| \bar{f}_{ \mathcal{A}_n }-  \hat{f}_{ \mathcal{A}_n } \|_{ \lt}^2     &     =       &  \displaystyle  \mathcal{A}_n^2    \|( \bar{f}_{ \mathcal{A}_n }-  \hat{f}_{ \mathcal{A}_n })/\mathcal{A}_n  \|_{ \lt}^2     \\
      & \leq& \displaystyle\mathcal{A}_n^2 \left[ \frac{3}{2}\|( \bar{f}_{ \mathcal{A}_n }-  \hat{f}_{ \mathcal{A}_n })/\mathcal{A}_n  \|_n^2  \,+\,\frac{\delta_n^2}{2} \right]\\
    \end{array}
\end{equation}
and the claim (\ref{eqn:claim2}) follows combining  (\ref{eqn:e28}) with (\ref{eqn:e27}).

    \end{proof}

\section{Proof of Theorem \ref{thm2_v2} }

\begin{theorem}
    \label{thm2}
    Suppose that   $\bar{g} \in \mathcal{G}$ is such that 
 \[
  \| \bar{g} -g^*\|_{\infty} \,\leq\,\sqrt{\psi_n},
 \]
 so that $\phi_n$ is the approximating error.     Let $\hat{g}$ be the estimator defined in (\ref{eqn:g_hat}) and consider $\hat{g}_{\mathcal{B}_n}$ with $\mathcal{B}_n$ satisfying
 \[
 \mathcal{B}_n \geq   \max\{  8\mathcal{A}_n \| g^* \|_{\infty}, 8\mathcal{A}_n \sqrt{\psi_n}, \mathcal{U}_n^2  +  \frac{9 \mathcal{U}_n \mathcal{A}_n }{4}  + \frac{65  \mathcal{A}_n^2   }{32}   \}  .
 \]
 Moreover, let  $\mathcal{G}_{ \mathcal{B}_n } \,:=\, \{g_{ \mathcal{B}_n }/\mathcal{B}_n  \,:\, g \in \mathcal{G} \}$ and assume that
 \[
 \log \mathcal N \left( \delta,  \mathcal{G}_{ \mathcal{B}_n},  \left\| \cdot\right\|_n \right)\,\leq\, \eta_n^{\prime}(\delta)
 \]
 for some decreasing function $\eta_n^{\prime} \,:\, (0,1) \rightarrow \mathbb{R}_{\geq 0}$.  If 
\begin{equation}
    \label{eqn:entropy_condition}
     \underset{n \rightarrow \infty}{\lim} \,\left[  \sum_{l=0}^{\infty} \sum_{l^{\prime}=1 }^{\infty} \exp\left( -  C_1  \eta_n^{\prime}( 2^{-l-l^{\prime}-1  }  )\right) \,+\, \sum_{l=0}^{\infty}    \exp\left(- C_2 \eta_n^{\prime}(2^{-l-1}) \right) \,+\,\mathbb{P}(\|\epsilon\|_{\infty} > \mathcal{U}_n)    \right]\,=\,0,
\end{equation}
for some constants $C_1,C_2>0$, and 
\begin{equation}
    \label{eqn:cond_sums2}
\underset{l \in \mathbb{N}}{\sup}\,     \sum_{l^{\prime}=1 }^{\infty} \frac{ \eta_n^{\prime}(2^{-l-l^{\prime}}   )  }{   2^{2 l^{\prime}   }  \eta_n^{\prime}( 2^{-l} )  } \,\leq\, 1,
\end{equation}
then if the assumptions of Theorem \ref{thm1}  hold, for $\delta_n^{\prime}$ a critical radius of $ \mathcal{G}_{ \mathcal{B}_n}$,  it holds that 
\begin{equation}
    \label{eqn:claim3}
    \| g^*-  \hat{g}_{ \mathcal{B}_n } \|_n^2    \,=\, o_{ \mathbb{P} }\left(  \mathcal{B}_n r_n  \,+\,\psi_n\,+\, \frac{ \left[\mathcal{U}_n^4\,+\, \|g^*\|_{\infty}^2 \,+\, \mathcal{A}_n^2\mathcal{U}_n^2\right]  \eta_n^{\prime}( 1/\delta_n^{\prime} )   \,+\,\mathcal{B}_n^2 }{n}  \,+\, \mathcal{B}_n\cdot (\delta_n^{\prime})^2 \right). \\
\end{equation}
provided that $\hat{f}_{\mathcal{A}_n}$ satisfies 
\[
 \| f^*-\hat{f}_{\mathcal{A}_n}\|_{ \lt}^2 \,=\, o_{\mathbb{P}}(r_n) 
\]
for  $r_n$ the rate of convergence in Theorem \ref{thm1}. In addition, 
\begin{equation}
    \label{eqn:claim4}
   \begin{array}{lll}
\displaystyle      \| g^*-  \hat{g}_{ \mathcal{B}_n } \|_{ \lt}^2   &\,=\,& \displaystyle   o_{ \mathbb{P} }\bigg(  \mathcal{B}_n r_n  \,+\,\psi_n\,+\, \frac{ \left[\mathcal{U}_n^4\,+\, \|g^*\|_{\infty}^2 \,+\, \mathcal{A}_n^2\mathcal{U}_n^2\right]  \eta_n^{\prime}( 1/\delta_n^{\prime} ) \,+\, \mathcal{B}_n^2    }{n} \,+\, \mathcal{B}_n \cdot (\delta_n^{\prime})^2\bigg).
   \end{array}
\end{equation}
\end{theorem}

\begin{proof}
For simplicity, we denote the sampled data $\{(x_i^{\prime},y_i^{\prime})\}_{i=1}^n$ as $\{(x_i,y_i)\}_{i=1}^n$. 
Next, notice that by Lemma \ref{lem1}, we have that $\bar{g} = \bar{g}_{ \mathcal{B}_n}$. Moreover, assume that 
  \[
     \Omega_1 \,:=\, \{   \| \epsilon \|_{\infty} \,\leq\,  \mathcal{U}_n   \}
   \]
   holds. Then, by Lemma \ref{lem2}
   \[
    \Omega_2\,:=\,    \left\{  \| \hat{\epsilon}\|_{\infty} \,\leq \,   \mathcal{U}_n^2  +  \frac{9 \mathcal{U}_n \mathcal{A}_n }{4}  + \frac{65  \mathcal{A}_n^2   }{32}   \right\}
   \]
   holds. Hence, 
   \[
        \begin{array}{lll}
            \displaystyle \sum_{i=1}^n (\hat{\epsilon}_i^2 - \hat{g}_{ \mathcal{B}_n }(x_i))^2 & \leq& \displaystyle \sum_{i=1}^n (\hat{\epsilon}_i^2 - \hat{g}(x_i))^2\\
             & \leq& \displaystyle \sum_{i=1}^n (\hat{\epsilon}_i^2 - \bar{g}(x_i))^2\\
              & =&\displaystyle \sum_{i=1}^n (\hat{\epsilon}_i^2 - \bar{g}_{ \mathcal{B}_n }(x_i))^2
        \end{array}
   \]
   and so
   \begin{equation}
       \label{eqn:e35}
      \frac{1}{n} \sum_{i=1}^n ( \bar{g}_{\mathcal{B}_n }(x_i) - \hat{g}_{\mathcal{B}_n }(x_i)  )^2\,\leq\,   \frac{2}{n}\sum_{i=1}^n (\hat{g}_{ \mathcal{B}_n }(x_i)\,-\, \bar{g}_{ \mathcal{B}_n }(x_i)   )( \hat{\epsilon}_i^2 - \bar{g}_{\mathcal{B}_n }(x_i) ).
   \end{equation}

   However, 
   \[
    \begin{array}{lll}
       \displaystyle  \hat{\epsilon}_i^2&=& \displaystyle (y_i -  f^*(x_i) +  f^*(x_i) - \hat{f}_{\mathcal{A}_n }(x_i) )^2 \\
         &  = & \displaystyle (y_i - f^*(x_i))^2 + ( f^*(x_i) - \hat{f}_{\mathcal{A}_n }(x_i)  )^2  + 2 (y_i -  f^*(x_i)) (f^*(x_i) - \hat{f}_{\mathcal{A}_n }(x_i) )\\
          & = & \displaystyle \epsilon_i^2 + ( f^*(x_i) - \hat{f}_{\mathcal{A}_n }(x_i)  )^2  + 2 \epsilon_i(f^*(x_i) - \hat{f}_{\mathcal{A}_n }(x_i) )
    \end{array}
   \]
   which combined with (\ref{eqn:e35}) leads to 
   \begin{equation}
       \label{eqn:e36}
   \begin{array}{lll}
          \displaystyle  \frac{1}{n}\sum_{i=1}^n ( \bar{g}_{\mathcal{B}_n }(x_i) - \hat{g}_{\mathcal{B}_n }(x_i)  )^2 &\leq&  \displaystyle  \frac{2}{n}\sum_{i=1}^n (\hat{g}_{ \mathcal{B}_n }(x_i)\,-\, \bar{g}_{ \mathcal{B}_n }(x_i)   )( \epsilon_i^2 - g^*(x_i) )\,+\, \\
           & &\displaystyle\frac{2}{n}\sum_{i=1}^n  (\hat{g}_{\mathcal{B}_n }(x_i)- \bar{g}_{\mathcal{B}_n }(x_i))(\hat{f}_{\mathcal{A}_n }(x_i) - f^*(x_i) )^2 \,+\,\\
                      & &\displaystyle\frac{4}{n}\sum_{i=1}^n  (\hat{g}_{\mathcal{B}_n }(x_i)- \bar{g}_{\mathcal{B}_n }(x_i))(f^*(x_i) - \hat{f}_{\mathcal{A}_n }(x_i) )\epsilon_i  \,+\,\\
            & & \displaystyle\frac{4}{n}\sum_{i=1}^n  (\hat{g}_{\mathcal{B}_n }(x_i)- \bar{g}_{\mathcal{B}_n }(x_i))( g^*(x_i) -  \bar{g}_{\mathcal{B}_n }(x_i) )  \\
              &=: & T_1 \,+\, T_2\,+\, T_3\,+\,T_4,
   \end{array}
   \end{equation}
   and we proceed to bound $T_1, T_2, T_3$ and $T_4$.

   \textbf{Bounding $T_1$.} Let $\mathcal{V}_n \geq  \mathcal{U}_n^2 +  \|g^*\|_{\infty}$ and notice that 
   \[
      \mathbb{P}\left(  \underset{i=1,\ldots,n}{\max}\vert \epsilon_i^2 - g^*(x_i)\vert \,\geq \, \mathcal{V}_n \right)\,\leq\, \mathbb{P}(  \|\epsilon^2\|_{\infty} \geq \mathcal{U}_n^2 ) \,\leq\, \mathbb{P}(  \|\epsilon\|_{\infty} \geq \mathcal{U}_n). 
   \]

   Hence, by Lemma \ref{lem3}, 
\begin{equation}
    \label{eqn:e40}
    \begin{array}{lll}
\displaystyle   \frac{ 2\mathcal{B}_n   }{n} \sum_{i=1}^n   \frac{( \hat{g}_{  \mathcal{B}_n }(x_i)  -  \bar{g}_{  \mathcal{B}_n }(x_i)   )}{2\mathcal{B}_n }(\epsilon_i^2 -g^*(x_i) )  &\leq&  2 C \mathcal{B}_n \mathcal{V}_n  \| (\hat{g}_{\mathcal{B}_n} -  \bar{g}_{  \mathcal{B}_n }   )/(2\mathcal{B}_n) \|_n \cdot\\
&&\sqrt{ \eta_n^{\prime}( \| (\hat{g}_{\mathcal{B}_n}  -  \bar{g}_{  \mathcal{B}_n }   )/(4\mathcal{B}_n)\|_n  ) /n   }\\
 & \leq&  C \mathcal{V}_n\| \hat{g}_{\mathcal{B}_n} -  \bar{g}_{  \mathcal{B}_n }  \|_n \cdot\sqrt{ \eta_n^{\prime}( \| (\hat{g}_{\mathcal{B}_n}  -  \bar{g}_{  \mathcal{B}_n }   )/(4\mathcal{B}_n)\|_n  ) /n   }\\
\end{array}
\end{equation}
with probability at least
\[
 1\,-\, 4\sum_{l=0}^{\infty} \sum_{l^{\prime}=1 }^{\infty} \exp\left( -  C_1  \eta_n^{\prime}( 2^{-l-l^{\prime}-1  }  )\right) \,-\, 4\sum_{l=0}^{\infty}    \exp\left(- C_2 \eta_n^{\prime}(2^{-l-1}) \right) \,-\,\mathbb{P}(\|\epsilon\|_{\infty} > \mathcal{U}_n).
\]
Let $\Omega_3$ be the event such that (\ref{eqn:e40}) holds. Suppose that $\Omega_2$ holds. Then, if 
\[
\| (\hat{g}_{\mathcal{B}_n}  -  \bar{g}_{  \mathcal{B}_n }   )/(4\mathcal{B}_n)\|_n \,\leq \,  \delta^{\prime}_n
\]
it follows that 
\[
\| (\hat{g}_{\mathcal{B}_n} -  \bar{g}_{  \mathcal{B}_n }   )\|_n^2 \,\leq \, 16\mathcal{B}_n^2 (\delta_n^{\prime})^2.
\]
Suppose now that 
\[
\| (\hat{g}_{\mathcal{B}_n}  -  \bar{g}_{  \mathcal{B}_n }   )/(4\mathcal{B}_n)\|_n \,>\, \delta_n^{\prime}
\]
then (\ref{eqn:e40}) implies
\begin{equation}
    \label{eqn:e41}
    \begin{array}{lll}
\displaystyle   \frac{ 2\mathcal{B}_n   }{n} \sum_{i=1}^n   \frac{( \hat{g}_{  \mathcal{B}_n }(x_i)  -  \bar{g}_{  \mathcal{B}_n }(x_i)   )}{2\mathcal{B}_n }(\epsilon_i^2 -g^*(x_i) )  
 & \leq& \displaystyle   C \mathcal{V}_n\| \hat{g}_{\mathcal{B}_n} -  \bar{g}_{  \mathcal{B}_n }  \|_n \cdot\sqrt{ \eta_n^{\prime}( \delta_n^{\prime}  ) /n   }\\
  & \leq& \displaystyle  \frac{1}{32}\| \hat{g}_{\mathcal{B}_n}  -  \bar{g}_{  \mathcal{B}_n }   \|_n^2\,+\, \frac{ 8 C^2 \mathcal{V}_n^2 \eta_n^{\prime}( \delta_n^{\prime}) }{n}.\\
\end{array}
\end{equation}

Therefore, combining the two cases, 
\begin{equation}
    \label{eqn:e43}
  T_1 \,\leq\, 16\mathcal{B}_n^2  (\delta_n^{\prime})^2\,+\, \frac{1}{32}\| \hat{g}_{\mathcal{B}_n}  -  \bar{g}_{  \mathcal{B}_n }   \|_n^2\,+\, \frac{ 8 C^2 \mathcal{V}_n^2 \eta_n^{\prime}( \delta_n^{\prime}) }{n}.
\end{equation}

  \textbf{Bounding $T_2$.} We observe that
  \begin{equation}
      \label{eqn:e45}
     \displaystyle     T_2  \leq       \displaystyle   \frac{4 \mathcal{B}_n }{n}\sum_{i=1}^n  (\hat{f}_{\mathcal{A}_n }(x_i) - f^*(x_i) )^2.\\
  \end{equation}
\textbf{Bounding $T_3$.}  Notice that 
\[
T_3 \,=\,\frac{1}{n}\sum_{i=1}^n  (\hat{g}_{\mathcal{B}_n }(x_i)- \bar{g}_{\mathcal{B}_n }(x_i)) \nu_i
\]
where $\nu_i \,:=\, 4(f^*(x_i) - \hat{f}_{\mathcal{A}_n }(x_i) )\epsilon_i  $. However, since $\hat{f}_{\mathcal{A}_n}$ is independent of $\{(x_i,\epsilon_i) \}_{i=1}^n$ we observe that 
 \[
   \begin{array}{lll}
      \mathbb{E}\left(  \nu_i \big| x_i\right)&=&  \mathbb{E}\left( \mathbb{E}\left(  \nu_i  \big| x_i, \hat{f}_{\mathcal{A}_n}\right) \,|\, x_i \right) \\
      &=& 4 \mathbb{E}\left(  (f^*(x_i) - \hat{f}_{\mathcal{A}_n }(x_i) ) \mathbb{E}\left(  \epsilon_i  \big| x_i, \hat{f}_{\mathcal{A}_n}\right) \,|\, x_i \right) \\
       & = &4 \mathbb{E}\left(  (f^*(x_i) - \hat{f}_{\mathcal{A}_n }(x_i) ) \mathbb{E}\left(  \epsilon_i  \big| x_i\right) \,|\, x_i \right) \\
              & = &4 \mathbb{E}\left(  (f^*(x_i) - \hat{f}_{\mathcal{A}_n }(x_i) ) \cdot 0 \,|\, x_i \right) \\
               & = &0.
   \end{array}
 \] 
 Moreover, 
 \[
 \mathbb{P}\left( \underset{i =1,\ldots,n}{\max}\, \vert\nu_i\vert > 8 \mathcal{A}_n \mathcal{U}_n \,\bigg|\, \{x_i\}_{i=1}^n \right)\,\leq\,  \mathbb{P}\left( \|\epsilon\|_{\infty}> \mathcal{U}_n \,\bigg|\, \{x_i\}_{i=1}^n \right).
 \]
 Hence, by Lemma \ref{lem3}, and proceeding as when we bounded $T_1$, we obtain that 
 \begin{equation}
    \label{eqn:e47}
  T_3 \,\leq\, 16\mathcal{B}_n^2 ( \delta_n^{\prime})^2\,+\, \frac{1}{32}\| \hat{g}_{\mathcal{B}_n}  -  \bar{g}_{  \mathcal{B}_n }   \|_n^2\,+\, \frac{ 8^3 C^2  \mathcal{A}_n^2\mathcal{U}_n^2 \eta_n^{\prime}( \delta_n^{\prime} ) }{n}
\end{equation}
with probability at least
\[
 1\,-\, 4\sum_{l=0}^{\infty} \sum_{l^{\prime}=1 }^{\infty} \exp\left( -  C_1  \eta_n( 2^{-l-l^{\prime} -1 }  )\right) \,-\, 4\sum_{l=0}^{\infty}    \exp\left(- C_2 \eta_n(2^{-l-1}) \right) \,-\,\mathbb{P}(\|\epsilon\|_{\infty} > \mathcal{U}_n).
\]

\textbf{Bounding $T_4$.} We observe that 
\begin{equation}
    \label{eqn:e44}
    \begin{array}{lll}
  \displaystyle   \frac{4}{n}\sum_{i=1}^n  (\hat{g}_{\mathcal{B}_n }(x_i)- \bar{g}_{\mathcal{B}_n }(x_i))( g^*(x_i) -  \bar{g}_{\mathcal{B}_n }(x_i) )  &\leq&  \displaystyle 4\| \hat{g}_{\mathcal{B}_n }- \bar{g}_{\mathcal{B}_n }\|_n \cdot \|g^*(x_i) -  \bar{g}_{\mathcal{B}_n }(x_i)\|_n\\
      \vspace{0.1in}
 & \leq &   \displaystyle  \frac{1}{32}\| \hat{g}_{\mathcal{B}_n }- \bar{g}_{\mathcal{B}_n }\|_n^2\,+\,  128\|g^*(x_i) -  \bar{g}_{\mathcal{B}_n }(x_i)\|_n^2 \\
     \vspace{0.1in}
  & \leq& \displaystyle  \frac{1}{32}\| \hat{g}_{\mathcal{B}_n }- \bar{g}_{\mathcal{B}_n }\|_n^2\,+\,   128\psi_n^2\\
\end{array}
\end{equation}
where the first and second inequality follow by Cauchy–Schwarz inequality.

    Therefore, combining (\ref{eqn:e36}), (\ref{eqn:e43}), (\ref{eqn:e45}), (\ref{eqn:e47}), and (\ref{eqn:e44})  we obtain that, with probability approaching one, 
    \[
    \begin{array}{lll}
       \displaystyle   \| \bar{g}_{\mathcal{B}_n } - \hat{g}_{\mathcal{B}_n }\|_n^2 & \leq &  \displaystyle 16\mathcal{B}_n^2 (\delta_n^{\prime} )^2\,+\, \frac{1}{32}\| \hat{g}_{\mathcal{A}_n}  -  \bar{g}_{  \mathcal{B}_n }   \|_n^2\,+\, \frac{ 8 C^2 \mathcal{V}_n^2 \eta_n^{\prime}( \delta_n^{\prime}) }{n}\,+\,\\
       \vspace{0.1in}
         & & \displaystyle  \frac{4 \mathcal{B}_n }{n}\sum_{i=1}^n  (\hat{f}_{\mathcal{A}_n }(x_i) - f^*(x_i) )^2  \,+\,16\mathcal{B}_n^2 (\delta_n^{\prime} )^2\,+\, \frac{1}{32}\| \hat{g}_{\mathcal{B}_n}  -  \bar{g}_{  \mathcal{B}_n }   \|_n^2\,+\,\\
             \vspace{0.1in}
         && \displaystyle\frac{ 8^3 C^2  \mathcal{A}_n^2\mathcal{U}_n^2 \eta_n^{\prime}( \delta_n^{\prime}) }{n}\,+\, \frac{1}{32}\| \hat{g}_{\mathcal{B}_n }- \bar{g}_{\mathcal{B}_n }\|_n^2\,+\,  128\psi_n^2\\ \\
    \end{array}
    \]
    for a positive constant $\tilde{C}$. This implies along with Theorem \ref{thm1_v2}, with high probability, that for a $c_0$ it holds 
    \begin{equation}
        \label{eqn:e50}
        \begin{array}{lll}
          \displaystyle  \frac{29}{32} \| \bar{g}_{\mathcal{B}_n } - \hat{g}_{\mathcal{B}_n }\|_n^2 & \leq&  \displaystyle 32\mathcal{B}_n^2 (\delta_n^{\prime} )^2\,+\, \frac{\left[8 C^2 \mathcal{V}_n^2 \,+\,8^3 C^2  \mathcal{A}_n^2\mathcal{U}_n^2\right]  \eta_n^{\prime}( \delta_n^{\prime}) }{n}\,+\,  c_0 \mathcal{B}_n r_n \,+\, 128 \psi_n 
        \end{array}
    \end{equation}
    and so (\ref{eqn:claim3})  follows.

    Next,  set
\[
    \mathcal{H} \,:=\,   \{ f-g \,:\,  f,g \in \mathcal{G}_{ \mathcal{A}_n } \}.
\]
Then for any $\delta>0$, 
\[
\log \mathcal N \left(\delta,   \mathcal{H} ,\left\| \cdot\right\|_n \right)  \,\leq\,  2\log \mathcal N \left(\delta/2,   \mathcal{G}_{\mathcal{A}_n} ,\left\| \cdot\right\|_n \right). 
\]
Hence,  Lemma \ref{lem9} implies that the event 
\[
\mathcal{E}\,:=\, \left\{    \sup_{ g \in \mathcal{H}  }\left|  \frac{ \vert \|g_{ \mathcal{B}_n } /\mathcal{B}_n \|_n^2  - \|g_{ \mathcal{B}_n }/\mathcal{B}_n \|_\lt ^2 \vert  }{ \frac{1}{2} \|g_{ \mathcal{B}_n }/\mathcal{B}_n\|_\lt ^2 + \frac{(\delta_n^{\prime})^2}{2} } \right|   \leq 1    \right\}
\]
for constants $c_2,c_3>0$ holds   with probability at least
\[
1  \,-\,   c_2 \exp( -c_3c_1^2\log n  )
\]
Therefore, if $\mathcal{E}$ holds, then 
\begin{equation}
    \label{eqn:e51}
    \begin{array}{lll}
\displaystyle     \| \bar{g}_{ \mathcal{A}_n }-  \hat{g}_{ \mathcal{B}_n } \|_{ \lt}^2     &     =       &  \displaystyle  \mathcal{B}_n^2    \|( \bar{g}_{ \mathcal{B}_n }-  \hat{g}_{ \mathcal{B}_n })/\mathcal{B}_n  \|_{ \lt}^2     \\
      & \leq& \displaystyle\mathcal{B}_n^2 \left[ \frac{3}{2}\|( \bar{g}_{ \mathcal{B}_n }-  \hat{g}_{ \mathcal{B}_n })/\mathcal{B}_n  \|_n^2  \,+\, \frac{(\delta_n^{\prime})^2}{2} \right]\\
    \end{array}
\end{equation}
and the claim (\ref{eqn:claim4}) follows combining  (\ref{eqn:e50}) with (\ref{eqn:e51}).
    
\end{proof}

\section{Proof of  Theorem \ref{thm3_v2} }



\begin{theorem}
    \label{thm3}
    Suppose that $\mathcal{B}_n $ satisfies
    \[
     \mathcal{B}_n \geq   \max\{  8\mathcal{A}_n \sigma^2,  \mathcal{U}_n^2  +  \frac{9 \mathcal{U}_n \mathcal{A}_n }{4}  + \frac{65  \mathcal{A}_n^2   }{32}   \}  .
    \]
    and $\mathcal{B}_n o(n^b)$ for some $b>0$. Moreover, assume that the conditions of Theorem \ref{thm1} hold. Then 
    \begin{equation}
    \label{eqn:claim3}
    (\sigma^2 - \hat{\sigma}^2)^2    \,=\, o_{ \mathbb{P} }\left(   \mathcal{B}_n r_n \,+\,\frac{ \left[\mathcal{U}_n^4\,+\, \mathcal{A}_n^2\mathcal{U}_n^2\right]  \log(n )   \,+\,\mathcal{B}_n^2 }{n}  \,+\, \frac{\mathcal{B}_n \log(n)}{n}\right) \\
\end{equation}
provided that $\hat{f}_{\mathcal{A}_n}$ satisfies 
\[
 \| f^*-\hat{f}_{\mathcal{A}_n}\|_{ \lt}^2 \,=\, o_{\mathbb{P}}(r_n). 
\]
\end{theorem}

\begin{proof}
    Notice that in this case $\mathcal{F} =[0,\mathcal{B}_n]$ and $\sigma^2 \in \mathcal{F}$. It is well known that 
    \begin{equation}
        \label{eqn:e70}
        \log(N(\delta, \mathcal{F},\|\cdot\|_n)\,=\,\log(N(\delta, \mathcal{F},\vert\cdot\vert)\,\leq \,\eta_n^{\prime}(\delta)\,:=\, C \log\left(\frac{ \mathcal{B}_n }{\delta}\right)
    \end{equation}
    for a constant $C >0$. To verify (\ref{eqn:cond_sums2}), notice that
    \[
\underset{l \in \mathbb{N}}{\sup}\,     \sum_{l^{\prime}=1 }^{\infty} \frac{ \eta_n^{\prime}(2^{-l-l^{\prime}}   )  }{   2^{2 l^{\prime}   }  \eta_n^{\prime}( 2^{-l} )  }   \,=\, \underset{l \in \mathbb{N}}{\sup}\,     \sum_{l^{\prime}=1 }^{\infty} \frac{  (l+l^{\prime})   } { 2^{2 l^{\prime}   }  l     }    \,\leq\, 1,
    \]
    
    Hence, by Lemma \ref{lem9} and our choice of $\mathcal{B}_n$,  the critical radius can be chosen to satisfy 
    \[
    \delta_n^{\prime} \,=\, c_1 \sqrt{\frac{\log n}{n}}
    \]
    for a constant $c_1>0$. 
    
    Next,  notice that (\ref{eqn:entropy_0}) implies that 
    \[
     \underset{ n \rightarrow \infty}{\lim}  \mathbb{P}( \|\epsilon\|_{\infty} >\mathcal{U}_n ) \,=\,0.
    \]
    This combined with (\ref{eqn:e70}) implies that (\ref{eqn:entropy_condition}) holds.  The conclusion follows from Theorem \ref{thm2}.

\end{proof}

\section{Proof of Theorem \ref{thm5_v2} }

\begin{theorem}
    \label{thm5}
    Suppose that $f^* \in  \mathcal{H}(l_1, \mathcal{P}_1) $ for some $ l_1 \in \mathbb{N}$ and $\mathcal{P}_1\subset [1,\infty ) \times \mathbb{N}$. In addition, assume that each function $g$ in the definition of $f^*$ can have different smoothness $p_g =  q_g +s_q$, for $q_g \in \mathbb{N} $, $s_g \in (0,1]$,  and of potentially different input dimension $K_g $, so that $(p_g, K_g ) \in \mathcal{P}_1$. Let $K_{\max}$ be the largest input dimension and $p_{\max}$ the largest smoothness of any of the functions $g$. Suppose that all the partial derivatives of order less than or equal to $q_g$ are uniformly bounded by constant $c_2$, and each function $g$ is Lipschitz continuous
with Lipschitz constant $C_{\mathrm{Lip} } \geq1 $. Also, assume that $\max\{p_{\max},K_{\max} \} =O(1)$.      Let  
    \[
    \phi_{n} = \max_{(p, K) \in \mathcal P_1 } n^{\frac{-2p}{ (2p+K)}}. 
\]
Then there exists sufficiently large positive constants $c_3$ and $c_4$ such that if 
\[
    L  =\lceil c_3 \log n \rceil\,\,\,\,\,\,\text{and}\,\,\,\,\,\, \nu = \left\lceil c_4  \max_{(p, K) \in \mathcal P_1  } n^{\frac{K}{ 2(2p+K)}}\right\rceil
\]
or 
\[
    L  =\left\lceil c_3  \max_{(p, K) \in \mathcal P_1  } n^{\frac{K}{ 2(2p+K)}} \log n\right\rceil\,\,\,\,\,\,\text{and}\,\,\,\,\,\, \nu = \left\lceil c_4  \right\rceil,
\]
then $\hat{f}_{\mathcal{A}_n}$, with $\hat{f}$ as defined in (\ref{eqn:f_hat}), satisfies, 
\begin{equation}
   \label{eqn:den_nn1}
     \| f^*-\hat{f}_{k_1,\mathcal{A}_n}\|_{ \lt}^2 \,=\, o_{\mathbb{P}}\left(   r_n \right),
\end{equation}
where
\begin{equation}
    \label{eqn:r_n_nn}
    r_n\,:=  \,\frac{ \max\{\mathcal{A}_n,\mathcal{U}_n^2 \} \log n  }{n}    \,+\,   \phi_n \log^3(n) \log(\mathcal{A}_n) \max\{\mathcal{A}_n,\mathcal{U}_n^2 \}  , 
\end{equation}
provided that 
\[
      \underset{n \rightarrow \infty}{\lim}   \mathbb{P}(\|\epsilon\|_{\infty} > \mathcal{U}_n ) \,=\, 0
\]
\end{theorem}

\begin{proof}

Let $\mathcal{F}_{\mathcal{A}_n}( L,\nu) \,:= \,\{ f_{\mathcal{A}_n }\,:\,   f \in \mathcal{F}( L,\nu)  \}  $. Next, we notice that by Theorem 6 from \cite{bartlett2019nearly}, it holds that the VC dimension of $\mathcal{F}(L,\nu)$ satisfies
\[
  \mathrm{VC}( \mathcal{F}(L,\nu))\,\lesssim\, L^2 \nu^2 \log(L \nu ). 
\]
Then by Lemma 9.2 and Theorem 9.4 from \cite{gyorfi2002distribution}, for any $\delta \in (0,1)$
\[
   \begin{array}{lll}
      \log N( \delta, \mathcal{F}_{\mathcal{A}_n}(L,\nu ),\|\cdot\|_n  )  & \lesssim &L^2 \nu^2 \log(L \nu ) \cdot \left[ \log(\mathcal{A}_n^2 \delta^{-2} ) \,+ \, \log \log(\mathcal{A}_n^2 \delta^{-2} ) \right]\\ 
       &\lesssim&  L^2 \nu^2 \log(L \nu ) \log(\mathcal{A}_n^2 \delta^{-1} ) \\
       &\leq & C_0 n \phi_n \log^3(n)  \log(\mathcal{A}_n^2 \delta^{-1} )  \\
        & =:& \eta_n(\delta),
   \end{array}
\]
for some positive constant $C_0>0$, where the last inequality follows from our choice of $L$ and $\nu$.

Now verify (\ref{eqn:entropy_0}). We observe that 
    \begin{equation}
        \label{eqn:ver5}
        \arraycolsep=1.4pt\def\arraystretch{1.6}
           \begin{array}{l}
  \displaystyle\sum_{l=0}^{\infty} \sum_{l^{\prime}=1 }^{\infty} \exp\left( -  C_1  \eta_n( 2^{-l-l^{\prime} -1}  )\right) \,+\, \sum_{l=0}^{\infty}    \exp\left(- C_2 \eta_n(2^{-l-1}) \right) \\
  =  \displaystyle\sum_{l=0}^{\infty} \sum_{l^{\prime}=1 }^{\infty} \exp\left( -  C_1  C_0  (l+l^{\prime}+1)n \phi_n \log^3(n) \log(2\mathcal{A}_n^2 )  \right) \\
\displaystyle \,\,\,\,\,+\, \sum_{l=0}^{\infty}    \exp\left(- C_2 C_0    n \phi_n \log^3(n)  \log(2\mathcal{A}_n^2  )  \right) \\
=  \displaystyle  \left[ \exp\left( -  C_1  C_0  n \phi_n \log^3(n) \log(2\mathcal{A}_n^2 )    \right)/ \left[ 1 - \exp\left( -  C_1  C_0  n \phi_n \log^3(n)  \log(2\mathcal{A}_n^2 )   \right) \right]  \right]^2 \,+\,\\
\,\,\,\,\,\,\,\displaystyle \exp\left(- C_2  C_0    n \phi_n \log^3(n)  \log(2\mathcal{A}_n^2  )\right)/\left[ 1-  \exp\left(- C_2 C_0   n \phi_n \log^3(n)  \log(2\mathcal{A}_n^2  )\right) \right]\\
  \underset{n \rightarrow \infty }{\rightarrow}  0,
   \end{array}
    \end{equation}
   since
    \[
     \underset{n \rightarrow \infty}{\lim }\,  n \phi_n \log^3(n)  \log(2\mathcal{A}_n^2  )  \,=\, \infty.
    \]
Therefore, (\ref{eqn:ver5}) and $  \underset{n \rightarrow \infty}{\lim}   \mathbb{P}(\|\epsilon\|_{\infty} > \mathcal{U}_n ) \,=\, 0$ imply that  (\ref{eqn:entropy_0}) holds.

    Furthermore, 
    \[
\underset{l \in \mathbb{N}}{\sup}\,     \sum_{l^{\prime}=1 }^{\infty} \frac{ \eta_n(2^{-l-l^{\prime}}   )  }{   2^{2 l^{\prime}   }  \eta_n( 2^{-l} )  }   \,=\, \underset{l \in \mathbb{N}}{\sup}\,     \sum_{l^{\prime}=1 }^{\infty}    \frac{ l+ l^{\prime}}{   2^{2 l^{\prime}     }  l}  \,\leq\, 1,
    \]
    and so (\ref{eqn:cond_sum}) holds.

Moreover, by Lemma \ref{lem9}, the critical radius $\delta_n$ can be taken as 
 \[
    \delta_n\,\geq \, c_1 \left[ \sqrt{ \frac{ \log n}{n} }  \,+\,  \sqrt{   \frac{ \log \mathcal N \left(1/(24n),   \mathcal{F}_{\mathcal{A}_n} ,\left\| \cdot\right\|_n \right) }{n}} \right] 
    \]
so we can take 
\[
\delta_n \,\asymp\,  \sqrt{ \frac{ \log n}{n} } \,+\, \sqrt{ \phi_n \log^3(n)  \log(\mathcal{A}_n^2 n ) }.
\]
Finally,  from  Theorem 3 in \cite{kohler2019rate}, it follows that  there exists $\bar{f} \in \mathcal{F}(L,\nu)$ such that 
\[
   \| \bar{f} -f^* \|_{\infty }\,\leq\, \sqrt{ \phi_n}.
\]
Therefore, the claim in (\ref{eqn:den_nn1}) follows form Theorem \ref{thm1}.
\end{proof}

\subsection{Proof of Theorem \ref{thm6}}

We start by recalling the notation. Let $a(\alpha)$ and $b(\alpha)$ be such that 
\[
 a(\alpha) \,\geq\, \bigg\vert a_1(\alpha) \,+\, a_2(\alpha) \,+\, 2A_n \sqrt{\frac{2\log(64/\alpha )+ 2\log  \widetilde{B}}{n}}  - \frac{1}{\vert I_4\vert} \sum_{i \in I_4 }\hat{g}_{A_n}^{(B+1)}(x_i)  \,+\,  a_0 \bigg\vert 
\]
 where
\[
a_0 = \vert  \mathbb{E}( g^*(X) )  \,-\, \mathbb{E}( \hat{g}_{A_n}^{(B+1)}(X)  \,|\,  \hat{g}_{A_n}^{(B+1)}) \vert, 
\]
\[
 a_2(\alpha) \,=\,  A_n^2 \sqrt{\frac{32[\log(8/ \alpha) + \,\log  \widetilde{B}] }{n}}.
\]
Moreover, let
\[
b(\alpha)  \,\geq\,  \frac{32 [ \mathbb{E}((\hat{f}_{A_n}^{(B+1)}(X) - Y )^2 )-  \mathbb{E}(g^*(X))  ]^{1/2} }{ 5\alpha (1-0.58 \alpha)    }.
\]


\begin{proof}

We start by setting  
$$a_3(\alpha) \,=\,  2A_n \sqrt{\frac{2\log(64/\alpha  )\,+  2\log   \widetilde{B} }{n}}.   $$

Notice that 
\[
 \mathbb{E}(Y)\,=\, \mathbb{E}(f^*(X))
\]
and
\[
 \mathbb{E}(g^*(X) ) \,=\, \mathbb{E}(\text{Var}(Y|X) ) \,=\,\text{Var}(Y) \,-\, \text{Var}(  \mathbb{E}(Y|X) )\,=\, \text{Var}(Y) \,-\, \mathbb{E}( (f^*(X)  - \mathbb{E}(Y)    )^2 ). 
\]
We also let $\epsilon =   (Y- f^*(X))/\sqrt{g^*(X)}$.

\textbf{Step 1.}\\
Let 
\[
 \Omega_1 \,=\, \left\{  \| \hat{f}_{A_n}^{(j)} - f^*\|_{\mathcal{L}_2 }^2  \leq  a(\alpha), \,\,\, j = B+1,\ldots,B + \widetilde{B} \right\}.
\]
for some $c>0$. Define 
\[
Z = \frac{1}{ \widetilde{B} }\sum_{j=B+1}^{B+ \widetilde{B}}\hat{f}^{(j)}_{A_n}(X)  -f^*(X)
\]
and notice that 
\[
\arraycolsep=1.4pt\def\arraystretch{1.6}
\begin{array}{lll}
    \mathbb{P}(\vert  Z\vert  \geq  2\sqrt{ a(\alpha)/( \widetilde{B} \alpha) } + b(\alpha) ) & \leq  &   \mathbb{P}(\vert  Z  \,-\, \mathbb{E}(Z  \,|\,  X,\Omega_1,\{(x_i,y_i)\}_{i \in I_1 \cup I_2\cup I_3  }) \vert  \geq2\sqrt{ a(\alpha)/( \widetilde{B} \alpha) } )   \,+\,\\
    &&\mathbb{P}(\vert \mathbb{E}(Z  \,|\,  X,\Omega_1,\{(x_i,y_i)\}_{i \in I_1 \cup I_2\cup I_3  }) \vert  \geq  b(\alpha) )  \\
     & \leq& \mathbb{P}(\vert  Z  \,-\,\mathbb{E}(Z  \,|\,  X,\Omega_1,\{(x_i,y_i)\}_{i \in I_1 \cup I_2\cup I_3  })\vert  \geq  2\sqrt{ a(\alpha)/( \widetilde{B} \alpha) } \,|\,\Omega_1 )  \,+\,  \\
      & &\mathbb{P}(\Omega^c_1) \,+\,\\
      & &\mathbb{P}(\vert \mathbb{E}(Z  \,|\,  X,\Omega_1,\{(x_i,y_i)\}_{i \in I_1 \cup I_2\cup I_3  }) \vert  \geq  b(\alpha)  )  \\
\end{array}
\]
and we proceed to bound each of the terms in the right hand side.

\textbf{Step 2.}\\
Let $Z_j \,=\, \hat{f}^{(j)}_{A_n}(X)  -f^*(X)$. Then, for $t>0$, by Markov's inequality,
\[
\arraycolsep=1.4pt\def\arraystretch{1.6}
 \begin{array}{l}
        \displaystyle   \mathbb{P}(\vert  Z  \,-\, \mathbb{E}(Z  \,|\,  X,\Omega_1,\{(x_i,y_i)\}_{i \in I_1 \cup I_2\cup I_3  }) \vert  \geq  t \,|\, \Omega_1 )\\
        = 
  \displaystyle      \mathbb{E}(  \mathbb{P}(\vert  Z  \,-\, \mathbb{E}(Z  \,|\,  X,\Omega_1,\{(x_i,y_i)\}_{i \in I_1 \cup I_2\cup I_3  }) \vert  \geq  t \,|\,  X,\Omega_1,\{(x_i,y_i)\}_{i \in I_1 \cup I_2\cup I_3  } )\,|\, \Omega_1 )\\
   \displaystyle\leq  \frac{1}{t^2} \mathbb{E}( \text{Var}(Z\,|\,   X,\Omega_1,\{(x_i,y_i)\}_{i \in I_1 \cup I_2\cup I_3  } )  \,|\, \Omega_1 )\\
      \displaystyle= \frac{1}{ \widetilde{B}^2 t^2} \sum_{j=B+1}^{B+\widetilde{B} } \mathbb{E}( \text{Var}(Z_j\,|\,   X,\Omega_1,\{(x_i,y_i)\}_{i \in I_1 \cup I_2\cup I_3  } )  \,|\, \Omega_1 )\\
            \displaystyle\leq  \frac{1}{ \widetilde{B}^2 t^2} \sum_{j=B+1}^{B+\widetilde{B} } \mathbb{E}( \mathbb{E}(Z_j^2\,|\,   X,\Omega_1,\{(x_i,y_i)\}_{i \in I_1 \cup I_2\cup I_3  } )  \,|\, \Omega_1 )\\
                                   \displaystyle=  \frac{1}{ \widetilde{B}^2 t^2} \sum_{j=B+1}^{B+\widetilde{B} } \mathbb{E}( Z_j^2 \,|\, \Omega_1 )\\
                    =  \displaystyle \frac{1}{t^2}\frac{1}{\widetilde{B}^2}\sum_{j=B+1}^{B+\widetilde{B}} \mathbb{E}\bigg( \mathbb{E}(  (\hat{f}^{(j)}_{A_n}(X)  \,-\,f^*(X) )^2\,\big|\, \hat{f}^{(j)}_{A_n},\Omega_1  )\,\big|\, \Omega_1 \bigg)\\
        \end{array}
\]
\[
\arraycolsep=1.4pt\def\arraystretch{1.6}
\begin{array}{l}
                    = \displaystyle \frac{1}{t^2}\frac{1}{\widetilde{B}^2}\sum_{j=B+1}^{B+\widetilde{B}} \mathbb{E}\bigg( \| \hat{f}_{A_n}^{(j)} - f^*\|_{\mathcal{L}_2 }^2  \,\big|\, \Omega_1 \bigg)\\
                      \leq\displaystyle\frac{  a(\alpha) }{  \widetilde{B} t^2}\\
                        \leq  \displaystyle\frac{  \alpha}{4}\\   

\end{array}
\]

where the second equality  follows since  $\{f_{A_n}^{(j)}(X) \}_{j=B+1}^{B+\widetilde{B}}$ are independent conditional on $X, \Omega_1$ and $\{(x_i,y_i)\}_{i \in I_1 \cup I_2\cup I_3  } $, and last inequality by setting $t = 2\sqrt{ a(\alpha)/( \widetilde{B} \alpha) }$.
Furthermore,
\begin{equation*}
\arraycolsep=1.4pt\def\arraystretch{1.6}
    \begin{array}{lll}
   \displaystyle       \mathbb{P}(\Omega^c_1)  & \leq   & \displaystyle   \widetilde{B} \cdot  \underset{j=1,\ldots, \widetilde{B} }{\max} \mathbb{P}(\| \hat{f}_{A_n}^{(B+j)} - f^*\|_{\mathcal{L}_2 }^2  >  a(\alpha))\\
          & =  &  \displaystyle   \widetilde{B} \cdot  \underset{j=1,\ldots, \widetilde{B} }{\max} \mathbb{P}(  \mathbb{E}((\hat{f}_{A_n}^{(B+j)}(X) - f^*(X) )^2\,|\,\hat{f}_{A_n}^{(B+j)}  )  >  a(\alpha)).\\
    \end{array}
\end{equation*}
However,
\begin{equation}
\label{eqn:300}
\arraycolsep=1.4pt\def\arraystretch{1.6}
    \begin{array}{lll}
   \displaystyle   \mathbb{E}((\hat{f}_{A_n}^{(B+j)}(X) - Y )^2\,|\,\hat{f}_{A_n}^{(B+1)}  ) &=    &  \displaystyle \mathbb{E}((\hat{f}_{A_n}^{(B+j)}(X) - f^*(X) )^2\,|\,\hat{f}_A^{(B+j)}  )  \,+\, \\
         & & \displaystyle 2\mathbb{E}( \epsilon \sqrt{g^*(X)}(f^*(X) - \hat{f}_{A_n}^{(B+j)}(X) ) \,|\,\hat{f}_{A_n}^{(B+j)} )\,+\,\mathbb{E}(\epsilon^2 g^*(X))\\
          & =&\displaystyle \mathbb{E}((\hat{f}_{A_n}^{(B+j)}(X) - f^*(X) )^2\,|\,\hat{f}_{A_n}^{(B+j)}  )  \,+\, \mathbb{E}( g^*(X) \mathbb{E}(\epsilon^2|X)  )\\          & =&\displaystyle \mathbb{E}((\hat{f}_{A_n}^{(B+j)}(X) - f^*(X) )^2\,|\,\hat{f}_{A_n}^{(B+j)}  )  )  \,+\, \mathbb{E}( g^*(X) ).\\       
    \end{array}
\end{equation}
Therefore, 
\begin{equation}
\label{eqn:e123}
    \begin{array}{lll}
   \displaystyle       \mathbb{P}(\Omega^c_1)  & \leq   & \displaystyle   \widetilde{B}\cdot  \underset{j=1,\ldots, \widetilde{B} }{\max} \mathbb{P}(  \mathbb{E}((\hat{f}_{A_n}^{(B+j)}(X) - Y )^2\,|\,\hat{f}_{A_n}^{(B+j)}  )  >  a(\alpha) + \mathbb{E}(g^*(X))   ).\\
    \end{array}
\end{equation}

\textbf{Step 3.}\\
Now, 
\begin{equation}
\label{eqn:121}
\arraycolsep=1.4pt\def\arraystretch{1.6}
    \begin{array}{lll}
         a(\alpha) + \mathbb{E}(g^*(X))& \geq & \displaystyle a_1(\alpha) \,+\, a_2(\alpha)\,+\, a_3(\alpha)- \frac{1}{\vert I_4\vert} \sum_{i \in I_4 }\hat{g}_{A_n}^{(B+1)}(x_i)  \,+\, \\
          & &\displaystyle\vert  \mathbb{E}( g^*(X) )  \,-\, \mathbb{E}( \hat{g}_{A_n}^{(B+1)}(X)  \,|\,  \hat{g}_{A_n}^{(B+1)} ) \vert\,+\,\\
       & & \displaystyle     \mathbb{E}( \hat{g}_{A_n}^{(B+1)}(X)  \,|\,  \hat{g}_{A_n}^{(B+1)} )  \,+\,
          \mathbb{E}( g^*(X) )  \,-\, \mathbb{E}( \hat{g}_{A_n}^{(B+1)}(X)   \,|\,  \hat{g}_{A_n}^{(B+1)})  \\
         &\geq &\displaystyle a_1(\alpha) \,+\, a_2(\alpha)\,+\, a_3(\alpha)- \big\vert \frac{1}{\vert I_4\vert} \sum_{i \in I_4 }\hat{g}_{A_n}^{(B+1)}(x_i)  \,-\,  \mathbb{E}( \hat{g}_{A_n}^{(B+1)}(X) \,|\,  \hat{g}_{A_n}^{(B+1)} ) \big\vert.  
    \end{array}
\end{equation}
Hence,
\begin{equation}
    \label{eqn:e124}
    \arraycolsep=1.4pt\def\arraystretch{1.6}
    \begin{array}{lll}
                \mathbb{P}(\Omega^c_1)   & \leq&\displaystyle   \widetilde{B} \cdot\mathbb{P}(  \mathbb{E}((\hat{f}_{A_n}^{(B+1)}(X) - Y )^2\,|\,\hat{f}_{A_n}^{(B+1)}  )  >  a_1(\alpha) +  a_2(\alpha)  )\,+\, \\
         & &\displaystyle  \widetilde{B} \cdot\mathbb{P}\big( \big\vert \frac{1}{\vert I_4\vert} \sum_{i \in I_4 }\hat{g}_{A_n}^{(B+1)}(x_i)  \,-\,  \mathbb{E}( \hat{g}_A^{(B+1)}(X)  \,|\,  \hat{g}_{A_n}^{(B+1)} ) \big\vert >  a_3(\alpha)  \big).\\
    \end{array}
\end{equation}

\textbf{Step 4.}\\
To bound the second term in the right-hand side of (\ref{eqn:e124}), we observe by 
Hoeffding's inequality that, 
\begin{equation}
    \label{eqn:107}
    \begin{array}{lll}
   \displaystyle   \widetilde{B} \cdot \mathbb{P}\bigg(  \vert \mathbb{E}(\hat{g}_{A_n}^{(B+1)}(X)\,|\, \hat{g}_{A_n}^{(B+1)} ) - \frac{1}{\vert I_4\vert} \sum_{i \in I_4 }\hat{g}_{A_n}^{(B+1)}(x_i)  \vert > a_3(\alpha) \,\bigg|\, \hat{g}_{A_n}^{(B+1)} \bigg)  
 &\leq&   \displaystyle   2  \widetilde{B} \cdot\exp\left(  - \frac{ n \,a_3(\alpha)^2 }{8 A_n^2 }\right)\\
  & \leq &   \displaystyle \frac{\alpha }{32}. \\
   \end{array}
\end{equation}

\textbf{Step 5.}\\
Next,  we bound the first term in the right-hand side of (\ref{eqn:e124}). Notice that 
\begin{equation}
    \label{eqn:104}
    \begin{array}{l}
 \widetilde{B}\cdot          \mathbb{P}(  \mathbb{E}((\hat{f}_{A_n}^{(B+1)}(X) - Y )^2\,|\,\hat{f}_{A_n}^{(B+1)}  )  > a_1(\alpha) +  a_2(\alpha) ) \\
         \leq  \displaystyle   \widetilde{B}\cdot \mathbb{P}\bigg(  \vert \mathbb{E}((\hat{f}_{A_n}^{(B+1)}(X) - Y )^2\,|\,\hat{f}_{A_n}^{(B+1)}) 
  -   \frac{1}{\vert I_4\vert } \sum_{i \in I_4} (\hat{f}_{A_n}^{(B+1)}(x_i) - y_i )^2   \vert >  a_2(\alpha)   \bigg)\\
          \,+\,  \displaystyle    \widetilde{B}\cdot \mathbb{P}\bigg(   \frac{1}{\vert I_4\vert } \sum_{i \in I_4} (\hat{f}_{A_n}^{(B+1)}(x_i) - y_i )^2    > a_1(\alpha)   \bigg)\\
        =: T_1\,+\, T_2,
    \end{array}
\end{equation}
and we proceed to bound each term.\\

\textbf{Step 6.}
To bound $T_1$, we see that 
\[
0\leq \vert \hat{f}_{A_n}^{(B+1)}(x_i) - y_i  \vert^2 \,\leq \,  4A_n^2.
\]
Then, by Hoeffding's inequality
\begin{equation}
    \label{eqn:107}
    \begin{array}{l}
  \displaystyle       \widetilde{B}\cdot \mathbb{P}\bigg(  \vert \mathbb{E}((\hat{f}_{A_n}^{(B+1)}(X) - Y )^2\,|\,\hat{f}_{A_n}^{(B+1)}) 
  -   \frac{1}{\vert I_4\vert } \sum_{i \in I_4} (\hat{f}_{A_n}^{(B+1)}(x_i) - y_i )^2   \vert >  a_2(\alpha)   \bigg)\\
   \displaystyle \leq 2    \widetilde{B}\cdot \exp\left(  -\frac{  2\vert I_4 \vert a_2(\alpha)^2  }{   (4A_n^2)^2 }  \right)\\
   \displaystyle \leq 2    \widetilde{B}\cdot \exp\left(  -\frac{  \vert I_4 \vert a_2(\alpha)^2  }{   8A_n^4 }  \right)\\
      \displaystyle \leq 2    \widetilde{B}\cdot \exp\left(  -\frac{  n a_2(\alpha)^2  }{   32 A_n^4 }  \right)\\
     \displaystyle \leq \frac{\alpha}{4}.
    \end{array}
\end{equation}

\textbf{Step 7.}\\
We now bound $T_2$. Towards that end let  $\mathcal{D}_4 \,=\, \{ (x_i,y_i) \}_{i \in  I_4 }$, $\mathcal{D}_{1:3} \,=\, \{ (x_i,y_i) \}_{i \in  I_j, j=1,\ldots,3 }$,  $\Lambda \,=\,  \{\hat{f}_{A_n}^{(j)}\}_{j=1}^B$, and let $G_{\mathcal{D}_4,\mathcal{D}_{1:3} },\hat{G}_{\mathcal{D}_4, \Lambda } \,:\,\mathbb{R} \rightarrow [0,1]$ be given as
\[
G_{\mathcal{D}_4,\mathcal{D}_{1:3} }(a) \,=\, \displaystyle\mathbb{P}\bigg(   \frac{1}{\vert I_4\vert } \sum_{i \in I_4} (\hat{f}_{A_n}^{(B+1)}(x_i) - y_i )^2    > a \,|\, \mathcal{D}_4,\mathcal{D}_{1:3}  \bigg)
\]
and 
\[
\hat{G}_{ \mathcal{D}_4,\Lambda }(a) \,=\,  \frac{1}{B } \sum_{j=1}^B  1_{ \{ \frac{1}{\vert I_4\vert } \sum_{i \in I_4} (\hat{f}_{A_n}^{(j)}(x_i) - y_i )^2    > a \}  }.
\]
Also, let
\[
   \Omega_2(\mathcal{D}_{1:3},\mathcal{D}_4) \,=\,  \bigg\{   \{ \hat{f}_{A_n}^{(j)}  \}_{j=1}^B\,:\,   \underset{a\in \mathbb{R}}{\sup} \,\vert G_{\mathcal{D}_4,\mathcal{D}_{1:3} }(a) \,-\, \hat{G}_{ \mathcal{D}_4,\Lambda }(a) \vert \leq \frac{\alpha }{32  \widetilde{B}}  \bigg\}.
\]
Then
\begin{equation}
    \label{eqn:30}
    \arraycolsep=1.4pt\def\arraystretch{1.6}
    \begin{array}{l}
          \displaystyle\mathbb{P}\bigg(   \frac{1}{\vert I_4\vert } \sum_{i \in I_4} (\hat{f}_{A_n}^{(B+1)}(x_i) - y_i )^2    > a_1(\alpha)   \bigg) \\
                          \displaystyle \,=\, \int \int  \int \mathbb{P}( \frac{1}{\vert I_4\vert } \sum_{i \in I_4} (\hat{f}_{A_n}^{(B+1)}(x_i) - y_i )^2    > a_1(\alpha)   \,|\, \mathcal{D}_4,\mathcal{D}_{1:3},\Lambda ) dP_{ \Lambda | \mathcal{D}_{1:3} }(d\Lambda ) dP_{\mathcal{D}_4}(d\mathcal{D}_4) dP_{\mathcal{D}_{1:3} }(d\mathcal{D}_{1:3}  )  \\ 
  \displaystyle \,=\, \int \int  \int_{   \Omega_2(\mathcal{D}_{1:3},\mathcal{D}_4) } \mathbb{P}( \frac{1}{\vert I_4\vert } \sum_{i \in I_4} (\hat{f}_{A_n}^{(B+1)}(x_i) - y_i )^2    > a_1(\alpha)   \,|\, \mathcal{D}_4,\mathcal{D}_{1:3},\Lambda ) dP_{ \Lambda | \mathcal{D}_{1:3} }(d\Lambda ) dP_{\mathcal{D}_4}(d\mathcal{D}_4) dP_{\mathcal{D}_{1:3} }(d\mathcal{D}_{1:3}  ) \,+\, \\   
    \displaystyle  \,\,\,\,\,\,\,\int \int  \int_{   \Omega_2(\mathcal{D}_{1:3},\mathcal{D}_4)^c } \mathbb{P}( \frac{1}{\vert I_4\vert } \sum_{i \in I_4} (\hat{f}_{A_n}^{(B+1)}(x_i) - y_i )^2    > a_1(\alpha)   \,|\, \mathcal{D}_4 ,\mathcal{D}_{1:3},\Lambda ) dP_{ \Lambda | \mathcal{D}_{1:3} }(d\Lambda ) dP_{\mathcal{D}_4}(d\mathcal{D}_4) dP_{\mathcal{D}_{1:3} }(d\mathcal{D}_{1:3}  )  \\   
    \end{array}
\end{equation}
\begin{equation*}
\arraycolsep=1.4pt\def\arraystretch{1.6}
     \begin{array}{l}
            \displaystyle \,=\, \int \int  \int_{   \Omega_2(\mathcal{D}_{1:3},\mathcal{D}_4) } G_{ \mathcal{D}_4, \mathcal{D}_{1:3}   }( a_1(\alpha) ) dP_{ \Lambda | \mathcal{D}_{1:3} }(d\Lambda ) dP_{\mathcal{D}_4}(d\mathcal{D}_4) dP_{\mathcal{D}_{1:3} }(d\mathcal{D}_{1:3}  ) \,+\, \\   
    \displaystyle  \,\,\,\,\,\,\,\int \int  \int_{   \Omega_2(\mathcal{D}_{1:3},\mathcal{D}_4)^c } G_{ \mathcal{D}_4,\mathcal{D}_{1:3}}( a_1(\alpha) ) dP_{ \Lambda | \mathcal{D}_{1:3} }(d\Lambda ) dP_{\mathcal{D}_4}(d\mathcal{D}_4) dP_{\mathcal{D}_{1:3} }(d\mathcal{D}_{1:3}  )  \\   
          \displaystyle \,\leq \, \int \int  \int_{   \Omega_2(\mathcal{D}_{1:3},\mathcal{D}_4) } \big[\hat{G}_{ \mathcal{D}_4,\Lambda }( a_1(\alpha) ) + \frac{ \alpha}{32\widetilde{B} } \big]  dP_{ \Lambda | \mathcal{D}_{1:3} }(d\Lambda ) dP_{\mathcal{D}_4}(d\mathcal{D}_4) dP_{\mathcal{D}_{1:3} }(d\mathcal{D}_{1:3}  ) \,+\, \\   
    \displaystyle  \,\,\,\,\,\,\,\int \int  \int_{   \Omega_2(\mathcal{D}_{1:3},\mathcal{D}_4)^c } G_{ \mathcal{D}_4,\mathcal{D}_{1:3}}( a_1(\alpha) ) dP_{ \Lambda | \mathcal{D}_{1:3} }(d\Lambda ) dP_{\mathcal{D}_4}(d\mathcal{D}_4) dP_{\mathcal{D}_{1:3} }(d\mathcal{D}_{1:3}  )  \\   
          \displaystyle \,\leq \, \int \int  \int_{   \Omega_2(\mathcal{D}_{1:3},\mathcal{D}_4) } \frac{9 \alpha}{32 \widetilde{B}} \,  dP_{ \Lambda | \mathcal{D}_{1:3} }(d\Lambda ) dP_{\mathcal{D}_4}(d\mathcal{D}_4) dP_{\mathcal{D}_{1:3} }(d\mathcal{D}_{1:3}  ) \,+\, \\   
    \displaystyle  \,\,\,\,\,\,\,\int \int  \int_{   \Omega_2(\mathcal{D}_{1:3},\mathcal{D}_4)^c } G_{ \mathcal{D}_4,\mathcal{D}_{1:3}}( a_1(\alpha) ) dP_{ \Lambda | \mathcal{D}_{1:3} }(d\Lambda ) dP_{\mathcal{D}_4}(d\mathcal{D}_4) dP_{\mathcal{D}_{1:3} }(d\mathcal{D}_{1:3}  )  \\  
     \displaystyle \,\leq \, \frac{9\alpha}{32 \widetilde{B}}\,+\, \int \int  \int_{   \Omega_2(\mathcal{D}_{1:3},\mathcal{D}_4)^c } dP_{ \Lambda | \mathcal{D}_{1:3} }(d\Lambda ) dP_{\mathcal{D}_4}(d\mathcal{D}_4) dP_{\mathcal{D}_{1:3} }(d\mathcal{D}_{1:3}  ) \\
          \displaystyle \,\leq \, \frac{9 \alpha}{32\widetilde{B} }\,+\, 2\exp\left( -2 B \alpha^2 /(32^2 \widetilde{B}^2 ) \right)\\
           \displaystyle < \frac{9\alpha}{32 \widetilde{B}} \,+\,\frac{\alpha}{64 \widetilde{B}}
     \end{array}
\end{equation*}
where the first inequality follows by the definition of  $ \Omega_2(\mathcal{D}_{1:3},\mathcal{D}_4)$, the second by the definition of $a_1(\alpha)$, the fourth by the Dvoretzky–Kiefer–Wolfowitz inequality (see \cite{massart1990tight}), and the last since $B  \alpha^2/\widetilde{B}^2 \rightarrow \infty$ fast enough. \\

Hence, 
\[
\widetilde{B}\cdot \mathbb{P}\bigg(   \frac{1}{\vert I_4\vert } \sum_{i \in I_4} (\hat{f}_{A_n}^{(B+1)}(x_i) - y_i )^2    > a_1(\alpha)   \bigg) \,\leq \frac{19\alpha}{64}.
\]
Therefore, combining (\ref{eqn:e124}), Steps 4--6,  and (\ref{eqn:30}), we obtain 
\begin{equation}
	\label{eqn:320}
	 \mathbb{P}( \Omega_1^c )\,<\,0.58 \alpha.
\end{equation}
\textbf{Step 8.}\\

Next, we  observe that
\[
\arraycolsep=1.4pt\def\arraystretch{1.8}
\begin{array}{l}
  \displaystyle  \mathbb{P}(\vert \mathbb{E}(Z  \,|\,  X,\Omega_1,\{(x_i,y_i)\}_{i \in I_1 \cup I_2\cup I_3  }) \vert  \geq  b(\alpha)  )  \\
  \leq  \displaystyle  \frac{ \mathbb{E}(\vert  \mathbb{E}(Z  \,|\,  X,\Omega_1,\{(x_i,y_i)\}_{i \in I_1 \cup I_2\cup I_3  }) )\vert )  }{ b(\alpha)  }\\
  \leq  \displaystyle\frac{ \mathbb{E}(\vert Z \vert   \,|\, \Omega_1 ) }{ b(\alpha)    } \\
    \leq  \displaystyle\frac{ \mathbb{E}(\vert Z_{B+1} \vert   \,|\, \Omega_1  ) }{ b(\alpha)    } \\
    \leq  \displaystyle \frac{\mathbb{E}(\vert Z_{B+1} \vert)}{ b(\alpha) \cdot \mathbb{P}(\Omega_1) }\\
       =\displaystyle \frac{\mathbb{E}(\vert Z_{B+1} \vert)}{ b(\alpha) \cdot (1- \mathbb{P}(\Omega_1^c) ) }\\
        \leq \displaystyle \frac{\mathbb{E}(\vert Z_{B+1} \vert)}{ b(\alpha) \cdot (1- 0.58\alpha ) }\\
   \leq  \displaystyle\frac{ (\mathbb{E}(\vert Z_{B+1} \vert^2))^{1/2} }{ b(\alpha) \cdot (1- 0.58\alpha )    }\\
    =  \displaystyle\frac{ [  \mathbb{E}((\hat{f}_{A_n}^{(B+1)}(X) - f^*(X) )^2  )  ]^{1/2} }{ b(\alpha) \cdot (1- 0.58\alpha )    }\\
    =\displaystyle \frac{ [ \mathbb{E}((\hat{f}_{A_n}^{(B+1)}(X) - Y )^2 )-  \mathbb{E}(g^*(X))  ]^{1/2} }{ b(\alpha) \cdot (1- 0.58\alpha )    }\\
     \leq \displaystyle\frac{5 \alpha}{32}
\end{array}
\]
where the fifth inequality follows from (\ref{eqn:320}) and  the last equality follows from (\ref{eqn:300}).

\textbf{Step 11.}\\
Next we show (\ref{eqn:lenght}).  Throughout, we use $\mathrm{poly}(\cdot)$ to denote a polynomial function that can change from line to line.

By the calculation in (\ref{eqn:ver5}),  and Theorem \ref{thm5_v2}, Theorem \ref{thm2_v2} and Corollary \ref{thm_var_v2}, 
the choice of $B$ satisfying
\[
 B \left[    \exp(  -c n \min\{\phi_n, \psi_n\} \log^3 n  )\,+\, \mathbb{P}(\|\epsilon\|_{\infty})  > \mathcal{U}_n )\,+\, \exp(-c\log n) \right]   \,\rightarrow \,0
\]
for an appropriate constant $c>0$, implies, by union bound, that 
\begin{equation}
    \label{eqn:151}
     \underset{j=1,\ldots,B +  \widetilde{B}}{\max}\,  \frac{1}{\vert I_4\vert} \sum_{i\in I_4} (\hat{f}_{A_n}(x_i   )    \,-\, f_{\widetilde{A}_n}^{(j)}(x_i)  )^2   \,=\,  o_{\mathbb{P}}( ( \phi_n+ n^{-1}) \mathrm{poly}(\log n)   ),
\end{equation}
and 
\begin{equation}
    \label{eqn:150}
     \max\{    \|g^* \,-\, \hat{g}_{ A_n }\|_{ \mathcal{L}_2 } , \| \hat{g}_{ A_n }\,-\,   \hat{g}_{\widetilde{A}_n}^{(B+1)} \|_{ \mathcal{L}_2 }  \}  \,=\,  o_{\mathbb{P}}( [ \phi_n^{1/2}  +  \psi_n^{1/2} + n^{-1/2} ]\mathrm{poly}(\log n)  ).
\end{equation}

Then
\begin{equation}
    \begin{array}{l}
    \displaystyle \big \vert a_1(\alpha)\,-\,  \frac{1}{\vert I_4 \vert }\sum_{i \in I_4}  \hat{g}_{\widetilde{A}_n}^{(B+1)}(x_i) \big\vert  \\
   \displaystyle   \leq  \underset{j=1,\ldots,B}{\max}\,\big\vert   \frac{1}{\vert I_4 \vert }\sum_{i \in I_4} ( y_i \,-\,  \hat{f}_{\widetilde{A}_n}^{(j)}(x_i)   )^2    \,-\,  \frac{1}{\vert I_4 \vert }\sum_{i \in I_4}  \hat{g}_{\widetilde{A}_n}^{(B+1)}(x_i) \big\vert  \\
  \displaystyle  \leq \,  \underset{j=1,\ldots,B}{\max}\,\big\vert   \frac{1}{\vert I_4 \vert }\sum_{i \in I_4} ( y_i \,-\,  \hat{f}_{\widetilde{A}_n}^{(j)}(x_i)  )^2   \,-\, \frac{1}{\vert I_4 \vert }\sum_{i \in I_4} ( y_i \,-\,  f^*(x_i)   )^2  \big\vert \,+\,\\
    \displaystyle \,\,\,\,  \big\vert  \frac{1}{\vert I_4 \vert }\sum_{i \in I_4} ( y_i \,-\,  f^*(x_i)   )^2    \,-\, \frac{1}{\vert I_4 \vert }\sum_{i \in I_4} \hat{g}_{\widetilde{A}_n}^{(B+1)}(x_i)  \big\vert   \\
    \end{array}
\end{equation}
\begin{equation}
    \label{eqn:error}
   \begin{array}{l}
         \displaystyle  \leq \,  \underset{j=1,\ldots,B}{\max}\,\big\vert   \frac{1}{\vert I_4 \vert }\sum_{i \in I_4}  2\epsilon_i( f^*(x_i)\,-\,  \hat{f}_{\widetilde{A}_n}^{(j)} (x_i)  ) \big\vert \,+\,   \underset{j=1,\ldots,B}{\max} \frac{1}{\vert I_4 \vert }\sum_{i \in I_4} ( f^*(x_i) \,-\,  \hat{f}_{\widetilde{A}_n}^{(j)}(x_i)   )^2  \,+\, \\
    \displaystyle \,\,\,\,\,\,\,\,\,\, \,\,\,\,\,\, \big\vert  \frac{1}{\vert I_4 \vert }\sum_{i \in I_4} ( y_i \,-\,  f^*(x_i)   )^2    \,-\, \frac{1}{\vert I_4 \vert }\sum_{i \in I_4} \hat{g}_{\widetilde{A}_n}^{(B+1)}(x_i)  \big\vert.
   \end{array}
  \end{equation}

However,  conditioning on $\{\hat{f}_{A_n}^{(j)}\}_{j=1}^B$ and $\{x_i\}_{i \in I_4}$, by union bound and the sub-Gaussian tail inequality,
\begin{equation}
    \label{eqn:160}
    \begin{array}{l}
    \displaystyle     \underset{j=1,\ldots,B}{\max}\,\big\vert   \frac{1}{\vert I_4 \vert }\sum_{i \in I_4}  2\epsilon_i( f^*(x_i)\,-\,  \hat{f}_{\widetilde{A}_n}^{(j)} (x_i)  ) \big\vert \\
   \displaystyle    \, = \, o_{\mathbb{P}}\bigg(\sqrt{ \frac{\log n}{n  }} \underset{j=1,\ldots,B}{\max}\,\sqrt{  \frac{1}{\vert I_4 \vert }\sum_{i \in I_4} ( f^*(x_i) \,-\,  \hat{f}_{\widetilde{A}_n}^{(j)}(x_i)   )^2 }\bigg). \\
    \end{array}
\end{equation}
Furthermore, 
\begin{equation}
    \label{eqn:160}
    \arraycolsep=1.4pt\def\arraystretch{1.6}
    \begin{array}{l}
    \displaystyle   \underset{j=1,\ldots,B}{\max} \frac{1}{\vert I_4 \vert }\sum_{i \in I_4} ( f^*(x_i) \,-\,  \hat{f}_{\widetilde{A}_n}^{(j)}(x_i)   )^2\\
    \displaystyle \leq \frac{2}{\vert I_4 \vert }\sum_{i \in I_4} ( f^*(x_i) \,-\,  \hat{f}_{A_n}(x_i)   )^2\,+\, \underset{j=1,\ldots,B}{\max} \frac{2}{\vert I_4 \vert }\sum_{i \in I_4} (  \hat{f}_{A_n}(x_i)\,-\,  \hat{f}_{\widetilde{A}_n}^{(j)}(x_i)   )^2\\
    \displaystyle \,=\, o_{\mathbb{P}}( [\phi_n  +  n^{-1} ]\mathrm{poly}(\log n) ),
    \end{array}
\end{equation}
where the equality follows from (\ref{eqn:151}) and Theorem \ref{thm5_v2}. So we proceed to bound the last term in the right hand-side of (\ref{eqn:error}). Towards that end, notice
that 
\begin{equation}
    \label{eqn:163}
    \arraycolsep=1.4pt\def\arraystretch{1.6}
    \begin{array}{l}
       \displaystyle   \big\vert  \frac{1}{\vert I_4 \vert }\sum_{i \in I_4} ( y_i \,-\,  f^*(x_i)   )^2    \,-\, \frac{1}{\vert I_4 \vert }\sum_{i \in I_4} \hat{g}_{\widetilde{A}_n}^{(B+1)}(x_i)  \big\vert \\
              \displaystyle \leq   \big\vert  \frac{1}{\vert I_4 \vert }\sum_{i \in I_4} ( y_i \,-\,  f^*(x_i)   )^2    \,-\, \frac{1}{\vert I_4 \vert }\sum_{i \in I_4} g^*(x_i) \big\vert\,+\, \\
            \displaystyle\,\,\,\,\, \big\vert  \frac{1}{\vert I_4 \vert }\sum_{i \in I_4} g^*(x_i) \,-\,  \frac{1}{\vert I_4 \vert }\sum_{i \in I_4} \hat{g}_{\widetilde{A}_n}^{(B+1)}(x_i)  \big\vert\\
    \displaystyle\leq  \big\vert  \frac{1}{\vert I_4 \vert }\sum_{i \in I_4} g^*(x_i) \,-\,  \mathbb{E}(g^*(X))  \big\vert \,+\, o_{\mathbb{P}}(n^{-1/2}\mathrm{poly}(\log n)  )\\ 
    \displaystyle \,\,\,\, \,\big\vert  \mathbb{E}( \hat{g}_{\widetilde{A}_n}^{(B+1)}(X)\,|\, \hat{g}_{\widetilde{A}_n}^{(B+1)} ) \,-\,  \frac{1}{\vert I_4 \vert }\sum_{i \in I_4} \hat{g}_{\widetilde{A}_n}^{(B+1)}(x_i)  \big\vert \,+\,\\
       \displaystyle \,\,\,\,\, \big\vert \mathbb{E}(g^*(X))   \,-\,\mathbb{E}( \hat{g}_{\widetilde{A}_n}^{(B+1)}(X)\,|\, \hat{g}_{\widetilde{A}_n}^{(B+1)} )\big\vert \\
        \displaystyle\leq  \big\vert \mathbb{E}(g^*(X))   \,-\,\mathbb{E}( \hat{g}_{\widetilde{A}_n}^{(B+1)}(X)\,|\, \hat{g}_{\widetilde{A}_n}^{(B+1)} )\big\vert \,+\, o_{\mathbb{P}}(n^{-1/2}\mathrm{poly}(\log n) )\\   
    \end{array}
\end{equation}
where the second and third inequalities follow by Chebyshev's inequality and the triangle inequality. It remains to control \[
a_0 \,=\,   \big\vert \mathbb{E}(g^*(X))   \,-\,\mathbb{E}( \hat{g}_{\widetilde{A}_n}^{(B+1)}(X)\,|\, \hat{g}_{\widetilde{A}_n}^{(B+1)} )\big\vert. 
\]
Notice that
\begin{equation}
\arraycolsep=1.4pt\def\arraystretch{1.6}
    \begin{array}{lll}
      \big\vert \mathbb{E}(g^*(X))   \,-\,\mathbb{E}( \hat{g}_{\widetilde{A}_n}^{(B+1)}(X)\,|\, \hat{g}_{\widetilde{A}_n}^{(B+1)} )\big\vert&  \leq   & \|g^*\,-\, \hat{g}_{\widetilde{A}_n}^{(B+1)}\|_{ \mathcal{L}_2 }   \\
      &  \leq   & \|g^*\,-\, \hat{g}_{A_n}\|_{ \mathcal{L}_2 } \,+\,\|\hat{g}_{A_n}\,-\, \hat{g}_{\widetilde{A}_n}^{(B+1)}\|_{ \mathcal{L}_2 } \\
         &=& o_{\mathbb{P}}( [ \phi_n^{1/2}  +  \psi_n^{1/2} + n^{-1/2} ]\mathrm{poly}(\log n)  )
         \end{array}
\end{equation}
where the last inequality follows from (\ref{eqn:150}). 

 The claim in  (\ref{eqn:lenght}) then follows.

\end{proof}

\subsection{Proof of Corollary \ref{cor1_v2}}

\begin{proof}
    This follows from the proof of Theorem \ref{thm6}.  We observe that
        \begin{equation}
\label{eqn:120}
    \begin{array}{lll}
    \displaystyle     \vert \mathbb{E}(g^*(X))\,-\,\mathbb{E}(\hat{g}_{A_n}^{(B+1)}(X) \,|\, \hat{g}_{A_n}^{(B+1)} \vert   & \leq \| g^*  \,-\,\hat{g}_A^{(B+1)}\|_{ \mathcal{L}_2}.     & \
    \end{array}
\end{equation}
 Then in Step 3 of the proof Theorem \ref{thm6}, we now have 
\begin{equation}
\label{eqn:170}
\arraycolsep=1.4pt\def\arraystretch{1.6}
    \begin{array}{lll}
         a(\alpha) + \mathbb{E}(g^*(X))& \geq & \displaystyle a_1(\alpha) \,+\, a_2(\alpha)\,+\, a_3(\alpha)- \frac{1}{\vert I_4\vert} \sum_{i \in I_4 }\hat{g}_{A_n}^{(B+1)}(x_i)  \,+\, \\
          & &\displaystyle\vert  \mathbb{E}( g^*(X) )  \,-\, \mathbb{E}( \hat{g}_{A_n}^{(B+1)}(X)  \,|\,  \hat{g}_{A_n}^{(B+1)} ) \vert\,+\,\\
       & & \displaystyle     \mathbb{E}( \hat{g}_{A_n}^{(B+1)}(X)  \,|\,  \hat{g}_{A_n}^{(B+1)} )  \,+\,
          \mathbb{E}( g^*(X) )  \,-\, \mathbb{E}( \hat{g}_{A_n}^{(B+1)}(X)   \,|\,  \hat{g}_{A_n}^{(B+1)}) \,+\, \\
           & &\displaystyle    \frac{\alpha}{100\log^s n} - \vert  \mathbb{E}( g^*(X) )  \,-\, \mathbb{E}( \hat{g}_{A_n}^{(B+1)}(X)  \,|\,  \hat{g}_{A_n}^{(B+1)} ) \vert\\
    \end{array}
\end{equation}
and so 
\begin{equation}
\label{eqn:171}
    \begin{array}{lll}
         a(\alpha) + \mathbb{E}(g^*(X))& \geq & \displaystyle a_1(\alpha) \,+\, a_2(\alpha)\,+\, a_3(\alpha)- \big\vert \frac{1}{\vert I_4\vert} \sum_{i \in I_4 }\hat{g}_{A_n}^{(B+1)}(x_i)  \,-\,  \mathbb{E}( \hat{g}_{A_n}^{(B+1)}(X) \,|\,  \hat{g}_{A_n}^{(B+1)} ) \big\vert  \,+\,\\
   & &\displaystyle    \frac{\alpha}{100\log^s n} - \vert  \mathbb{E}( g^*(X) )  \,-\, \mathbb{E}( \hat{g}_{A_n}^{(B+1)}(X)  \,|\,  \hat{g}_{A_n}^{(B+1)} ) \vert \\
    \end{array}
\end{equation}
which implies 
\begin{equation}
    \label{eqn:e172}
    \arraycolsep=1.4pt\def\arraystretch{1.6}
    \begin{array}{lll}
                \mathbb{P}(\Omega^c_1)   & \leq&\displaystyle   \widetilde{B} \cdot\mathbb{P}(  \mathbb{E}((\hat{f}_{A_n}^{(B+1)}(X) - Y )^2\,|\,\hat{f}_{A_n}^{(B+1)}  )  >  a_1(\alpha) +  a_2(\alpha)  )\,+\, \\
         & &\displaystyle  \widetilde{B} \cdot \mathbb{P}\big( \big\vert \frac{1}{\vert I_4\vert} \sum_{i \in I_4 }\hat{g}_{A_n}^{(B+1)}(x_i)  \,-\,  \mathbb{E}( \hat{g}_A^{(B+1)}(X)  \,|\,  \hat{g}_{A_n}^{(B+1)} ) \big\vert >  a_3(\alpha)  \big)\,+\,\\
         & &\displaystyle  \widetilde{B} \cdot  \mathbb{P}(  \vert  \mathbb{E}( g^*(X) ) \,-\, \mathbb{E}( \hat{g}_{A_n}^{(B+1)}(X)  \,|\,  \hat{g}_{A_n}^{(B+1)} ) \vert  \geq \frac{\alpha}{100\log^s n})\\ 
          & \leq&\displaystyle   \widetilde{B} \cdot\mathbb{P}(  \mathbb{E}((\hat{f}_{A_n}^{(B+1)}(X) - Y )^2\,|\,\hat{f}_{A_n}^{(B+1)}  )  >  a_1(\alpha) +  a_2(\alpha)  )\,+\, \\
         & &\displaystyle   \widetilde{B} \cdot\mathbb{P}\big( \big\vert \frac{1}{\vert I_4\vert} \sum_{i \in I_4 }\hat{g}_{A_n}^{(B+1)}(x_i)  \,-\,  \mathbb{E}( \hat{g}_A^{(B+1)}(X)  \,|\,  \hat{g}_{A_n}^{(B+1)} ) \big\vert >  a_3(\alpha)  \big)\,+\,\\
          & &\displaystyle  \widetilde{B} \cdot\mathbb{P}(  \| g^*  \,-\,\hat{g}_A^{(B+1)}\|_{ \mathcal{L}_2}     \geq \frac{\alpha}{100\log^s n})\\ 
          & \leq&\displaystyle   \widetilde{B} \cdot\mathbb{P}(  \mathbb{E}((\hat{f}_{A_n}^{(B+1)}(X) - Y )^2\,|\,\hat{f}_{A_n}^{(B+1)}  )  >  a_1(\alpha) +  a_2(\alpha)  )\,+\, \\
         & &\displaystyle   \widetilde{B} \cdot\mathbb{P}\big( \big\vert \frac{1}{\vert I_4\vert} \sum_{i \in I_4 }\hat{g}_{A_n}^{(B+1)}(x_i)  \,-\,  \mathbb{E}( \hat{g}_A^{(B+1)}(X)  \,|\,  \hat{g}_{A_n}^{(B+1)} ) \big\vert >  a_3(\alpha)  \big)\,+\,   o\left(\frac{1}{n}\right) \\
    \end{array}
\end{equation}
where the second inequality follows from (\ref{eqn:e124}) and the last by our assumption in the statement of Corollary \ref{cor1_v2}. 

The rest of the proof of Theorem \ref{thm6} remains the same noticing that for any $t>0$
\[
    \arraycolsep=1.4pt\def\arraystretch{1.8}
   \begin{array}{l}
      \displaystyle  \frac{ 32[  \mathbb{E}((\hat{f}_{A_n}^{(B+1)}(X) - f^*(X) )^2  )  ]^{1/2} }{ 5\alpha  (1-0.58\alpha)     }\\
      = \displaystyle \frac{32}{5\alpha (1-0.58\alpha) } \,[\mathbb{E}( \|\hat{f}_{A_n}^{(B+1)} - f^*  \|_{\mathcal{L}_2}^2  )   ]^{1/2}  \\
      \leq   \displaystyle \frac{32}{5\alpha (1-0.58\alpha) } \,[\mathbb{E}( \|\hat{f}_{A_n}^{(B+1)} - f^*  \|_{\mathcal{L}_2}^2 \,|\, \|\hat{f}_{A_n}^{(B+1)} - f^*  \|_{\mathcal{L}_2}\leq t )    \
      \,+\, \\
       \displaystyle\mathbb{E}( \|\hat{f}_{A_n}^{(B+1)} - f^*  \|_{\mathcal{L}_2}^2 \,|\, \|\hat{f}_{A_n}^{(B+1)} - f^*  \|_{\mathcal{L}_2} > t) \,\cdot\, \mathbb{P}(  \|\hat{f}_{A_n}^{(B+1)} - f^*  \|_{\mathcal{L}_2} >  t) ]^{1/2}  \\
        \leq  \displaystyle    \frac{32}{5 \alpha (1-0.58\alpha) }[ t^2\,+\,   4 A_n^2 \mathbb{P}(  \|\hat{f}_{A_n}^{(B+1)} - f^*  \|_{\mathcal{L}_2} >  t) ]^{1/2}\\
                \leq  \displaystyle    \frac{32t}{5 \alpha (1-0.58\alpha) }\,+\,    \frac{64}{5 \alpha} A_n [\mathbb{P}(  \|\hat{f}_{A_n}^{(B+1)} - f^*  \|_{\mathcal{L}_2} >  t) ]^{1/2}\\
                 \leq \displaystyle \frac{32t}{5 \alpha (1-0.58\alpha) }\,+\,     O(n^{-1})\\
                 \leq \displaystyle\frac{1}{100 \log^s n}
   \end{array}
\]
by taking $t \asymp \frac{\alpha}{ \log^s n}$ and the conditions in Corollary \ref{cor1_v2}.
\end{proof}

\end{document}